\documentclass[3p]{elsarticle}
\usepackage{float}
\usepackage{graphicx}
\usepackage{color,soul}
\usepackage{amsmath,amssymb,verbatim,bm}
\usepackage{ctable} % needed for fancy tables (booktabs also works)
\usepackage{url}
\usepackage{bm}
\usepackage{algorithm}
\usepackage{algpseudocode}
\usepackage[english]{babel}
\usepackage{ulem} 
\usepackage{mdframed}
\usepackage{stmaryrd}
\newmdenv[
leftmargin = 0pt,
innerleftmargin = 1em,
innertopmargin = 0pt,
innerbottommargin = 0pt,
innerrightmargin = 0pt,
rightmargin = 0pt,
linewidth = 3pt,
topline = false,
rightline = false,
bottomline = false,
linecolor=red,
]{leftbar}

\definecolor{marine}{RGB}{85,100,255}
\definecolor{pomegrenate}{RGB}{255,18,94}

\usepackage{xspace}
\usepackage{xparse}
\usepackage{amsmath,amssymb,verbatim,bm}
\usepackage{float}
\usepackage{graphicx}
\usepackage{algorithm}
\usepackage{algpseudocode}
\DeclareMathOperator{\theargmin}{argmin}
\DeclareMathOperator{\BERN}{Bernoulli}

\newcommand{\Gk}{G^{(k)}}
\newcommand{\sGk}{\lbrace \Gk \rbrace}

\DeclareMathOperator{\M}{\mathbb{M}}
\DeclareMathOperator{\N}{\mathbb{N}}
\DeclareMathOperator{\R}{\mathbb{R}}

\DeclareMathOperator{\cM}{\mathcal{M}}
\DeclareMathOperator{\cO}{\mathcal{O}}

\newcommand{\bA}{\bm{A}}

\newcommand{\cG}{\mathcal{G}}

\newcommand{\bQ}{\bm{Q}}

\newcommand{\bp}{\bm{p}}
\newcommand{\bs}{\bm{s}}

\newcommand{\blamb}{\bm{\lambda}}
\newcommand{\Cov}{\text{Cov}}
\newcommand{\Var}{\text{Var}}

\DeclareRobustCommand{\argmin}[1]{\underset{#1}{\theargmin}\mspace{4mu}}
\DeclareRobustCommand{\bern}[1]{\BERN \left(#1\right)}

\NewDocumentCommand \E { m o}
{
  \IfNoValueTF {#2} {\mathbb{E}\left[#1\right]}{\mathbb{E}_{#2}\mspace{-4mu}\left[#1\right] }
}
\NewDocumentCommand \prob { m o}
{
  \IfNoValueTF {#2} {\mathbb{P}\left(#1\right)}{\mathbb{P}_{#2} \left(#1\right)}
}

\newcommand{\ER}{Erd\H{o}s-R\'enyi}

\makeatletter
\def\@opargbegintheorem#1#2#3{\trivlist
   \item[]{\bfseries #1\ #2\ (#3)} \itshape}
\makeatother

\newtheorem{corollary}{Corollary}
\newtheorem{Definition}{Definition}
\newtheorem{cproof}{Proof of Corollary}

\newtheorem{Lemma}{Lemma}

\newtheorem{lproof}{Proof of Lemma}
\newtheorem{proof}{Proof of Theorem}
\newtheorem{Remark}{Remark}
\newtheorem{theorem}{Theorem}

\newtheorem{Proposition}{Proposition}
\newtheorem{propproof}{Proof of Proposition}

\begin{document}
\begin{frontmatter}

  \title{Probability density estimation for sets of large graphs with respect to spectral information using stochastic block models}
  \author{Daniel Ferguson}
  \author{Fran\c{c}ois G. Meyer\fnref{fnt2}}
  \address{Applied Mathematics, University of Colorado at Boulder, Boulder CO 80305}
  \fntext[fnt2]{Corresponding author: fmeyer@colorado.edu\\ This work was supported by the National Science Foundation, CCF/CIF 1815971.}

\begin{abstract}
\small\baselineskip=9pt For graph-valued data sampled iid from a distribution $\mu$, the sample moments are computed with respect to a choice of metric. In this work, we equip the set of graphs with the pseudo-metric defined by the $\ell_2$ norm between the eigenvalues of the respective adjacency matrices. We use this pseudo metric and the respective sample moments of a graph valued data set to infer the parameters of a distribution $\hat{\mu}$ and interpret this distribution as an approximation of $\mu$. We verify experimentally that complex distributions $\mu$ can be approximated well taking this approach.
\end{abstract}

\end{frontmatter}

%__________________________________________________________
\section{Introduction.}
\label{sec:Intro}
%__________________________________________________________

The ubiquity of graphs to represent relationships between objects, such as those in social networks, biology networks, traffic networks, molecular data, etc. has showcased a need for advancement in algorithms that analyze such objects. Analysis of graphs can be categorized, albeit broadly, into two types: the analysis of graph patterns and the analysis of patterns of graphs. The former covers topics such as node classification, degree distribution, link prediction, graph embedding, and subgraph presence, among countless others. The latter involves patterns among sets of graphs. Examples of such patterns include the presence of communities, hubs, or the small-world phenomenon; each of which has been observed in various real-world networks. The problem of graph generation falls into the latter category of graph analysis. The process of generating graphs with pre-specified structures observed from real-world networks has led to the advent of various graph ensembles, such as the Barabasi-Albert model, the Watts-Strogatz model, and the stochastic block model. Each ensemble of graphs guarantees that a pre-specified structure exists in the graph with high probability; however, the specifics of each pre-specified structure may vary greatly within the ensemble depending on the initial parameters. 

In this work, we consider the problem in graph generation, called density estimation. Given a set of graphs with a distribution of structures, we seek to generate new graphs according to this distribution. We utilize the sample Fr\'echet mean graph and the sample total Fr\'echet variance of the data to inform the parameters of our generative model. Notably, because our work is always performed with respect to a distance (a requirement to determine the Fr\'echet mean and variance), we consider the spectral information captured by the adjacency matrix of two graphs to determine their similarity, specifically the $\ell_2$ norm between the largest eigenvalues of the adjacency matrices of the observed graphs.

We consider sets of simple graphs with $n$ vertices that have an edge density, $\rho_n$, that satisfies
\begin{equation}
	n^{-2/3} \ll \rho_n \ll 1.
\end{equation}
We also note that the vertex set must be sufficiently large, and the method used in this work will perform poorly for sets of small graphs.

We approach the problem of density estimation by considering the following two ideas: (1) there exists a stochastic block model whose Fr\'echet mean graph has an adjacency matrix where the largest eigenvalues are arbitrarily close to that of the sample Fr\'echet mean graph; (2) we may adjust the variance of the recovered stochastic block model by defining a distribution on the parameters. From this perspective, we may align both the mean and variance of a distribution with the sample mean and sample variance from the sample set of graphs.

%__________________________________________________________
\section{State of the Art.}
\label{sec:StateOfArt}
%__________________________________________________________
Generative models for graphs  have a long history. Popular ensembles of graphs aim to capture various observed real-world phenomenon such as the presence of ``hubs'' \cite{BA02}, the small-world phenomenon \cite{WS98}, community structure \cite{HLL83}, or specific subgraphs \cite{KS80}. Popular variations of such ensembles further generalize the structures captured by the ensemble (see for instance the degree corrected stochastic block model \cite{LS19}, in-homogeneous $\text{\ER}$ models \cite{BBK19, CCH20}, or exponential random graph models \cite{RSWHP07}). A new method using neural networks, called deep generative models, seeks to model the structures of observed graphs without pre-specifying the structures and instead learns relevant structures in the observed graphs (see \cite{BHPS20, GZ20} for reviews on these models).

While each ensemble captures various structural phenomena, the graph generation process is dependent on a set of initial parameters. A data-driven generative model will seek to infer the correct value of these parameters conditioned on some set of observations. 

Current work in this field is offered by Lunagomez et. al. \cite{LOW20}, wherein the authors specify a generative process that distributes graphs about the Fr\'echet mean graph of an observed set. In this work, the Fr\'echet mean graph is determined with respect to the Hamming distance, but the ideas generalize to any Fr\'echet mean graph (i.e., for any choice of distance assuming one can compute it). A mixture model with respect to the Hamming distance is proposed in \cite{YKN22}, where multiple mean graphs are considered as defining the centers of each mixture component and graphs are sampled within each mixture according to the observed deviations from the mean graph. When considering a Markov random graph model (or its generalization, the exponential random graph model), various procedures have been proposed to determine the parameters, e.g., maximum pseudo-likelihood estimation \cite{SI90} and Monte Carlo Markov chain maximum \cite{HH06, S02}.

All parameter estimations in the above models compare graphs using a Hamming distance. The Hamming distance identifies local changes in the connectivity structure between nodes. At times, the Hamming distance can be used to detect global phenomena in the graphs, i.e., relating the presence of triangles to the density of $\text{\ER}$ random graph models; however, most commonly, the Hamming distance measures only the local connectivity.

Taking a similar perspective as Lunagomez et. al. \cite{LOW20}, we determine the parameters of our model using the sample moments of the observed data. However, as an extension to the work in \cite{LOW20}, here we consider both the first and second moments of the data. The choice of metric is crucial to the location and spread of graphs as each metric induces a different mean graph and different spread about the mean graph. The Fr\'echet mean and Fr\'echet variance have been analyzed with respect to the edit distance (see \cite{BNY20,G12,J16,JO09, JMB01, LOW20} and references therein); our aim is to capture the mean and variability of global structures within the data set of graphs.

To this end, we consider a distance that can detect such structural changes  (e.g., community structure \cite{A17, LGT14}, modularity \cite{GN02}). The adjacency spectral distance, which we define as the $\ell_2$ norm of the difference between the spectra of the adjacency matrices of the two graphs of interest \cite{WZ08} exhibits good performance when comparing various types of graphs  \cite{WM20}, making it a reliable choice for a wide range of problems. Spectral distances also offer practical advantages as they can inherently compare graphs of different sizes and graphs without known vertex correspondence (see e.g., \cite{FSS05, FVSRB10} and references therein).

In practice, it is often the case that only the $c$ largest eigenvalues are considered, with $c \ll n$. We still refer to such distances as spectral distances but comparison using the $c$ largest eigenvalues for small values of $c$ allows one to focus on the global structures of the graphs while omitting the local structures \cite{LGT14}. Of notable importance when considering the distance between the $c$ largest eigenvalues is recent work showcasing how to approximately compute the sample Fr\'echet mean graph (see prior work in \cite{FM22}). With access to the sample Fr\'echet mean graph, a generative model may be centered about this mean graph.

Further work in the realm of generative modeling when considering the sample moments of the data comes in the form of regression \cite{PM19} wherein the authors construct a parameterization of a weighted sample Fr\'echet mean as a generalization of linear regression for metric objects. When considering the second moment of graph-valued data, a theoretical analysis of the Fr\'echet variance allows the user to construct a test to compare samples of metric-valued objects such as graphs \cite{DM17}.
%__________________________________________________________
\section{Main Contributions.}
\label{sec:Contri}
%__________________________________________________________
In this paper, we introduce a variation on the stochastic block model for graph generation, namely the random-parameter stochastic block model, by allowing the parameters of a stochastic block model graph to vary according to some distribution $J$ where the choice of $J$ is at the discretion of the researcher. The need for such a generalization exists since stochastic block models generate graphs in a homogeneous way, meaning there is little variance between graphs sampled according to this model. We showcase that methods which estimate the parameters of a stochastic block model can be used to estimate the moments of the distribution $J$, which, under certain parametric assumptions of $J$, uniquely characterizes the distribution.  When $J$ is assumed to be non-parametric, a generalization of kernel density estimation is suggested as a method to better infer the distribution $J$.

When considering sample sets of graphs, the sample arithmetic mean of the largest eigenvalues of the adjacency matrices of the graphs observed and the corresponding sample covariance is shown to be connected to the sample Fr\'echet mean graph and total sample Fr\'echet variance. This indicates that inferring a generative model using the sample mean eigenvalues and sample covariance matrix is equivalent to inferring a model using the sample Fr\'echet mean and sample total Fr\'echet variance.

We experimentally verify our results on four different data sets to explore the limitations of the proposed model. We verify that we can recover the parameters of a random-parameter stochastic block model and showcase the impact of the distribution $J$ when considering data not generated from a random-parameter stochastic block model. We also construct a generative model for real-world data where the graphs in the sample are collected according to face-to-face connections formed at a primary school.

%__________________________________________________________
\section{Notations.}
\label{sec:Not}
%__________________________________________________________
$G = (V,E)$ denotes a graph with vertex set $V = \lbrace 1,2,...,n \rbrace$ and edge set $E \subset V \times V$. For
vertices $i,j \in V$ an edge exists between them if the pair $(i,j) \in E$. The size of a graph is
$n = \vert V \vert$ and the number of edges is $m = \vert E \vert$. The density of a graph is
$\rho_n = \frac{m}{n(n-1)/2}$. 

The matrix $\bA$ is the adjacency matrix of the graph and is defined as
\begin{align}
  \bA_{ij} = 
  \begin{cases}
    1 \quad \text{if }(i,j) \in E,\\
    0 \quad \text{else.}
  \end{cases}
\end{align}
  We define the function $\sigma$ to be the mapping from the set of $n \times n$ adjacency matrices (square, symmetric matrices with zero entries
  on the diagonal), $\M_{n \times n}$ to $\R^n$ that assigns to an adjacency matrix the vector of its $n$ sorted eigenvalues,
  \begin{align}
    \sigma:        \M_{n\times n} & \longrightarrow \R^n,\\
    \bA & \longmapsto \blamb = [\lambda_1,\ldots,\lambda_n],
  \end{align}
  where $\lambda_1 \geq \ldots \geq \lambda_n$. Because we often consider the $c$ largest eigenvalue of the
  adjacency matrix $\bA$, we define the mapping to the truncated spectrum as $\sigma_c$ ,
  \begin{align}
    \sigma_c:        \M_{n\times n} & \longrightarrow \R^c,\\
    \bA & \longmapsto \blamb_c = [\lambda_1,\ldots,\lambda_c].
  \end{align}
\begin{Definition}
We define the adjacency spectral pseudometric as the $\ell_2$ norm between
  the spectra of the respective adjacency matrices,
  \begin{align} 
    d_A(G,G') = ||\sigma(\bA)-\sigma(\bA')||_2. \label{distance-adj}
  \end{align}
 
\end{Definition}
The pseudometric $d_A$ satisfies the symmetry and triangle inequality axioms, but not the identity axiom. Instead, $d_A$
      satisfies the reflexivity axiom
      \begin{equation*}
        d_A(G,G) = 0, \quad \forall G \in \cG.
      \end{equation*}

When the adjacency (or Laplacian) matrices of graphs have a similar spectra, it can be shown that the graphs have similar global and structural properties \cite{WM20}. As a natural extension of this spectral metric, sometimes only the $c$ largest eigenvalues are measured where $c \ll n$. We refer to this next metric as a truncation of the adjacency spectral pseudometric.

\begin{Definition}
\label{def:TruncMet}
We define the truncated adjacency spectral pseudometric as the $\ell_2$ norm between
  the largest $c$ spectra of the respective adjacency matrices, 
    \begin{align}
      d_{A_c}(G,G') = ||\sigma_c(\bA)-\sigma_c(\bA')||_2. \label{distance-trunc}
    \end{align}
 \end{Definition}
  \begin{Definition}
  \label{def:SetOfGraphs}
    $\cG$ denotes the set of all simple unweighted graphs on $n$ nodes.
  \end{Definition}

%________________________________________________________________________
\subsection{Random Graphs} 
\label{subsec:RandGraphs}
%________________________________________________________________________
$\cM (\cG)$ denotes the space  of probability measures on $\cG$. In this work, when we refer to a measure we always
mean a probability measure.
  \begin{Definition}
    We define the set of random graphs distributed according  to $\mu$ to be the probability space $\left(\cG, \mu\right)$.
  \end{Definition}
  \begin{Remark}
    In this paper, the $\sigma$-field associated with $\left(\cG, \mu\right)$ will always be the power set of $\cG$.
  \end{Remark}
  This definition allows us to unify various ensembles of random graphs (e.g., \ER, inhomogeneous \ER, small-world, Barabasi-Albert, etc.) via the unique concept of a probability space.
%________________________________________________________________________
\subsubsection{Kernel Probability Measures}
  \label{subsubsec:kernProbMeas}
%________________________________________________________________________
  \noindent Here we define an important class of probability measures for our study.
  \begin{Definition}
  \label{def:f}
    A probability measure $\mu \in \cM (\cG)$ is called a kernel probability measure if there exist a positive constant $\omega_n \leq 1$ and a function $f$,
    \begin{equation}
      f: [0,1]\times[0,1] \mapsto [0,1],
    \end{equation}
    such that $f(x,y) = f(y,x)$, and 
    \begin{align*}
      \forall G \in \cG, \text{with adjacency matrix}\; \bA=\left(a_{ij}\right), \\
      \mu\left( \left\{ \bA \right \}\right) = \mspace{-12mu} \prod_{1 \leq i < j \leq n}\mspace{-12mu}\prob{a_{ij}} = \mspace{-12mu} \prod_{1 \leq i < j \leq
        n}\mspace{-12mu} \bern{\omega_n f(\frac i n,\frac j n)}.
    \end{align*}
    The function $f$ is called a kernel of $\mu$.
\end{Definition}
\begin{Remark}
We refer to these measures as kernel probability measures because the kernels naturally give rise to linear integral operators with kernels $f$. Observe that when $||f||_1 = 1$, $\omega_n$ defines the expected density of the graphs sampled from $\mu$.
\end{Remark}

\noindent We note that given the sequence $\left \{\frac i n \right \}_{i=1}^n$ and the measure $\mu$, the kernel $f$ forms an
  equivalence class of functions, characterized by their values on the grid $\left \{\frac i n \right \}_{i=1}^n \times \left \{\frac j n \right \}_{j=1}^n$.

\begin{Definition}
$G_{\mu}$ denotes a random realization of a graph $G \in \left(\cG,\mu\right)$.
\end{Definition}
We use the following notation to denote random vectors.
\begin{Definition}
Let $\cM(\R^m)$ denote the set of probability measures on $\R^m$. For $\nu \in \cM(\R^m)$, $\bm{X}_{\nu}$ denotes a random realization of a vector $\bm X \in \left(\R^m, \nu \right)$.
\end{Definition}

%______________________________________________________________________________________________
\subsubsection{Stochastic Block Models
  \label{subsec:sbm}}
%_____________________________________________________________________________________________
The stochastic block model (see \cite{A17}) plays an important role in this work. We review the specific features of this model using the notations that were defined in the previous paragraphs. The key aspects of the model are: the geometry of the blocks, the within-community edges densities, and the across-community edge densities. An example of the kernel function and associated adjacency matrix from a stochastic block model is given in Fig.~\ref{fig:Ex-f}.

We denote by $c$ the number of communities in the stochastic block model. The {\noindent \bf geometry} of the stochastic block model is encoded using the relative sizes of the communities. We denote by $\bs \in \ell_1$ a non-increasing non-negative sequence of relative community sizes with $c$ non-zero entries and $||\bs||_1 = 1$.

For the geometry specified by $\bm s$ we define an associated {\bf edge density} vector $\bp$ such that $0<p_i$ for $i = 1,...,c$ and $p_i = 0$ for $i > c$ which describes the within-community edge densities. 

Finally, $\bQ = (q_{ij})$ denotes an infinite matrix of cross-community edge densities where $q_{i,i} = 0$, $q_{i,j}  = q_{j,i}$, and $q_{i,j} = 0$ if $i > c$ or $j > c$.

\begin{Remark}
We allow for infinite vectors with a finite number of non-zero entries so that we may allow for the smooth introduction of new communities within the stochastic block model. For example, let $t \in [0,1]$ and parametrize $\bs$ and $\bp$ by $t$ as
\begin{equation}
\bs(t) = \begin{bmatrix}1 - t/2 \\ t/2 \\ 0 \\ \vdots\end{bmatrix} \quad \omega_n \bp(t) = \begin{bmatrix} 0.2 + t/2 \\ 0.1+t/2 \\ 0 \\ \vdots \end{bmatrix}.
\end{equation}
\end{Remark}

We can parameterize a stochastic block model using one representative of the equivalence class of kernels, $f$. We simply
consider the function $f$, which is piecewise constant over the blocks, and is defined by\\ $f: [0,1] \times [0,1]  \longrightarrow [0,1]$
\begin{align}
  (x,y) \longmapsto
         \begin{cases}
         p_i & \text{if} \quad \sum_{k=1}^{i-1}s_k  \leq x <  \sum_{k=1}^{i} s_k, \; \\ & \text{and} \quad 
         \sum_{k=1}^{i-1} s_k  \leq y <  \sum_{k=1}^{i} s_k,\\
          q_{ij} & \text{if} \quad \sum_{k=1}^{i-1}s_k  \leq x <  \sum_{k=1}^{i} s_k, \; \\ & \text{and} \quad 
         \sum_{k=1}^{j-1} s_k \leq y <  \sum_{k=1}^{j} s_k.
       \end{cases}
\end{align}
This piecewise constant function is called the canonical kernel of the block model with measure
$\mu$ (see e.g., Fig.~\ref{fig:Ex-f}), and we denote it by $f(x,y,\bp,\bQ,\bs)$. The scaling constant $\omega_n$ controls the rate at which the density of the stochastic block model goes to $0$.
\begin{Definition}
	The probability measure for a stochastic block model kernel with scaling $\omega_n$ and parameters $\bp, \bQ, $ and $\bs$ is denoted by
	\begin{align}
		\mu_{\omega_n,\bp, \bQ, \bs}.
	\end{align}
\end{Definition}

\noindent \textit{Example.} Given $\bm s = \begin{bmatrix} 1/2 & 1/4 & 1/4 & 0  \cdots \end{bmatrix}^T$, the values of $f(x,y; \bm p, \bQ, \bm s)$ in the unit square are shown in Fig. \ref{fig:Ex-f}.
\begin{figure}[H]
\begin{minipage}{.5\textwidth}
	  \centering
	  \includegraphics[clip, trim = 3.5cm 8cm 3.5cm 8.5cm, width= \textwidth]{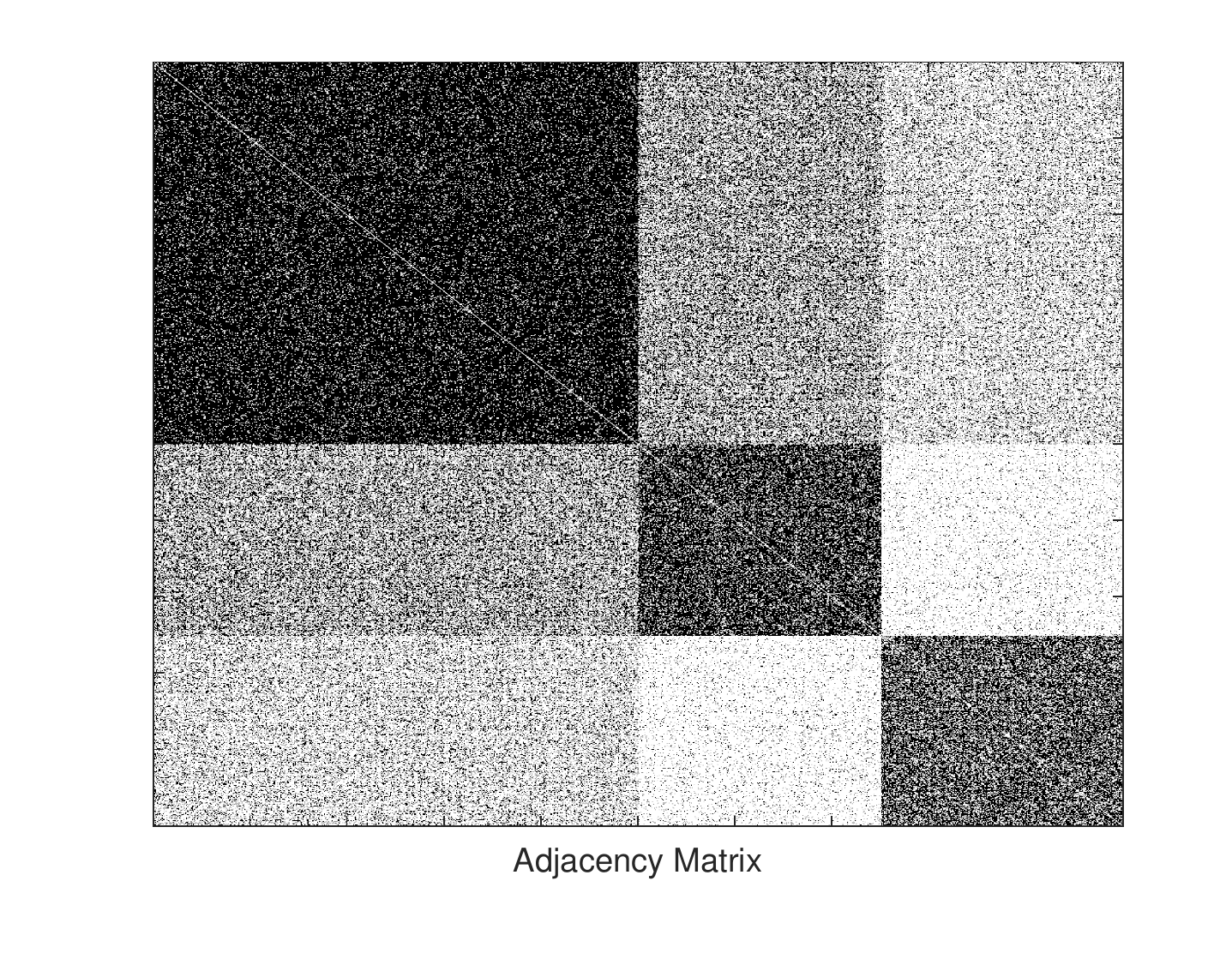}
	  \caption{An example adjacency matrix: $\bA$}
	\end{minipage}%
	\begin{minipage}{.5\textwidth}
	\centerline{
	  \includegraphics[clip, trim = 7cm 3cm 9cm 2cm, width = \textwidth]{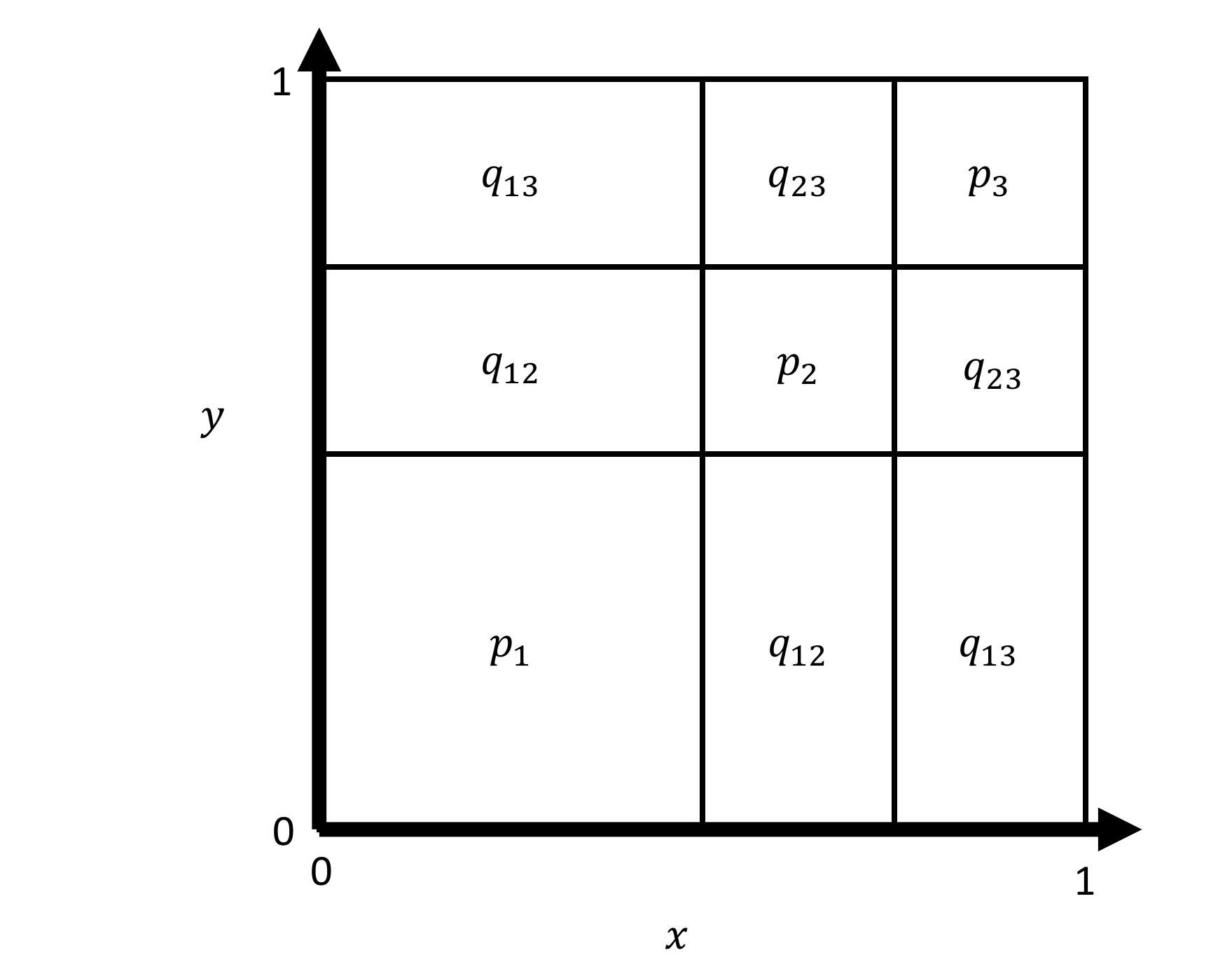}
	  }
	  \caption{An example $f(x,y; \bp, \bm Q, \bs)$}
	  \label{fig:Ex-f}
	\end{minipage}
\end{figure}
\begin{Remark}
	Stochastic block models are presented here in generality for reference. Throughout this work, we always take the non-zero entries of the matrix $\bQ$ to be a constant given by $q$. In this case, the probability measure is denoted by $\mu_{\omega_n, \bp, q, \bs}$.
\end{Remark}
%______________________________________________________________________________________________
\section{Random parameter stochastic block models}
\label{sec:RandParam}
%______________________________________________________________________________________________
In this section we introduce the random parameter stochastic block model as a generalization of the classic stochastic block model. The need for this generalization when considering density estimation results from the following observation. For a fixed geometry vector $\bs$, the limiting distribution of the largest $c$ eigenvalues, $\lambda_i(\bA_{\mu_{\omega_n, \bp, q, \bs}})$ for $i = 1,...,c$, have a dependency between the location and its scale. This fact can be seen clearly in \cite{FK81} where it is shown for an $\text{\ER}$ random graph with parameters $n$ and $p$, that
\begin{align}
	\lambda_1 \overset{d}{\to} N\left( (n-2)p + 1, 2p(1-p) + \cO(\frac{1}{\sqrt{n}})\right). \label{eqn:ExEqn}
\end{align}
Equation (\ref{eqn:ExEqn}) shows that a different choice of $p$ affects both the location and scale of the limiting distribution. The same phenomenon can be seen in previous studies \cite{CCH20, FFHL19, T18}, where the limiting distribution is derived for inhomogeneous $\text{\ER}$ random graphs. The dependency between the location and scale parameters for the largest eigenvalues of stochastic block model graphs is fully expected as a consequence of the Bernoulli process that defines the probability of an edge existing in a graph. The implications of these observations suggest that when attempting to model a graph valued data set, the stochastic block model cannot capture simultaneously notions of location and scale.

By allowing the parameters of a stochastic block model to be random, several degrees of freedom are introduced when considering the distribution of the eigenvalues. Most notably, allowing randomness in the parameter space allows the user to increase the variance of the eigenvalues of the adjacency matrices from a classic stochastic block model while preserving their expected value. We first introduce the random parameter stochastic block model and the distribution of the eigenvalues of graphs sampled in this manner. We then discuss methods by which we estimate the parameters of this model given a sample set of data (see Alg. \ref{alg:DetJ}).

We take a random graph $G_{\mu}$ to be distributed according to a stochastic block model, $\mu_{\omega_n, \bp ,q, \bs}$, with unknown parameter $\bp$ that is distributed according to some distribution $J$. Note that by the definition of $f(x,y)$, see definition \ref{def:f}, the support of $J$ is $[0,1]^c$.

By allowing $\bp$ to vary according to $J$, the result is a distribution over $\cG$ where the intra-community strengths of the graphs sampled follow a multivariate distribution $J$. The associated probability measure given some $J$ is denoted by
\begin{align}
	\mu_{\omega_n, J, q, \bs} \in \cM(\cG),
\end{align}
which is similar to the probability measure for a general stochastic block model except the parameter of intra-community strengths $\bp$ is replaced by $J$. 
\begin{Remark}
\label{rem:dirac}
	Observe that when $J = \delta(\bp - \bp^*)$ for some fixed $\bp^*$, then 
	\begin{align}
		\mu_{\omega_n, J, q, \bs} = \mu_{\omega_n, \bp^*, q, \bs}
	\end{align}
	which is the typical stochastic block model probability measure with intra-community probability of connection given by $\bp^*$.
\end{Remark}

Because our aim is to perform density estimation in $\cG$ with respect to $d_{A_c}$ it is necessary to understand the behavior of the eigenvalues of graphs distributed according to $\mu_{\omega_n, J, q, \bs}$. By mapping $G_{\mu_{\omega_n, J, q, \bs}} \mapsto \sigma_c(\bA_{\mu_{\omega_n, J, q, \bs}})$, the resulting distribution of the largest $c$ eigenvalues is denoted by $H_n$ and has a probability distribution given by
\begin{align}
	p_{H_n}(\blamb) = \int p_{\mu_{\omega_n, J, q, \bs}}(\sigma_c(\bA) \vert \bm P_J = \bp) p_J(\bp) d\bp \label{eqn:EigDist}
\end{align}
where the first moment and normalized second moments of $H_n$ are
\begin{align}
	\E{\blamb}[H_{n}] &= \E{\E{\sigma_c(\bA)\vert \bm P_J = \bp}[\mu]}[J] \label{eqn:TotMean},\\
	\Cov_{H_{n}}\left( \frac{\lambda_i}{\sqrt{\omega_n}}, \frac{\lambda_j}{\sqrt{\omega_n}} \right) &= \Cov_{J}(\E{\frac{1}{\sqrt{\omega_n}}\lambda_i(\bA)  \vert \bm P_J = \bp}[\mu],\E{\frac{1}{\sqrt{\omega_n}} \lambda_j(\bA) \vert \bm P_J = \bp}[\mu])\\
	 & + \E{\Cov_{\mu}\left(\frac{1}{\sqrt{\omega_n}}\lambda_i(\bA) ,\frac{1}{\sqrt{\omega_n}}\lambda_j(\bA)  \vert \bm P_J = \bp \right)}[J]. \label{eqn:TotVar}
\end{align}

Theorem \ref{thm:CCH} and Corollary \ref{cor:FiniteMatrixRep} show two methods to determine the quantities 
\begin{align}
	&\E{\sigma_c(\bA)\vert \bm P_J = \bp}[\mu],\\
	&\Cov_{\mu}\left(\frac{1}{\sqrt{\omega_n}}\lambda_i(\bA) ,\frac{1}{\sqrt{\omega_n}}\lambda_j(\bA)  \vert \bm P_J = \bp \right),
\end{align}
from equations (\ref{eqn:TotMean}) and (\ref{eqn:TotVar}) in terms of the parameters of a stochastic block model.

Let $f$ be a canonical stochastic block model kernel function and let $L_f$ be the associated linear integral operator with eigenfunctions $r_i(x)$ and eigenvalues denoted by $\theta_i = \lambda_i(L_f)$. Assume that $n^{-2/3} \ll \omega_n \ll 1$ and that $\lim_{n \to \infty} \omega_n = 0$. Because $\mu$ is taken to be a stochastic block model kernel probability measure with parameters $\omega_n, \bm p,$ $q,$ and $\bs$, we denote the adjacency matrix of a random graph as $\bA_{\mu} = \bA_{\mu_{\omega_n, \bp, q, \bs}}$ where we have suppressed all the subscripts. 

\begin{theorem}[Chakrabarty, Chakraborty, Hazra 2020]
\label{thm:CCH}
\begin{align}
	\left(\omega_n^{-1/2}(\lambda_i(\bA_{\mu}) - \E{\lambda_i(\bA_{\mu})}[\mu]) \right) \overset{d}{\longrightarrow} (Z_i: 1 \leq i \leq c), \label{eqn:LimVar}
\end{align}
where the right hand side is a multivariate normal random vector in $\R^c$ with zero mean and
\begin{align}
	\Cov(Z_i, Z_j)  = 2 \int_0^1 \int_0^1 r_i(x) r_i(y) r_j(x) r_j(y) f(x,y) dx dy, \label{eqn:CovCCH}
\end{align}
for all $1 \leq i,j \leq c$. The first order behavior of $\E{\lambda_i(\bA_{\mu})})$ is given by the following: For every $1 \leq i \leq c$,
 	\begin{equation}
		\E{\lambda_i(\bA_{\mu_{\rho_n f}})} = \lambda_i(\bm B) + \mathcal{O}(\sqrt \omega_n + \frac{1}{n \omega_n}), \label{eqn:EstLargeEigs-CCH}
	\end{equation}
	where $\bm B$ is a $c \times c$ symmetric deterministic matrix defined by
	\begin{align}
	b_{j,l} = \sqrt{\theta_j \theta_l} n \omega_n \bm e^T_j \bm e_l + \theta_i^{-2} \sqrt{\theta_j \theta_l} (n \omega_n)^{-1} \bm e^T_j \E{(\bA - \E{\bA})^2} \bm e_l + \mathcal{O}(\frac{1}{n \omega_n}),
	\end{align}
	and $\bm e_j$ is a vector with entries $\bm e_j(k) = \frac{1}{\sqrt n} r_j(\frac k n)$ for $1 \leq j \leq c$.
\end{theorem}
\begin{proof}
	This is a compilation of Theorems 2.3 and 2.4 from \cite{CCH20}.
\end{proof}

We offer the following corollary as a minor simplification to the results of \cite{CCH20} which allows for the estimation of $\E{\lambda_i(\bA_{\mu})}$ and the computation of $\Cov(Z_i, Z_j)$ in a different manner when considering stochastic block models. First define the matrices
\begin{align}
		\bm M = 
	\begin{bmatrix} 
		s_1 p_1 & \sqrt{s_1 s_2} q & \dots & \sqrt{s_1 s_c} q\\
		\sqrt{s_2 s_1} q & s_2 p_2 & \dots & \sqrt{s_2 s_c} q\\
		\vdots &  \vdots & \ddots & \vdots \\
		\sqrt{s_c s_1} q & \sqrt{s_c s_2} q & \dots & s_c p_c
	 \end{bmatrix} \quad 
	 	\bm M_{f} = 
	\begin{bmatrix} 
		 p_1 & q & \dots &  q\\
		 q & p_2 & \dots & q\\
		\vdots &  \vdots & \ddots & \vdots \\
		 q &  q & \dots &  p_c
	 \end{bmatrix} \label{eqn:FiniteMatrices}
\end{align}
where $\nu_k$ and $\bm v_k$ are the eigenvalues and eigenvectors of $\bm M$ respectively. 
\begin{corollary}
\label{cor:FiniteMatrixRep}
\begin{align}
	\left(\omega_n^{-1/2}(\lambda_i(\bA_{\mu}) - \E{\lambda_i(\bA_{\mu})}[\mu]) \right) \overset{d}{\longrightarrow} (Z_i: 1 \leq i \leq c), \label{eqn:LimVar}
\end{align}
where the right hand side is a multivariate normal random vector in $\R^c$ with zero mean and
\begin{align}
	\Cov(Z_i, Z_j)  = 2 \left( \bm v_k .* \bm v_j \right)^T \bm M_f \left( \bm v_k .* \bm v_j \right) \label{eqn:FirstOrderVar},
\end{align}
for all $1 \leq i,j \leq c$ and $.*$ denotes the component-wise product of the vectors.

The first order behavior of $\E{\lambda_i(\bA_{\mu})})$ is given by the following: For every $1 \leq i \leq c$,
 	\begin{equation}
		\E{\lambda_i(\bA_{\mu})} = \lambda_i(\bm B^*) + \mathcal{O}(\sqrt{\omega_n}) \label{eqn:FirstOrderMean}
	\end{equation}
	where $\bm B^* =  \bm B^{*,(1)} + \bm B^{*,(2)}$, the components of which are given as
\begin{align}
	\left( \bm B^{*,(1)} \right)_{j,l} &= b_{j,l}^{*,(1)} = \begin{cases}  \nu_j n \omega_n \quad \text{if } j = l \\ 0 \quad \text{if } j \neq l \end{cases} \label{eqn:defB1}\\
	\left( \bm B^{*,(2)} \right)_{j,l} &= b_{j,l}^{*,(2)} = \nu_i^{-2} \sqrt{\nu_j \nu_l}  \sum_{k=1}^c \nu_k  \sum_{m=1}^c \frac{1}{\sqrt{s_m}} \bm v_j(m) \bm v_l(m) \bm v_k(m)\sum_{w=1}^c \sqrt{s_w} \bm v_k(w). \label{eqn:defB2}
\end{align}
\end{corollary}
\begin{cproof}
	We show these results in \ref{app:FiniteRank}. The proof technique observes that the estimation of the expected eigenvalues in Theorem \ref{thm:CCH} is given in terms of the inner products between discretized eigenfunctions and a discretized operator. Because every term is piecewise Lipschitz, these discretizations converge to their respective limits at a rate $\frac 1 n$ which allows one to replace the vectors $\bm e_j$ with the corresponding eigenfunctions. Then, because the eigenfunctions of stochastic block model graphs are piecewise constant, we represent these quantities with the eigenvectors of the matrix $\bm M$.
\end{cproof}

The prior corollary discusses the convergence in distribution of the random variable 
\begin{align}
	\left(\omega_n^{-1/2}(\lambda_i(\bA_{\mu}) - \E{\lambda_i(\bA_{\mu})}[\mu]) \right).
\end{align}
In general, we do not have convergence of the second moment as a result of convergence in distribution. However, we assume that for a large value of $n$, the quantity $\Cov(Z_i, Z_j)$ provides a good estimate of the term 
\begin{align}
	\Cov_{\mu}\left(\frac{1}{\sqrt{\omega_n}}\lambda_i(\bA) ,\frac{1}{\sqrt{\omega_n}}\lambda_j(\bA)  \vert \bm P_J = \bp \right).
\end{align}
With this assumption, Corollary \ref{cor:FiniteMatrixRep} provides a relationship between the expected value and covariance of $H_n$ in terms of the parameters of the stochastic block model. 

For general stochastic block models, equations (\ref{eqn:FirstOrderVar}) and (\ref{eqn:FirstOrderMean}) are not analytic in terms of the parameters of the model even when $q$ is taken to be constant. A consequence is that the estimation of the moments of $J$ becomes non-trivial. We consider a regime where we define $p_{\min} = \underset{i = 1,...,c}{\min} \bp$ along with a fixed parameter $\epsilon \ll 1$ and set $q = \epsilon p_{\min}$. Under these conditions, analytic expressions for (\ref{eqn:FirstOrderVar}) and (\ref{eqn:FirstOrderMean}) are determined up to $\cO(\epsilon^2)$ in the following lemma. The implications of these analytic expressions results in a subsequent lemma to determine the first and second moments of $J$.
\begin{Lemma}
\label{lem:FirstOrder}
	Let $\bm M$ and $\bm M_f$ be as defined in equation (\ref{eqn:FiniteMatrices}). Assume $q = \epsilon p_{\min}$, where $p_{\min} = \underset{i = 1,...,c}{\min}\bp$ and $\epsilon \ll 1$.
	\begin{align}
	&\left \vert \frac{\E{\lambda_i(\bA)}}{ n \omega_n s_i} - p_i \right \vert = \cO\left(\epsilon^2 p^2_{\min}\right) + \cO \left(\frac{t_i(\bp)}{n \omega_n s_i}\right), \label{eqn:FirstOrderErrMean}\\
	&\left \vert \Cov(Z_i, Z_j) - \begin{cases}2 p_i \quad \text{if } i =j\\ 0 \quad \text{if } i \neq j\end{cases}\right \vert = \cO(\epsilon^2 p_{\min}^2), \label{eqn:FirstOrderErrVar}
\end{align}
where $t_i(\bp)$ is a bounded function of the parameters $\bp$ and $Z$ is defined in Theorem \ref{thm:CCH}.
\end{Lemma}
\begin{lproof}
	The proof is in \ref{app:FirstOrder}
\end{lproof}

When conditioned on an observation $\bp$, Lemma \ref{lem:FirstOrder} shows how to determine a relative error in the parameters. Utilizing these results, we now show how to determine the relative error in the first two moments of $J$ with the following lemma.
\begin{Lemma}
\label{lem:RelErrJMoments}
Let $\bm P_J$ be an observation from $J$ with components $P_i$, let $P_{\min} = \underset{i=1,...,c}{\min} P_i$ and define $q = \epsilon P_{\min}$ where $\epsilon \ll 1$.
\begin{align}
	\frac{\E{\lambda_i}[H_n]}{n \omega_n s_i} &= \E{P_i}[J] + \cO(\epsilon^2 ) + \cO\left(\frac{1}{n \omega_n}\right) \label{eqn:First-E}\\
 	\frac{\Cov_{H_{n}}(\lambda_i, \lambda_j)}{n^2 \omega_n^2 s_i s_j}&+ \begin{cases} \frac{2\E{\lambda_i}[H_n]}{n^3 \omega_n^2 s_i^3} \quad \text{if } i = j \\ 0 \quad \text{if } i \neq j \end{cases}\\
 	 &= \Cov_{J}\left(P_i + \cO(\epsilon^2 P_{\min}^2) + \cO\left(\frac{t_i(\bp)}{n \omega_n }\right),  P_j + \cO(\epsilon^2 P_{\min}^2) + \cO\left(\frac{t_j(\bp)}{n \omega_n}\right)\right) +  \cO\left(\frac{\epsilon^2}{n^2 \omega_n }\right)
\end{align}
where $t_i(\bp)$ is a bounded function of the parameters for each $i$.
\end{Lemma}
\begin{lproof}
	The proof is in \ref{app:RelErr}.
\end{lproof}
When the random parameter stochastic block model has a parameter $q$ that satisfies the assumptions of Lemma \ref{lem:RelErrJMoments}, the associated probability measure is denoted as $\mu_{\omega_n, J, \epsilon, \bs}$. Lemma  \ref{lem:RelErrJMoments} suggests the following algorithm in practice to estimate the probability measure $\mu_{\omega_n, J, \epsilon, \bs}$ given a sample set of graphs.
\begin{algorithm}[H]
  \caption{Estimation of $\mu_{\omega_n, J, \epsilon, \bs}$ given sample data}
  \label{alg:DetJ}
  \begin{algorithmic}[1]
    \Require Set of graphs, $M = \lbrace \Gk \rbrace_{k=1}^N$.
    \State Compute the eigenvalues of the adjacency matrix of each graph as $\blamb^{(k)} = \sigma_c(\bA^{(k)})$.
    \State Compute the arithmetic average of the $c$ largest eigenvalues as $\bar{\blamb}$.
    \State Compute the sample covariance matrix of the $c$ largest eigenvalues as $\hat{\Sigma}$.
    \State Determine the average density of the set of graphs as $\bar{\rho}_n$.
    \State Set $\bs$ such that $s_i = \frac 1 c$ for all $1 \leq i \leq c$.
    \State Set $\omega_n = C \bar{\rho}_n$. Taking $C =1$ is typically sufficient, a discussion follows in Remark \ref{rem:geom_density}.
    \State Determine the first moment of $J$ according to 
    \begin{align}
    	\E{P_i}[J] = \frac{\bar{\lambda}_i}{n \omega_n s_i}.
    \end{align}
    \State Determine the second moment of $J$ according to
    \begin{align}
    	\Cov_J(P_i,P_j) = \frac{\hat{\Sigma}_{ij}}{n^2 \omega_n^2 s_i s_j} - \begin{cases} \frac{2 \bar{\lambda}_i}{n^3 \omega_n^2 s_i^3} \quad \text{if } i = j \\0 \quad \text{if } i \neq j \end{cases} \label{alg:step8}
    \end{align}
    \State Set 
    \begin{align}
    	\epsilon = \frac{1 - \sum_{i = 1}^c \E{P_i}[J] s_i^2}{ \underset{i = 1,...,c}{\min}( \E{P_i}[J]) (1 - \sum_{i=1}^c s_i^2)}
    	\end{align}
    	\State Return: $\omega_n, J, \epsilon, \bs$
  \end{algorithmic}
\end{algorithm}
\begin{Remark}
\label{rem:geom_density}
The choice of scaling as $\omega_n = \bar{\rho}_n$ is for simplicity. In practice, taking $\omega_n = C\bar{\rho}_n$, where $C$ is a positive constant, will always yield the same result for the following reason. Observe that for graphs sampled according to a kernel probability measure, the probability of an edge existing in the graph is modeled by
	\begin{align}
		\prob{a_{ij}} = \bern{\omega_n f(\frac i n,\frac j n)} = \bern{\left(C \omega_n\right) \frac{f(\frac i n,\frac j n)}{C}}.
	\end{align}
	Implying that the choice $\tilde{\omega}_n = C \omega_n$ and $\tilde{f}(x,y) = \frac 1 C f(x,y)$ for any $C$ such that $\tilde{f}(x,y) \in [0,1]$ defines an equivalent probability measure, $\mu_{\omega_n f} = \mu_{\tilde{\omega}_n \tilde f}$.
		
The choice of $\epsilon$ is such that in expectation, the density of graphs sampled from the estimated probability distribution is equal to the sample density of graphs observed. It is derived by setting $\E{||f||_1}[J] = 1$. Note that when taking $\omega_n = C\bar{\rho}_n$, $\epsilon$ should be chosen such that $\E{||f||_1}[J] = \frac{1}{C}$ so the expected density of the graphs is preserved when this is a desired quantity. The choice of $\omega_n$ and $\epsilon$ is a suggestion, other quantities may exist to suggest a different method of selecting these parameters of the model.

The work in \cite{FM22} shows that for arbitrarily large graphs, taking $s_i$ as a constant allows for the determination of a disconnected stochastic block model with the correct expected eigenvalues. We discuss in Remark \ref{rem:geom} that in practice, an estimate of $\bs$ that is data set dependent may perform better.
\end{Remark}

When fitting a distribution to sample data, aligning the sample moments with the mean and variance of the probability measure is a common practice and is the approach taken when fitting a random parameter stochastic block model. In the following section, it is shown that both $\bar{\bm \lambda}$ and $\hat{\Sigma}$ provide a good estimate of the eigenvalues of the sample Fr\'echet mean graph and the information contained within the total sample Fr\'echet variance respectively. Another perspective common to parameter estimation is likelihood maximization which is not explored in this paper.
%______________________________________________________________________________________________
\section{The first and second moments of $\mu \in \cM(\cG)$}
\label{sec:GraphMoments}
%______________________________________________________________________________________________
The first and second moments of a distribution $\mu \in \cM(\cG)$ respectively characterize the mean and spread of the probability measure. The mean and variance for metric spaces was generalized in \cite{F48}, and are respectively called the Fr\'echet mean and total Fr\'echet variance along with their sample alternatives. In this section, the Fr\'echet mean and total Fr\'echet variance are introduced and it is shown that the arithmetic mean of the eigenvalues of the adjacency matrices of a set of graphs, $\sGk_{k=1}^N$, and the sample covariance of the eigenvalues are closely related to the sample alternatives of the Fr\'echet mean and total Fr\'echet variance when $n$ is large.
%______________________________________________________________________________________________
\subsection{The first moment: Fr\'echet mean}
\label{subsec:FM}
%______________________________________________________________________________________________
We equip the set of graphs, $\cG$, defined on $n$ vertices (see definition \ref{def:SetOfGraphs}) with the pseudometric defined by the $\ell_2$ norm between the spectra of the respective adjacency matrices, $d_{A_c}$, (see (\ref{distance-trunc})). We consider a probability measure $\mu \in \cM(\cG)$ that describes the probability of obtaining a given graph when we sample $\cG$ according to $\mu$. Using $d_{A_c}$, we quantify the spread of the graphs, and we define a notion of centrality, which gives the location of the expected graph, according to $\mu$.

\begin{Definition}[Fr\'echet mean \cite{F48}]
  The Fr\'echet mean of the probability measure $\mu$ in the pseudometric space $(\cG,d_{A_c})$ is the set of graphs $G^*$ whose expected distance squared to $\cG$ is the minimum,
  \begin{equation}
    \left\{G^* \in \cG\right\} =  \argmin{G \in \cG} \E{d_{A_c}^2(G,G_{\mu})}[\mu], \label{def:FM}
    \end{equation}
\end{Definition}
where $G_\mu$ is a random realization of a graph distributed according to the probability measure $\mu$, and the expectation $\E{d^2(G,G_{\mu})}[\mu]$ is computed with respect to the probability measure $\mu$. The analysis in this section applies to any graph in the Fr\'echet mean set. We therefore assume that the Fr\'echet mean is unique and present the sample Fr\'echet mean as a unique graph rather than the more general set valued sample Fr\'echet mean.

The sample Fr\'echet mean is a natural extension of the population Fr\'echet mean that is defined by replacing the measure $\mu$ with the empirical measure. 

\begin{Definition}[Sample Fr\'echet mean \cite{F48}] 
  Let $\left\{\Gk\right\}\; 1 \leq k \leq N$ be a set of graphs in $\cG$. The sample Fr\'echet mean is defined by
  \begin{align}
    G_N^* = \argmin{G \in \cG} \frac 1 N \sum_{k=1}^N d_{A_c}^2(G,\Gk). \label{eqn:EFM}
  \end{align}
\end{Definition}

	\begin{Remark}
	The sample Fr\'echet mean, when consider the distance $d_{A_c}$, is discussed at length in \cite{FM22}. Of particular note, the largest $c$ eigenvalues of the adjacency matrix of the sample Fr\'echet mean graph are shown to cluster about the arithmetic average of the largest $c$ eigenvalues from the sample, namely Theorem 2 in \cite{FM22}.
	\end{Remark}
%______________________________________________________________________________________________
\subsection{The second moment: Fr\'echet variance}
\label{subsec:FV}
%______________________________________________________________________________________________
With a notion of mean in hand, the second moment, which captures the variability about the mean, follows naturally.
\begin{Definition}[Total Fr\'echet variance \cite{F48}]
	The total Fr\'echet variance of the probability measure $\mu$ in the pseudometric space $(\cG,d_{A_c})$ is defined  as
  \begin{equation}
    V^*_{tot} =  \E{d_{A_c}^2(G^*,G_{\mu})}[\mu], \label{def:TFV}
    \end{equation}
\end{Definition}
which is the evaluation of the Fr\'echet mean objective at the Fr\'echet mean graph. Similarly, the sample total Fr\'echet variance is given by evaluating the sample Fr\'echet mean objective at the sample Fr\'echet mean graph.
\begin{Definition}[Sample total Fr\'echet variance \cite{F48}]
	Let $\left\{\Gk\right\}\; 1 \leq k \leq N$ be a set of graphs in $\cG$. The sample total Fr\'echet variance is defined as
  \begin{align}
    V_{N, tot}^* =  \frac{1}{N-1} \sum_{k=1}^N d_{A_c}^2(G_N^*,\Gk). \label{eqn:STFV}
  \end{align}
\end{Definition}

While the (sample) total Fr\'echet variance applies for any metric object (all it requires is a distance), the adjacency spectral pseudo-metric, $d_{A_c}$, allows for a more specific definition of variance. In fact, the covariance matrix of the observed eigenvalues captures identical information to (\ref{eqn:STFV}) as shown in the following lemma.

For a dataset of graphs $\lbrace \Gk \rbrace_{k=1}^N$, denote the arithmetic mean of the eigenvalues of the adjacency matrices as
\begin{equation}
	\bar{\blamb} = \frac 1 N \sum_{k=1}^N \blamb^{(k)}
\end{equation}
where $\blamb^{(k)}$ is the vector of $c$ largest eigenvalues of the adjacency matrix, $\bA^{(k)}$, of graph $\Gk$. Define the sample covariance matrix as
\begin{align}
	\hat \Sigma = \frac{1}{N-1} \sum_{k=1}^N (\blamb^{(k)} - \bar{\blamb}) (\blamb^{(k)} - \bar{\blamb})^T.
\end{align}

\begin{Lemma}
\label{lem:STFV_Cov}
	Let $\lbrace \Gk \rbrace_{k=1}^N$ be a sample of graphs with sample Fr\'echet mean $G_N^*$ and total sample Fr\'echet variance $V_{N,tot}^*$. Let $\bar \blamb$ denote the arithmetic mean of the largest $c$ eigenvalues, and let $\hat{\Sigma}$ be the sample covariance matrix; then,
	\begin{equation}
		\lim_{n \to \infty} \vert V_{N, tot}^* - \sum_{i=1}^c \hat{\Sigma}_{ii}\vert = 0
	\end{equation}
\end{Lemma}
\begin{lproof}
	The proof is in \ref{app:ApproxFV}.
\end{lproof}
%______________________________________________________________________________________________
\section{Conditions for a feasible distribution $J$}
\label{sec:ProbDensityEst}
%______________________________________________________________________________________________
This section explores the situation in which steps 7 and 8 in Algorithm \ref{alg:DetJ} are solvable independent of any assumptions on $J$. First, we consider Step 7. Because the support of $J$ is a subset of $[0,1]^c$,
\begin{align}
	\frac{\bar{\lambda}_i}{n \omega_n s_i}  \in [0,1] \quad \forall i. \label{eqn:GeomFeas}
\end{align}
\begin{Remark}
\label{rem:geom}
While it is shown in \cite{FM22} that the size of the communities is arbitrary when the size of the graphs is arbitrarily large, in practice, this condition suggests that an estimate of $s_i$ that is data set dependent may perform better. Unless otherwise specified in this manuscript, we always take $\bs$ such that $s_1 \geq s_j$ for all $j = 2,...,c$ and $s_i = s_j$ for all $2 \leq i,j \leq c$.
\end{Remark}
Interpreting this condition is straightforward; it states that the largest eigenvalue from a block in the stochastic block model cannot exceed the total number of connections available within that block. 

 Another condition on the feasibility of the moments of $J$ comes from Step 8. An obvious observation is simply that the variance is bounded below by $0$,
\begin{align}
	0 &\leq Var_J(P_i) \implies 0 \leq \frac{\hat{\Sigma}_{ii}}{n^2 \omega_n^2 s_i^2} -\frac{2 \bar{\lambda}_i}{n^3 \omega_n^2 s_i^3 }.
\end{align}
This results in the following condition on the relationship between the sample variance and the sample mean eigenvalues;
\begin{align}
	0 &\leq  \frac{\hat{\Sigma}_{ii}}{\omega_n} -\frac{2 \bar{\lambda}_i}{n \omega_n s_i } \iff 0 \leq \hat{\Sigma}_{ii} -\frac{2 \bar{\lambda}_i}{n s_i } \label{eqn:CondHom}
\end{align}
where we have expressed the condition in two forms, the first for interpretability and the second motivated by practical implementation.

We consider that a data set of graphs, $\sGk_{k=1}^N$, is considered homogeneous when the data set has low total sample Fr\'echet variance and is considered heterogeneous when the data set has high total sample Fr\'echet variance. The condition in equation (\ref{eqn:CondHom}) is related to the homogeneity of the sample set of graphs. It relates the scaled sample variance $\hat{\Sigma}_{ii}$ to the expected variance inherent to the stochastic block model process. Recall that the expected value of the inherent variance of the stochastic block model when $q = \epsilon p_{\min}$ is estimated to the first order as
\begin{align}
	\E{\Var(Z_i)}[J] = \E{2P_i}[J] + \cO(\epsilon^2)
\end{align}
where $Z$ is defined as in Corollary \ref{cor:FiniteMatrixRep}. By rewriting $\E{P_i}[J] = \frac{\bar{\lambda}_i}{n \omega_n s_i}$, it is clear that the condition given in equation (\ref{eqn:CondHom}) is equivalent to
\begin{align}
0 &\leq  \frac{\hat{\Sigma}_{ii}}{\omega_n} - \E{\Var(Z_i)}[J].
\end{align}

If the sample variance of the eigenvalues, $\hat{\Sigma}_{ii}$, is smaller than $\E{\Var(Z_i)}[J]$, the data set, $\sGk_{k=1}^N$, is considered to be more homogeneous with respect to the eigenvalues than a set of graphs sampled according to a stochastic block model with parameters $p_i = \frac{\bar{\lambda}_i}{n \omega_n s_i}$ and $\omega_n = \bar{\rho}_n$.

At times, the term $\frac{2 \bar{\lambda}_i}{n s_i }$ will be negligible when compared to $\hat{\Sigma}_{ii}$. This observation, along with equation (\ref{eqn:CondHom}) leads to the following three regimes of variance that may be considered when fitting a random parameter stochastic block model to a data set of graphs.

%____________________________________________________________________________________________
\subsection{Regimes of variance}
\label{subsec:VarReg}
%____________________________________________________________________________________________
Let $\lbrace \Gk \rbrace_{k=1}^N$ be a sample set of graphs where $\bar{\blamb}$ and $\hat{\Sigma}$ respectively denotes the arithmetic mean of the largest $c$ eigenvalues of the adjacency matrices and the sample covariance of the eigenvalues about the arithmetic mean. When 
\begin{align}
	\hat{\Sigma}_{ii} < \frac{2 \bar{\lambda}_i}{n s_i}
\end{align}
 the $i$-th largest eigenvalue is said to live in the \textit{small variance regime} and there does not exist a distribution $J$ such that $\mu_{\omega_n, J, q, \bs}$ captures both the sample mean eigenvalue and the sample variance of the eigenvalues.
 
The \textit{large variance regime} occurs when the term $\frac{2 \bar{\lambda}_i}{n s_i}$ is negligible as compared to $\hat{\Sigma}_{ii}$. For each $i = 1,...,c$, we determine whether the $i$-th eigenvalue is in this regime by examining whether 
\begin{align}
	\frac{2 \bar{\lambda}_i}{n s_i} \ll \hat{\Sigma}_{ii}.
\end{align}
When the data set of graphs falls within this regime, omitting the contribution to variance from $\frac{2 \bar{\lambda}_i}{n s_i}$ as
\begin{align}
	Var_J(P_i) &= \frac{\hat{\Sigma}_{ij}}{n^2 \omega_n^2 s_i^2} - \frac{2\bar{\lambda}_i}{n^3 \omega_n^2 s_i^3} \\
	&\approx \frac{\hat{\Sigma}_{ij}}{n^2 \omega_n^2 s_i^2}
\end{align}
provides a reasonable estimate for the variance of $J$. In this case, equation (\ref{alg:step8}) in Step 8 of Alg. \ref{alg:DetJ} simplifies to
\begin{align}
    	\Cov(P_i,P_j) = \frac{\hat{\Sigma}_{ij}}{n^2 \omega_n^2 s_i s_j} 
\end{align}

The \textit{medium variance regime} occurs when the term $\E{2 P_i}[J] = \frac{2 \bar{\lambda}_i}{n s_i}$ is significant when compared to $\hat{\Sigma}_{ii}$, which is the case when
\begin{align}
	\frac{2 \bar{\lambda}_i}{n s_i\hat{\Sigma}_{ii}} = \cO(1).
\end{align}
When this is the case, the distribution $J$ will only add a minor contribution to the total variance of the $i$-th largest eigenvalue. In this situation, the data may be interpreted as being accurately modeled by a classic stochastic block model and the need for a distribution on the parameters is diminished.

\subsubsection{Summary}
When $J$ comes from a two parameter family of distributions, resolving the first and second moments precisely identifies the distribution $J$. In some cases, when $J$ has more than two parameters, further moments may need to be considered.

When possible, taking $J$ to be parametric is preferred primarily because of the faster convergence rates when estimating parametric probability density functions. Often it is the case that for large graphs, the number of samples, $N$, is small and thus a fast convergence rate with respect to $N$ is desirable. However, when $J$ is non-parametric we suggest an alternative approach to estimate the probability density from which the graphs $\lbrace \Gk \rbrace_{k=1}^N$ were sampled.

%______________________________________________________________________________________________
\subsection{Non-parametric density estimation}
\label{subsec:Nonparam}
%______________________________________________________________________________________________
Assume that $J$ is a non-parametric distribution over the parameters $\bp$. Let $\bm P_J$ be an observation from $J$ with components $P_i$, let $P_{\min} = \underset{i=1,...,c}{\min} P_i$ and define $q = \epsilon P_{\min}$, where $\epsilon \ll 1$. For a sample set of graphs $\lbrace \Gk \rbrace_{k=1}^N$, assume that each $\Gk$ is sampled iid from a random parameter stochastic block model, $\mu_{\omega_n, J,\epsilon,\bs}$. Rather than determining $J$ via the moments, we suggest an alternative approach to the estimation of the probability density which is a generalization of kernel density estimation. We define an estimate of $\mu_{\omega_n J,\epsilon,\bs}$ as
\begin{align}
	\hat{\mu} = \frac 1 N \sum_{k=1}^N  \mu_{\omega_n^{(k)}, J^{(k)}, \epsilon^{(k)}, \bs^{(k)}}, \label{eqn:KernDensEst}
\end{align}
where the probability measures $\mu_{\omega_n^{(k)},J^{(k)}, \epsilon^{(k)}, \bs^{(k)}}$ act analogously to kernels when performing kernel density estimation and the distribution $J^{(k)}$ can be interpreted as the bandwidth parameter which determines the variance of each kernel.

To determine $\mu_{\omega_n^{(k)}, J^{(k)}, \epsilon^{(k)}, \bs^{(k)}}$ for each $k$, we suggest the following algorithm, which is closely related to Alg. \ref{alg:DetJ}. The difference is in the determination of the second moment of $J^{(k)}$.

\begin{algorithm}[H]
  \caption{Determination of  $\hat{\mu}$ given sample data}
  \label{alg:DetJKDE}
  \begin{algorithmic}[1]
    \Require Set of graphs, $M = \lbrace \Gk \rbrace_{k=1}^N$.
    \State Compute the eigenvalues of the adjacency matrix of each graph as $\blamb^{(k)} = \sigma_c(\bA^{(k)})$.
    \State Define the set $\lbrace \blamb^{(k)} \rbrace_{k=1}^N \subset \R^c$.
    \State Consider the $c$ largest eigenvalues as a subset of $\R^c$ and perform kernel density estimation. Define
    \begin{align}
	\widehat{T}_{\bm H}(\blamb) = \frac 1 N \sum_{k=1}^N K^{(k)}_{\bm H}(\blamb). \label{eqn:KernR}
    \end{align}
    as the kernel density estimator given $\lbrace \blamb^{(k)} \rbrace_{k=1}^N \subset \R^c$ where each kernel $K^{(k)}_{\bm H}(\blamb)$ is a probability density function that satisfies
    \begin{align}
    		\E{\blamb}[K^{(k)}_{\bm H}] = \blamb^{(k)} \quad \forall k = 1,...,N\\
    		\Cov_{K^{(k)}_{\bm H}}(\lambda_i, \lambda_j) =  \bm H_{ij} \quad \forall k = 1,...,N
    \end{align}
    where $\bm H$ is estimated at the choice of the researcher. For an overview of methods to estimate $\bm H$ see \cite{KDE}.
    \State For each $k$,
    \State \hspace{0.5cm} Compute the eigenvalues of the adjacency matrix of each graph as $\blamb^{(k)} = \sigma_c(\bA^{(k)})$.
    \State \hspace{0.5cm} Determine the density of each graph as $\rho^{(k)}_n$.
    \State \hspace{0.5cm} Set $\bs^{(k)}$ such that $s_i^{(k)} = \frac 1 c$ for all $1 \leq i \leq c$.
    \State \hspace{0.5cm} Set $\omega_n^{(k)} = \rho^{(k)}_n$.
    \State \hspace{0.5cm} Determine the first moment of $J^{(k)}$ according to 
    \begin{align}
    	\E{P_i}[J^{(k)}] = \frac{\lambda^{(k)}_i}{n \omega^{(k)}_n s^{(k)}_i}.
    \end{align}
    \State \hspace{0.5cm} Determine the second moment of $J^{(k)}$ according to
    \begin{align}
    	\Cov_{J^{(k)}}(P_i,P_j) = \frac{\bm H_{ij}}{n^2 (\omega^{(k)}_n)^2 s^{(k)}_i s^{(k)}_j} - \begin{cases} \frac{ 2\lambda^{(k)}_i}{n^3 (\omega^{(k)}_n)^2 (s^{(k)}_i)^3} \quad \text{if } i = j \\0 \quad \text{if } i \neq j \end{cases} \label{eqn:DetJ2}
    \end{align}
    \State \hspace{0.5cm} Set 
    \begin{align}
    	\epsilon^{(k)} = \frac{1 - \sum_{i = 1}^c \E{P_i}[J^{(k)}] (s_i^{(k)})^2}{ \underset{i = 1,...,c}{\min}( \E{P_i}[J^{(k)}]) (1 - \sum_{i=1}^c (s_i^{(k)})^2)}
    	\end{align}
	\State Return: $\hat{\mu} = \frac 1 N \sum_{k=1}^N \mu_{\omega_n^{(k)}, J^{(k)}, \epsilon^{(k)}, \bs^{(k)}}$.
  \end{algorithmic}
\end{algorithm}
	Two new conditions for the feasibility of each $J^{(k)}$ arise for Alg. \ref{alg:DetJKDE}, namely
	\begin{align}
		&\frac{\lambda_i^{(k)}}{n \omega^{(k)}_n s^{(k)}_i}  \in [0,1], \label{eqn:CondKDE1}\\
		&0 < \bm H_{ii} - \frac{2 \lambda^{(k)}_i}{n s^{(k)}_i}. \label{eqn:CondKDE2}
	\end{align}
	Condition (\ref{eqn:CondKDE2}) is explored in depth in subsection \ref{subsec:ExplorationCriticalN}.
\begin{Remark}
\label{rem:Coarse}
	In practice, there are two common properties of graph valued data sets to note. First, for a large value of $n$, the inherent variance of the largest eigenvalues from a stochastic block model is small. Therefore, for heterogeneous data sets of graphs, we expect a small contribution to the total variance from the inherent variance of the stochastic block model. This suggests that the term $\frac{2 \lambda^{(k)}_i}{n s^{(k)}_i}$ is small in practice and that equation (\ref{eqn:CondKDE2}) is likely to be satisfied.
	
	Second, for large graphs, the data is typically coarse in that $N \ll n$. Given the coarse data set in $\cG$, the expectation is that the data, $\lbrace \blamb^{(k)} \rbrace_{k=1}^N \subset \R^c$, is also coarse. When performing kernel density estimation in $\R^c$, for a coarse data set, the estimation of $\bm H$ will be large, see for instance Silverman's rule of thumb when $N$ is small while $\hat{\Sigma}$ is large.
	
	With the expectation that $\bm H_{ii}$ is large and $\frac{2 \lambda^{(k)}_i}{n s^{(k)}_i}$ is small, equation (\ref{eqn:CondKDE2}) is expected to be satisfied for a wide variety of graph valued data.
	
	For each $i$ and $k$ where equation (\ref{eqn:CondKDE2}) is not met, we take the $i$-th marginal distribution of $J^{(k)}$ to be a Dirac delta distribution centered at $\frac{\lambda_i^{(k)}}{n s_i^{(k)}}$ and accept a local over-smoothing\footnote{Smoothing is a symptom of bandwidth selection and is relative to the choice of bandwidth selection in $\R^c$. Remark \ref{rem:Coarse} states that the kernel density estimator will over-smooth the $i$-th eigenvalue at the $k$-th data point as compared to the method chosen for kernel density estimation in $\R^c$, defined by equation (\ref{eqn:KernR}).} of the distribution of the $i$-th eigenvalue at the $k$-th data point.
\end{Remark}
%______________________________________________________________________________________________
\section{Simulation studies}
\label{sec:Sim}
%______________________________________________________________________________________________
Throughout this section we fit a random parameter stochastic block model to various datasets taking either a parametric or non-parametric approach. We first verify that when the data comes from a random parameter stochastic block model, we recover the parameters of the model well. We then test the limits of the model when the data is not sampled according to a random parameter stochastic block model. Within each section, we describe how the data set of graphs was generated and showcase the quality of the estimated probability density.
%______________________________________________________________________________________________
\subsection{Recoverability}
\label{subsec:Recov}
%______________________________________________________________________________________________
We first verify that when data is sampled according to a random parameter stochastic block model, Alg. \ref{alg:DetJ} recovers the parameters well. Let $\lbrace \Gk \rbrace_{k=1}^N$ be a data set of graphs sampled according to a random parameter stochastic block model, $\mu_{\omega_n J, \epsilon, \bs}$, where
\begin{align}
	N = 50, \quad n = 1000, \quad \omega_n = 10n^{-1/2}, \quad \epsilon = 0.05, \quad \bs = [0.5, 0.5, 0, ...]^T, \quad J = \mathcal{U}_{[0.8,0.9]} \times \mathcal{U}_{[0.55,0.6]}.
\end{align}
Here, $\mathcal{U}_{[0.8,0.9]}$ and $\mathcal{U}_{[0.55,0.6]}$ denote uniform probability measures on the respective intervals. We represent the support of $J$ by determining the centers of each uniform distribution as, $\bp^* = [0.85, 0.55]$, and capturing the width about the mean values by the vector $\bm{\delta}^* = [0.1, 0.0.5]$. We denote the estimated parameters by $\hat{\omega}_n, \hat{\bp}^*,$ and $\hat{\bm \delta}^*$. 
\begin{Remark}
When determining $J$ we aim to recover the product of $\omega_n \bp^*$ and $\omega_n \bm{\delta}^*$, which defines the support of the observations from $J$ when $n = 1000$. If the estimate for $\hat{\omega}_n = \omega_n$, then the direct comparison of $\hat{\bp}^*$ and $\hat{\bm \delta}^*$ is equivalent, but because we take $\hat{\omega}_n = \bar{\rho}_n$ these comparisons do not provide useful information. This analysis is unique to the case of recoverability; in general, we will not have oracle knowledge of the parameters that generated the data, and we will only report the error in terms of the sample Fr\'echet mean and sample total Fr\'echet variance.
\end{Remark}

Let $\blamb^{(k)} = \sigma_c(\bA^{(k)})$, where $\bA^{(k)}$ denotes the adjacency matrix of graph $\Gk$. We compute the sample mean spectrum and sample covariance matrix as
\begin{align}
	\bar \blamb &= \frac 1 N \sum_{k=1}^N \blamb^{(k)}, \quad \hat{\Sigma} = \frac{1}{N-1} \sum_{k=1}^N (\blamb^{(k)} - \bar \blamb)^T(\blamb^{(k)} - \bar \blamb).
\end{align}
The values of the mean eigenvalue and diagonal entries of the covariance are
\begin{align}
	\bar \blamb = \begin{bmatrix} 133.9401\\ 91.1005 \end{bmatrix} \quad \hat{\Sigma}_{11} = 25.3956 \quad \hat{\Sigma}_{22} = 5.5628
\end{align}
We assume a constant community size, so $s_i = 0.5$ for each $i$, and check which regime of variance each eigenvalue falls within by computing $\hat{\Sigma}_{ii} - \frac{2 \bar{\lambda}_i}{n s_i}$ for each $i$:
\begin{align}
	\begin{bmatrix}
		\hat{\Sigma}_{11} - \frac{2 \bar{\lambda}_1}{n s_1}\\ 
		\hat{\Sigma}_{22} - \frac{2 \bar{\lambda}_2}{n s_2}
	\end{bmatrix}
	= 
	\begin{bmatrix}
		24.8599 \\ 
		5.1984
	\end{bmatrix}.
\end{align}
The magnitude of each term indicates that each eigenvalue is in the large variance regime and the contribution to variance from $ \frac{2 \bar{\lambda}_i}{n s_i}$ can be ignored. Determining the parameters according to Alg. \ref{alg:DetJ},
\begin{align*}
	&\bar{\rho}_n \hat{p}_1^*= 0.2679, \quad \bar{\rho}_n \hat{p}_2^* = 0.1822, \quad \bar{\rho}_n \hat{\delta}_1 = 0.0173, \quad \bar{\rho}_n \hat{\delta}_2 = 0.0079.
\end{align*}
The absolute relative error for each component of each parameter is
\begin{align*}
	&\left \vert \frac{\bar{\rho}_n  \hat{p}_1^* - \omega_n p_1^*}{\omega_n p_1^*}\right \vert = 0.2455\% \quad \left \vert \frac{\bar{\rho}_n  \hat{p}_2^* - \omega_n p_2^*}{\omega_n p_2^*}\right \vert= 0.0673\%\quad \left \vert \frac{\bar{\rho}_n \hat{\delta}_1- \omega_n \delta_1}{\omega_n \delta_1}\right \vert = 9.2371\% \quad \left \vert \frac{\bar{\rho}_n \hat{\delta}_2- \omega_n \delta_2}{\omega_n \delta_2}\right \vert = 0.0958\%\\	
\end{align*}
The most notable error is in the recovery of $\epsilon$. To compare this term, we consider
\begin{align}
	\left \vert \frac{\hat{\epsilon}  -  \epsilon }{ \epsilon}\right \vert = 83.023\%
\end{align}
which is a significant relative error; however, recall that
\begin{align}
	\frac{\E{\lambda_i(\bA) \vert \bm P_{J} = \bm p}}{ n \omega_n s_i}  = n \omega_n s_i \left( p_i + \cO\left(\epsilon^2 p^2_{\min}\right) \right) + \cO \left(t_i(\bp)\right).
\end{align}
The term 
\begin{align}
	n \hat{\omega}_n s_1 \hat{\epsilon}^2 \underset{i = 1,...,c}{\min}\E{P_i}^2 = n \omega_n s_1 \hat{\epsilon}^2 \hat{p}_2^* = 0.0105
\end{align}
suggests that even with a large relative error, the contribution to the largest eigenvalue from this error remains negligible. Evaluating for $s_2$ shows the error in the second largest eigenvalue, which is similarly small.

In addition to verifying the recovery of the parameters, we also compare the sample Fr\'echet mean and sample total Fr\'echet variance of the recovered probability measure. We define a new sample of graphs $\lbrace \tilde G^{(k)} \rbrace_{k=1}^{50}$ sampled iid from $\mu_{\hat{\omega}_n, \hat{J}, \hat{\epsilon}, \bs}$ and compute the new sample mean eigenvalue and sample variance of the data set as
\begin{align}
	\bar \blamb_{new} = \begin{bmatrix} 133.3465\\ 91.4807 \end{bmatrix} \quad \hat{\Sigma}_{11, new} = 25.1882 \quad \hat{\Sigma}_{22, new} = 6.2949
\end{align}
with absolute relative errors of
\begin{align}
	\left \vert \frac{\bar \lambda_{1,new} - \bar{\lambda}_1}{\bar{\lambda}_1} \right \vert = 0.4432\%, \quad \left \vert \frac{\bar \lambda_{2,new} - \bar{\lambda_2}}{\bar{\lambda_2}} \right \vert = 0.4173\%, \quad \left \vert \frac{\hat{\Sigma}_{11, new} - \hat{\Sigma}_{11}}{\hat{\Sigma}_{11}} \right \vert = 1.3207\%, \quad  \left \vert \frac{\hat{\Sigma}_{22, new} - \hat{\Sigma}_{22}}{\hat{\Sigma}_{22}} \right \vert = 21.0933\%,
\end{align}
indicating that the sample Fr\'echet mean graph and sample total Fr\'echet variance are approximated well.
%______________________________________________________________________________________________
\subsection{Mixture model estimation via parametric $J$}
\label{subsec:MixtureModel}
%______________________________________________________________________________________________
We now test the flexibility of the model by attempting to fit a random parameter stochastic block model to a mixture model. Let $\lbrace \Gk \rbrace_{k=1}^N$ be a set of graphs sampled from a mixture of four stochastic block models where $\bp$ is allowed to vary between one of four different values. The following parameters are uniform across the mixture of models
\begin{align}
	n = 1000, \quad N = 200, \quad c = 2, \quad q = 0.05, \quad \omega_n = 10n^{-1/2}
\end{align}
The values of $\bp$ considered are
\begin{align}
	\bp \in \left \lbrace \begin{bmatrix} 0.9\\ 0.5 \end{bmatrix}, \begin{bmatrix} 0.9\\ 0.3 \end{bmatrix}, \begin{bmatrix} 0.6\\ 0.5 \end{bmatrix}, \begin{bmatrix} 0.6\\ 0.3 \end{bmatrix}\right \rbrace.
\end{align}
We generate a data set of graphs, $\lbrace \Gk \rbrace_{k=1}^N$, by first assigning equal weights to each possible value of $\bp$ and randomly selecting a value of $\bp^{(k)}$ from the possible set. A graph $\Gk$ is then sampled from $\mu_{\omega_n, \bp^{(k)}, q, \bs}$. As before, we compute $\bar \blamb$ and the diagonal of $\hat{\Sigma}$;
\begin{align}
	\bar{\blamb} = \begin{bmatrix} 120.2088 \\ 61.78\end{bmatrix}, \quad \begin{bmatrix} \hat{\Sigma}_{11} \\ \hat{\Sigma}_{22} \end{bmatrix} = \begin{bmatrix} 524.6778\\ 228.0952\end{bmatrix}
\end{align}
which we use to check the regime of variance of the data,
\begin{align}
	\begin{bmatrix}
		\hat{\Sigma}_{11} - \frac{2 \bar{\lambda}_1}{n s_1} = \\ 
		\hat{\Sigma}_{22} - \frac{2 \bar{\lambda}_2}{n s_2}
	\end{bmatrix}
	= 
	\begin{bmatrix}
		524.1970 \\ 
		227.8481
	\end{bmatrix}.
\end{align}
The values clearly indicate the data resides in the regime of high variance and thus we may safely ignore the inherent variability of the stochastic block model again. 

There are multiple perspectives to consider when constructing a generative model given this data. A researcher could first cluster the data and attempt to reconstruct the mixture, or she may interpret the data as a bi-modal distribution of graphs. We select the latter perspective for experimental purposes, and construct a generative model taking $J$ to be a product measure of shifted Beta probability measures. Taking this perspective allows for no consideration of the covariance for $J$. We first ensure that this assumption is consistent with the data by examining the sample correlation,
\begin{align}
	\frac{\hat{\Sigma}_{1,2}}{\sqrt{\hat{\Sigma}_{1,1} \hat{\Sigma}_{2,2}}} = -0.0188.
\end{align}
With the correlation being nearly $0$, we are assured in our assumption of a negligible covariance contribution and taking $J$ to be a product between shifted beta probability measures is therefore reasonable.

Each shifted Beta distribution, denoted $\text{Beta}_i$, depends on the four parameters, 
\begin{align*}
	&a_i \text{: minimum value of range,}\\
	&b_i \text{: maximum value of range,}\\
	&\alpha_i \text{: shape parameter,}\\
	&\beta_i \text{: shape parameter,}
\end{align*}
with mean and variance
\begin{align}
	&\E{P_i}[Beta_i] = a_i + (b_i - a_i)\frac{\alpha_i}{\alpha_i + \beta_i},\\
	&Var_{Beta_i}(P_i) = (b_i - a_i)^2 \frac{\alpha_i \beta_i}{(\alpha_i + \beta_i)^2(\alpha_i + \beta_i + 1)}.
\end{align}
Crude estimates for $a_i$ and $b_i$ are given by
\begin{align}
	a_i = \underset{k = 1,...,N}{\min} \frac{\lambda_i(\bA^{(k)})}{n \bar{\rho}_n s_i},\\
	b_i = \underset{k = 1,...,N}{\max}\frac{\lambda_i(\bA^{(k)})}{n \bar{\rho}_n s_i}.
\end{align}
where we have taken $\omega_n = \bar{\rho}_n$ as described in Alg. \ref{alg:DetJ}. Solving the following set of equalities for $\alpha_i$ and $\beta_i$,
\begin{align}
	\E{P_i}[Beta_i]&= a_i + (b_i - a_i)\frac{\alpha_i}{\alpha_i + \beta_i}\\
	Var_{Beta_i}(P_i) &= (b_i - a_i)^2 \frac{\alpha_i \beta_i}{(\alpha_i + \beta_i)^2(\alpha_i + \beta_i + 1)}
\end{align}
where $\E{P_i}[Beta_i] = \frac{\bar{\lambda}_i}{n \bar{\rho}_n s_i}$ and $Var_{Beta_i}(P_i) = \frac{\hat{\Sigma}_{ii}}{n^2 \bar{\rho}_n^2 s_i^2}$ result  in the following set of parameters for the two Beta distributions;
\begin{align}
	\bm a = \begin{bmatrix} 1.94656\\ 0.942685\end{bmatrix}, \quad \bm b = \begin{bmatrix} 2.97084\\ 1.63105\end{bmatrix}, \quad \bm \alpha = \begin{bmatrix} 0.0996378\\ 0.111218\end{bmatrix}, \quad \bm \beta = \begin{bmatrix} 0.102785\\ 0.13049\end{bmatrix}.
\end{align}

To compare the goodness of fit between the estimated distribution $\hat{\mu}$ resulting from Alg. \ref{alg:DetJ} we take a new sample of graphs, $\lbrace \Gk_{new} \rbrace_{k=1}^{200}$ from $\mu_{\hat{\omega}_n,\hat{J}, \hat{\epsilon}, \bs}$ and compare the sample Fr\'echet mean and sample total Fr\'echet variance. Lemma \ref{lem:STFV_Cov} shows that an estimate for the eigenvalues of the sample Fr\'echet mean for large graphs is given by the arithmetic mean eigenvalue, and an estimate for the sample total Fr\'echet variance is given by the diagonal of the sample covariance matrix. The sample mean eigenvalue and sample covariance matrix of the new sample are
\begin{align}
\bar{\blamb}_{new} = \begin{bmatrix} 121.0192 \\ 60.8119\end{bmatrix}, \quad \begin{bmatrix} \hat{\Sigma}_{11,new} \\ \hat{\Sigma}_{22,new}\end{bmatrix} = \begin{bmatrix} 504.1689\\ 209.7644\end{bmatrix}.
\end{align}
The relative error for each eigenvalue and the diagonal entries of the covariance matrix
\begin{align}
	&\left \vert \frac{\bar{\lambda}_{1,new} - \bar{\lambda}_1}{\bar{\lambda}_1} \right \vert = 0.6742\%, \quad \left \vert \frac{\bar{\lambda}_{2,new} - \bar{\lambda}_2}{\bar{\lambda}_2} \right \vert = 1.567\%, \quad \left \vert \frac{\hat{\Sigma}_{11,new} - \hat{\Sigma}_{11}}{\hat{\Sigma}_{11}} \right \vert = 3.9089 \%, \quad \left \vert \frac{\hat{\Sigma}_{22,new} - \hat{\Sigma}_{22}}{\hat{\Sigma}_{22}} \right \vert = 8.0365 \%.
\end{align}

While we have confirmed that the statistics of the recovered distribution match both the first and second moments of the original data set, we may further check the goodness of fit for the choice of $J$ by considering an estimate of the probability density of the two largest eigenvalues from each data set, as presented in Fig. \ref{fig:PDFBeta}.
\begin{figure}[H]
	  \centering
	  \includegraphics[width= \textwidth]{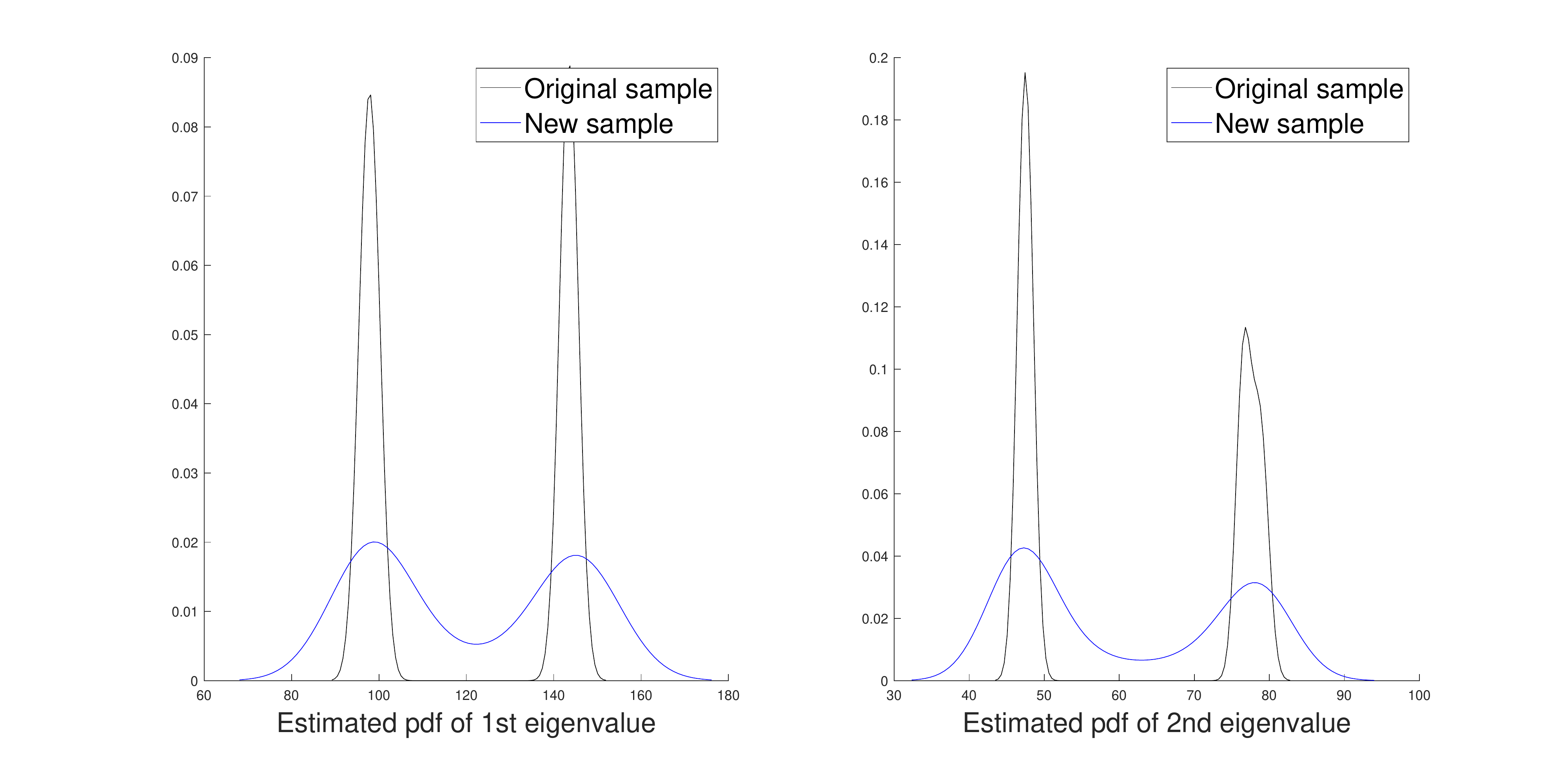}
	  \caption{A comparison between the estimated probability densities of the two largest eigenvalues from the original data set (black) compared to the recovered data set (blue). Each curve was generated by first embedding the observed sets of eigenvalues in $\R^c$ and performing a kernel density estimate.}
	   \label{fig:PDFBeta}
\end{figure}
As is evident in Fig. \ref{fig:PDFBeta}, using a random parameter stochastic block model with $J$ given by the product of Beta distributions to fit a mixture model may not always result in an ideal fit of the distributions, even when aligning the first and second moments.

%______________________________________________________________________________________________
\subsection{The effect of sample size for non-parametric models when $J^{(k)}$ is given by a Dirac delta}
\label{subsec:ExplorationCriticalN}
%______________________________________________________________________________________________
Throughout this section we allow $N$, the sample size of the data set, to increase while $n$ is fixed and examine the effect on the condition given by equation (\ref{eqn:CondKDE2}). In this section, graphs will be sampled according to various $\text{\ER}$ models. We first review the properties of the limiting distribution of the largest eigenvalue of graphs sampled according to this model, as shown by Theorem \ref{thm:CCH}. For an $\text{\ER}$ model, we denote the associated probability measure given parameters $p$ and $\omega_n$ as
\begin{align}
	\mu_{\omega_n, p}
\end{align}
because there is no inter-community connection probability, $q$, and $\bs$ is given trivially because all nodes belong to one community.
Recall from Remark \ref{rem:dirac} that $\mu_{\omega_n, p} = \mu_{\omega_n, J}$ when $J$ is given by a Dirac delta distribution with mean $p$.

Let $G \sim \mu_{\omega_n, p}$ with adjacency matrix $\bA_{\mu_{\omega_n, p}}$. Theorem \ref{thm:CCH} shows that
\begin{align}
	\frac{1}{\sqrt{\omega_n}} \left( \lambda_1(\bA_{\mu_{\omega_n, p}}) - \E{\lambda_1(\bA_{\mu_{\omega_n, p}})} \right) \overset{d}{\to} N(0,2p)
\end{align}
with 
\begin{align}
	\E{\lambda_1(\bA_{\mu_{\omega_n, p}})} = \lambda_1(\bm B) + \cO(\sqrt{\omega_n} + \frac{1}{n \omega_n}) = n \omega_n p + 1 + \cO(\sqrt{\omega_n} + \frac{1}{n \omega_n}).
\end{align}
For a given $\text{\ER}$ model, we define the probability density function
\begin{align}
	f(z) = \Phi(n \omega_n p + 1, \sqrt{2p}) = \frac{1}{\sqrt{4 p \pi}} e^{-\frac{1}{2} \left( \frac{z - (n \omega_n p +1)}{\sqrt{2 p}}\right)^2}
\end{align}
as an estimate of the distribution of the largest eigenvalue when $n$ is large.
%______________________________________________________________________________________________
\subsubsection{Simulation setup and algorithm}
%______________________________________________________________________________________________
We now consider the recovery of a simple mixture of two $\text{\ER}$ models with the following parameters:
\begin{align}
	n = 1000 \quad \omega_n = 2 n^{-1/2} \quad p_1 = 0.75 \quad p_2 = 0.85.
\end{align}
Define
\begin{align}
	\mu = \frac 1 2 \left( \mu_{\omega_n, p_1} + \mu_{\omega_n p_2} \right).
\end{align}
The following algorithm describes the simulations performed in this section. The algorithm is similar to Alg. \ref{alg:DetJKDE}. 
\begin{algorithm}[H]
  \caption{Estimation of  $\mu$ given sample data}
  \label{alg:DetJKDE-ER}
  \begin{algorithmic}[1]
    \Require Set of graphs, $M = \lbrace \Gk \rbrace_{k=1}^N$, with $\Gk \sim \mu$.
    \State Compute $\lambda_1^{(k)} = \lambda_1(\bA^{(k)})$.
    \State Compute $\rho_n^{(k)}$ as the density of graph $\Gk$.
    \State Construct an oracle knowledge estimate of the mixture of two Gaussian distributions,
    \begin{align}
    		f^{true}(z) = \frac 1 2 \left(\Phi(n \omega_n p_1 + 1, \sqrt{2p_1})  + \Phi(n \omega_n p_2 + 1, \sqrt{2p_2}) \right).
    \end{align}    
    \State \textbf{For each $k$},
    \State \hspace{0.5cm} Define $p^{(k)} = \frac{\lambda_1^{(k)} - 1}{n \rho_n^{(k)}}$.
    \State \hspace{0.5cm} Define the probability measures $\mu_{\rho_n^{(k)}, p^{(k)}}$.
    \State \hspace{0.5cm} Define $f^{(k)}(z) = \Phi(n \rho_n^{(k)} p^{(k)} + 1, \sqrt{2p^{(k)}})$.
    \State \textbf{end for}
    \State Define the kernel density estimator in $\cG$ as $\hat{\mu} = \frac 1 N \sum_{k=1}^N \mu_{\rho_n^{(k)}, p^{(k)}}$.
    \State Define
    \begin{align}
    		\hat{f}(z) = \frac 1 N \sum_{k=1}^N \Phi(n \rho_n^{(k)} p^{(k)} + 1, \sqrt{2p^{(k)}}).
    \end{align}
    \State Construct a kernel density estimator given  $\lbrace \lambda^{(k)}_1 \rbrace_{k=1}^N \subset R^1$ where $f^{(k)}_{Silverman}(z) = \Phi(n \rho_n^{(k)} p^{(k)} + 1, \sqrt{h_N})$ such that $h_N$ is constructed with oracle knowledge of the variance of $\mu$ according to Silverman's rule of thumb,
	\begin{align}
		&\sigma^2_{oracle} = \frac 1 2 (2 p_1 + 2 p_2) + \frac 1 2 \left(  \left( n \omega_n p_1 \right)^2 + \left( n \omega_n p_2 \right)^2\right) - \left(\frac 1 2 \left( n \omega_n p_1 + n \omega_n p_2\right) \right)^2\\
		&h_N = \left(\left( \frac{4}{3}\right)^{\frac{1}{5}} N^{\frac{-1}{5}} \sigma _{oracle} \right)^2\label{eqn:Silver}\\
    		&f^{Silverman}(z) = \frac 1 N \sum_{k=1}^N \Phi(n \rho_n^{(k)} p^{(k)} + 1, \sqrt{h_N}).
    	\end{align}
	\State Plot $f^{true}(z), \hat{f}(z),$ and $f^{Silverman}(z)$.
  \end{algorithmic}
\end{algorithm}
Alg. \ref{alg:DetJKDE-ER} is used to explore the condition given by equation (\ref{eqn:CondKDE2}). Rewriting this condition yields,
\begin{align}
	\frac{2 \lambda^{(k)}_i}{n s^{(k)}_i} < \bm H_{ii}
\end{align}
which, for the simulation in this section, reduces to
\begin{align}
	2 p^{(k)} < h_N.
\end{align}
For this condition to be met for every $k$, it requires that
\begin{align}
	\underset{k = 1,..., N}{\max}2 p^{(k)} < h_N. \label{eqn:Crit}
\end{align}
Because $h_N$ monotonically decays in $N$ and $\underset{k = 1,..., N}{\max}2 p^{(k)}$ monotonically increases in $N$, there exists a finite critical value of $N$ where the condition is not met for all $N > N_{crit}$. For the particular choice of $\mu$, the critical value of $N$ is estimated to be $N_{crit} = 125$. 

The relation to equation (\ref{eqn:CondKDE2}) is that for all values of $N$ where (\ref{eqn:Crit}) is satisfied, equation (\ref{eqn:DetJ2}) in Alg. \ref{alg:DetJKDE} which defines the second moment of $J^{(k)}$, can be solved for every $k = 1,...,N$. The feasibility of determining the second moment of $J^{(k)}$ is visualized below where each value of $N$ that satisfies equation (\ref{eqn:Crit}) is shaded in green.

\begin{figure}[H]
	 \includegraphics[width= \textwidth]{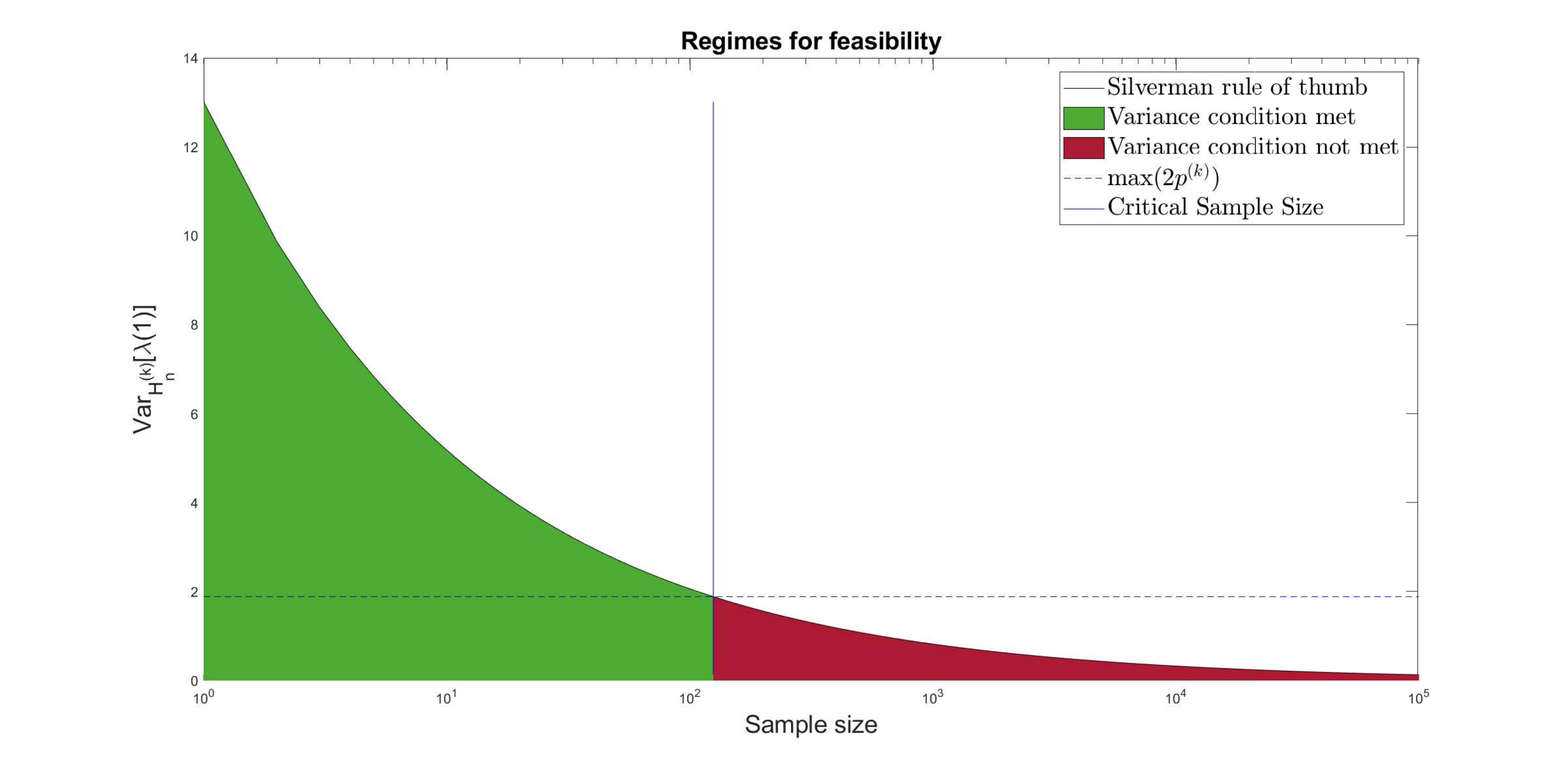} 
	  \caption{(Black) $h_N$: Silverman's rule of thumb as a function of sample size, $N$, when considering the oracle knowledge of the true variance, $\sigma^2$. (Green) This area depicts all sample sizes $N$ from $\mu$ where equation (\ref{eqn:DetJ2}) can be solved fro all $k$. (Red) This area depicts all sample sizes $N$ from $\mu$ where at least one $k$ exists such that equation (\ref{eqn:DetJ2}) cannot be resolved. (Blue dashed) An estimate of $\max_{k = 1,...,N} p^{(k)}$. (Blue solid) The critical sample size where there exists at least one $k^*$ such that $2 p^{(k)} > h_N$.}
	  \label{fig:feasibility}
\end{figure}

Fig. \ref{fig:feasibility} indicates that for large sample sizes and a finite $n$, the family of random parameter stochastic block models is a poor choice for kernel density estimation as defined in equation (\ref{eqn:KernDensEst}). In fact, Theorem \ref{thm:CCH} shows that any change in $\theta_k$ and $r_k(x)$ for any inhomogeneous $\text{\ER}$ random graph model affects both $\E{\lambda_i(\bA)}$ and $\Cov(Z_i,Z_j)$, see equations (\ref{eqn:CovCCH}) and (\ref{eqn:EstLargeEigs-CCH}). This observation supports a broader claim that this critical value of $N$ will exist for general inhomogeneous $\text{\ER}$ random graph models, not just stochastic block models, because the variance of the eigenvalues cannot decay to $0$ without changing the mean.

For the simulation presented in this section, Fig. \ref{fig:feasibility} suggests examining three different sample sizes to determine the quality of the proposed estimator found by Alg. \ref{alg:DetJKDE-ER}.
%______________________________________________________________________________________________
\subsubsection{Three different sample sizes}
\label{subsec:ExplorationCriticalN}
%______________________________________________________________________________________________
Three different samples sizes of graphs, $N_1 < N_2 = N_{crit}< N_3$, are drawn from $\mu$:
\begin{align}
	&M_1 = \lbrace \Gk \rbrace_{k=1}^{10}\\
	&M_2 = \lbrace \Gk \rbrace_{k=1}^{125}\\
	&M_3 = \lbrace \Gk \rbrace_{k=1}^{325}.
\end{align}
For each data set, we evaluate Alg. \ref{alg:DetJKDE-ER} and plot the results.
\begin{figure}[H]
	 \includegraphics[width= \textwidth]{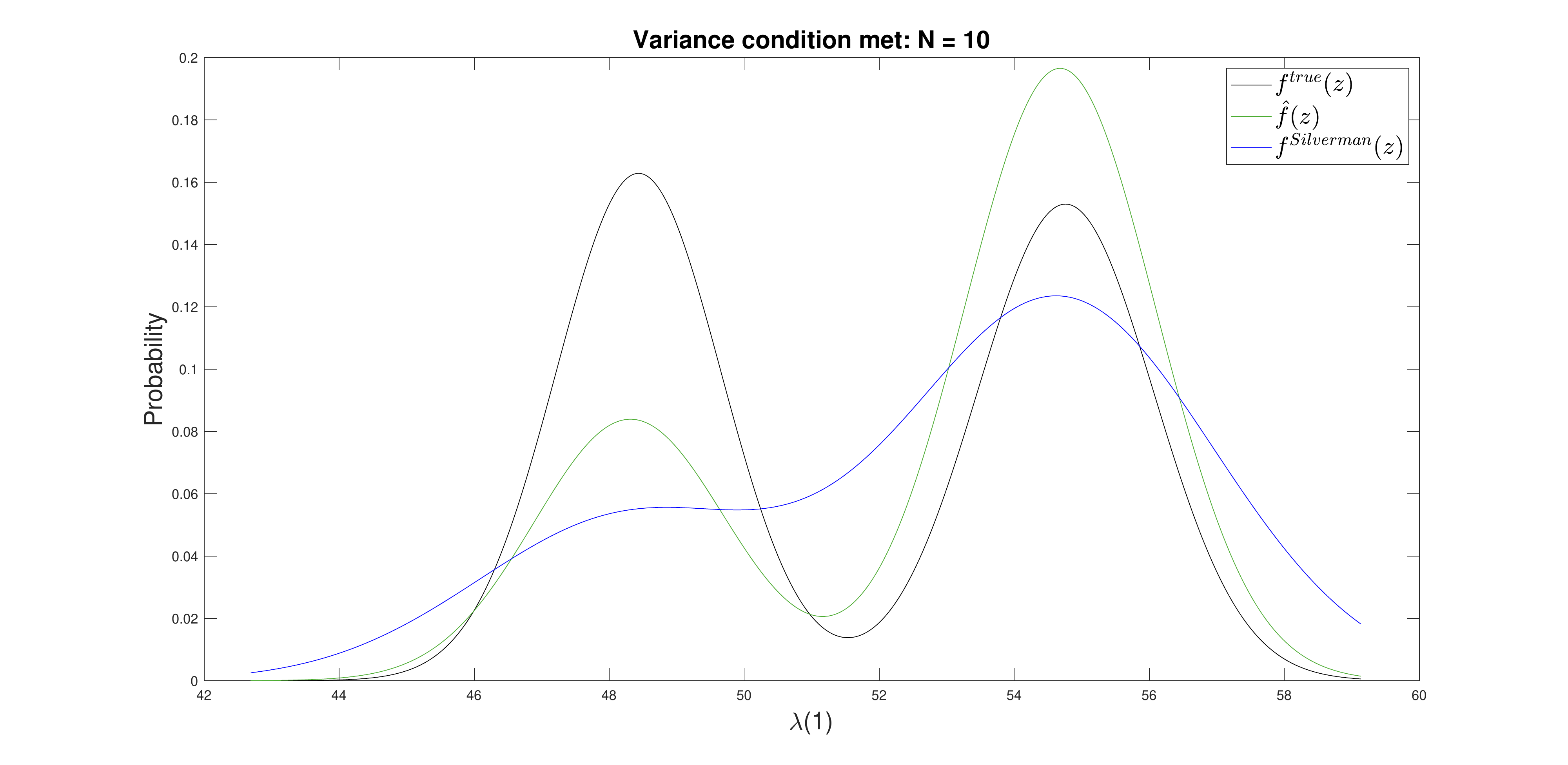} 
	  \caption{(Black) Oracle knowledge of the limiting behavior of the true distribution for $\mu$ given by $f(z)$. (Green) The limiting behavior of the eigenvalues when graphs are sampled according to the mixture model $\hat{\mu}$ given by $\hat{f}_{\delta}^{N}(z)$ . (Blue) The limiting behavior of the eigenvalues when performing kernel density estimation in $\R$ when the bandwidth is chosen according to Silverman's rule of thumb for the set $\lbrace \lambda_1^{(k)} \rbrace_{k=1}^N$, given by $\hat{f}^{N}_{Silverman}(z)$.}
	  \label{fig:Subcritical}
\end{figure}
When the condition on variance is met, defined by equation (\ref{eqn:Crit}), the regime is termed feasible due to the feasibility of solving equation (\ref{eqn:DetJ2}) in Alg. \ref{alg:DetJKDE}. Note however that the blue and green curves need not align even when resolving (\ref{eqn:DetJ2}); all that is guaranteed is a condition on the variance.

We next examine the behavior of the kernel density estimator at the critical value of $N^*$. By construction, the quantity
\begin{align}
	\max \left(2 p^{(k)} \right) - \min \left(2 p^{(k)}\right)
\end{align}
is small due to the small variation between $p_1$ and $p_2$ when defining the mixture model $\mu$. For this reason, the difference between $\hat{f}(z)$ and $f^{Silverman}(z)$ is expected to be minimal because $2 p^{(k)}$ is close to $h_{N_{crit}}$ for each value of $k$. Fig. \ref{fig:Critical} shows this approximate alignment.
\begin{figure}[H]
	 \includegraphics[width= \textwidth]{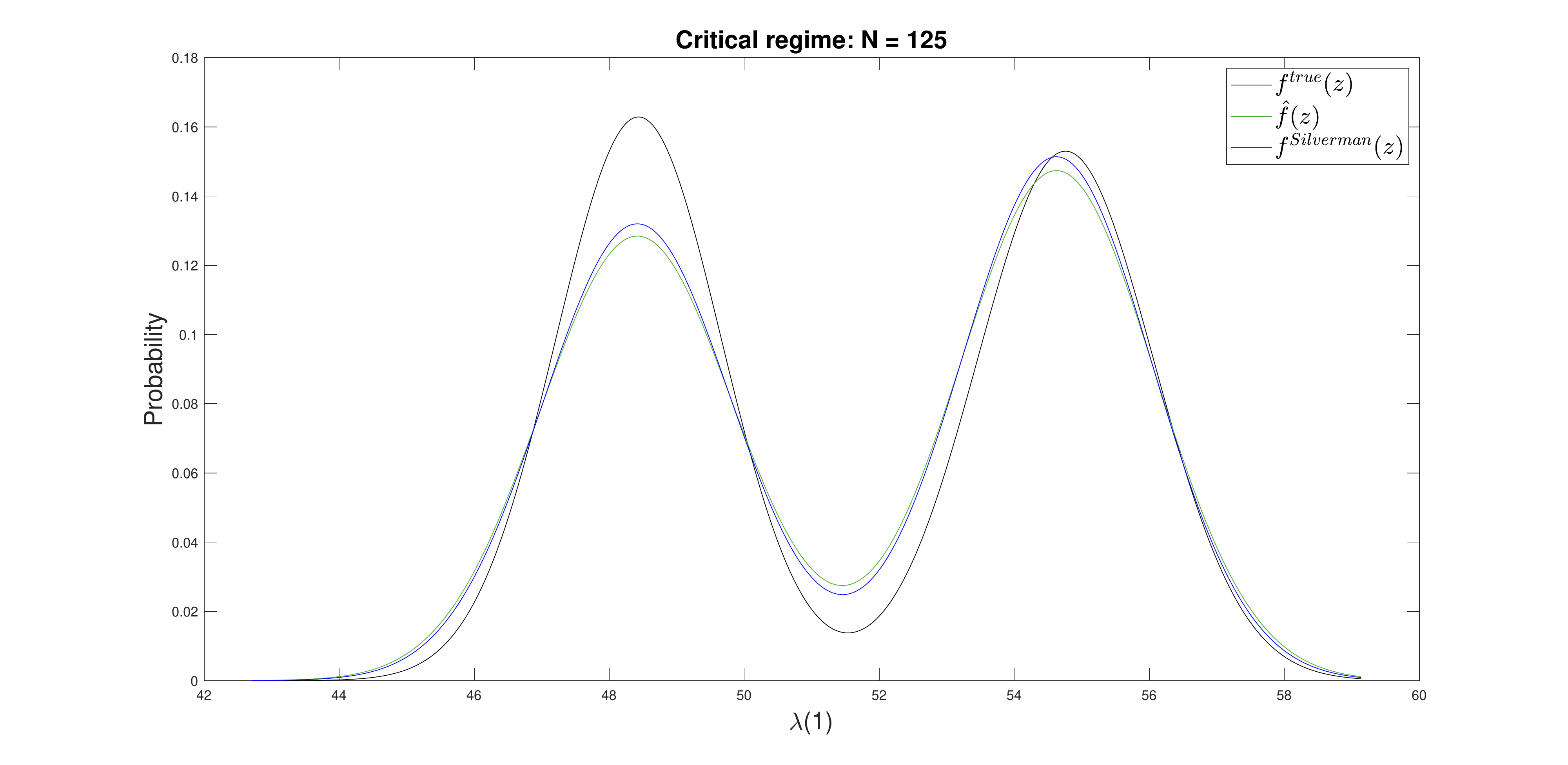} 
	  \caption{(Black) Oracle knowledge of the limiting behavior of the true distribution for $\mu$ given by $f(z)$. (Green) The limiting behavior of the eigenvalues when graphs are sampled according to the mixture model $\hat{\mu}$ given by $\hat{f}_{\delta}^{N}(z)$.  (Blue) The limiting behavior of the eigenvalues when performing kernel density estimation in $\R$ when the bandwidth is chosen according to Silverman's rule of thumb for the set $\lbrace \lambda_1^{(k)} \rbrace_{k=1}^N$, given by $\hat{f}^{N}_{Silverman}(z)$.}
	  \label{fig:Critical}
\end{figure}
Past this critical value of $N$, $\hat{f}(z)$ and $f^{Silverman}(z)$ are expected to be dissimilar again, as shown in Fig. \ref{fig:Supercritical}.
\begin{figure}[H]
	 \includegraphics[width= \textwidth]{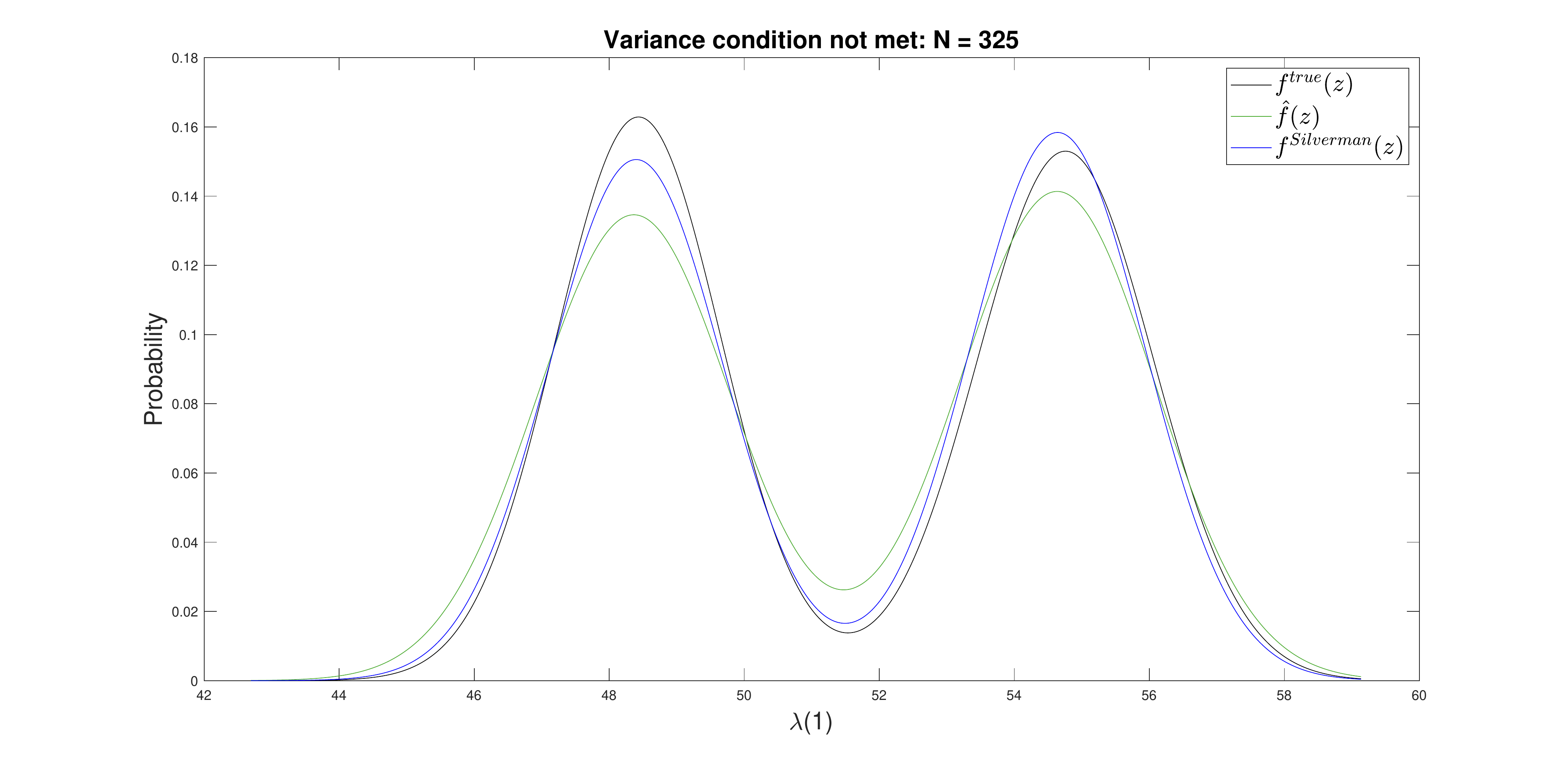} 
	  \caption{(Black) Oracle knowledge of the limiting behavior of the true distribution for $\mu$ given by $f(z)$. (Green) The limiting behavior of the eigenvalues when graphs are sampled according to the mixture model $\hat{\mu}$ given by $\hat{f}_{\delta}^{N}(z)$.  (Blue) The limiting behavior of the eigenvalues when performing kernel density estimation in $\R$ when the bandwidth is chosen according to Silverman's rule of thumb for the set $\lbrace \lambda_1^{(k)} \rbrace_{k=1}^N$, given by $\hat{f}^{N}_{Silverman}(z)$.}
	  \label{fig:Supercritical}
\end{figure}
%______________________________________________________________________________________________
\subsection{Primary school time varying networks}
\label{subsec:Primary}
%______________________________________________________________________________________________
The final data set of graphs considered is a time series of networks, collected via RFID tags in a French primary school \cite{Socio1, Socio2}. These networks exhibit dynamic structural behaviors temporally by the manner of merging communities throughout the day in addition to the random fluctuations of the individual interactions. The school is composed of five grades with two classes per grade for a total of 10 different classes and a fixed vertex set of size $n = 242$.

At every time $t$, a collection of edges is given that corresponds to the face-to-face contacts of the graph, and a new graph is recorded every $20$ seconds. We collect all connections within a time window of $45$ minutes (2700 seconds) to define a single graph and shift the window by one time step to collect the next graph. The graph $\Gk$ describes all connections made from $20 (k-1)$ seconds to $2700 + 20(k-1)$ seconds.

We expect a high degree of correlation from $\Gk$ to $G^{(k+1)}$ and yet, as is displayed in Fig. \ref{fig:VecsDifTimes}, there is notable dynamic behavior in the graphs as observed by the change in the largest eigenvectors. Many perspectives analyze the time series of graphs as a change point problem or perhaps anomaly detection. We ask a fundamentally different question: What is the distribution of the graphs during a time interval?

We consider only the initial 2 hours of the school day, 9:00 - 11:00 a.m. This time interval includes the behavior of students both in the classrooms and during the 10:30 - 11:00 a.m. recess and gives a data set of graphs with size $N = 225$, denoted as $M = \lbrace \Gk \rbrace_{k=1}^N$. 

Informed by the dynamics of the networks, such as the merging of communities during inter-classroom projects as well as the mixing of classes during recess, we suggest having a dynamic estimate for the geometry vector. Rather than taking $\bs$ to be constant, as was done in all prior experiments when recovering a random-parameter stochastic block model, we now estimate $\bs^{(k)}$ for each $k$. The work in \cite{GLM17, HST13, LS14,N06} showcases that the sum of the logarithms of the largest eigenvectors can be used to identify the geometry vector $\bs^{(k)}$ (see Fig. \ref{fig:VecsDifTimes} and Fig. \ref{fig:ChangePoints}). An estimate for the largest $K$ eigenvectors from each $\bA^{(k)}$ was determined by counting the number of eigenvalues greater than $\vert \lambda_{242}(\bA^{(k)}) \vert$, which is interpreted as an estimate of the number of extremal eigenvalues.

\begin{figure}[H]
\begin{minipage}{.5\textwidth}
	  \centering
	  \includegraphics[width= \textwidth]{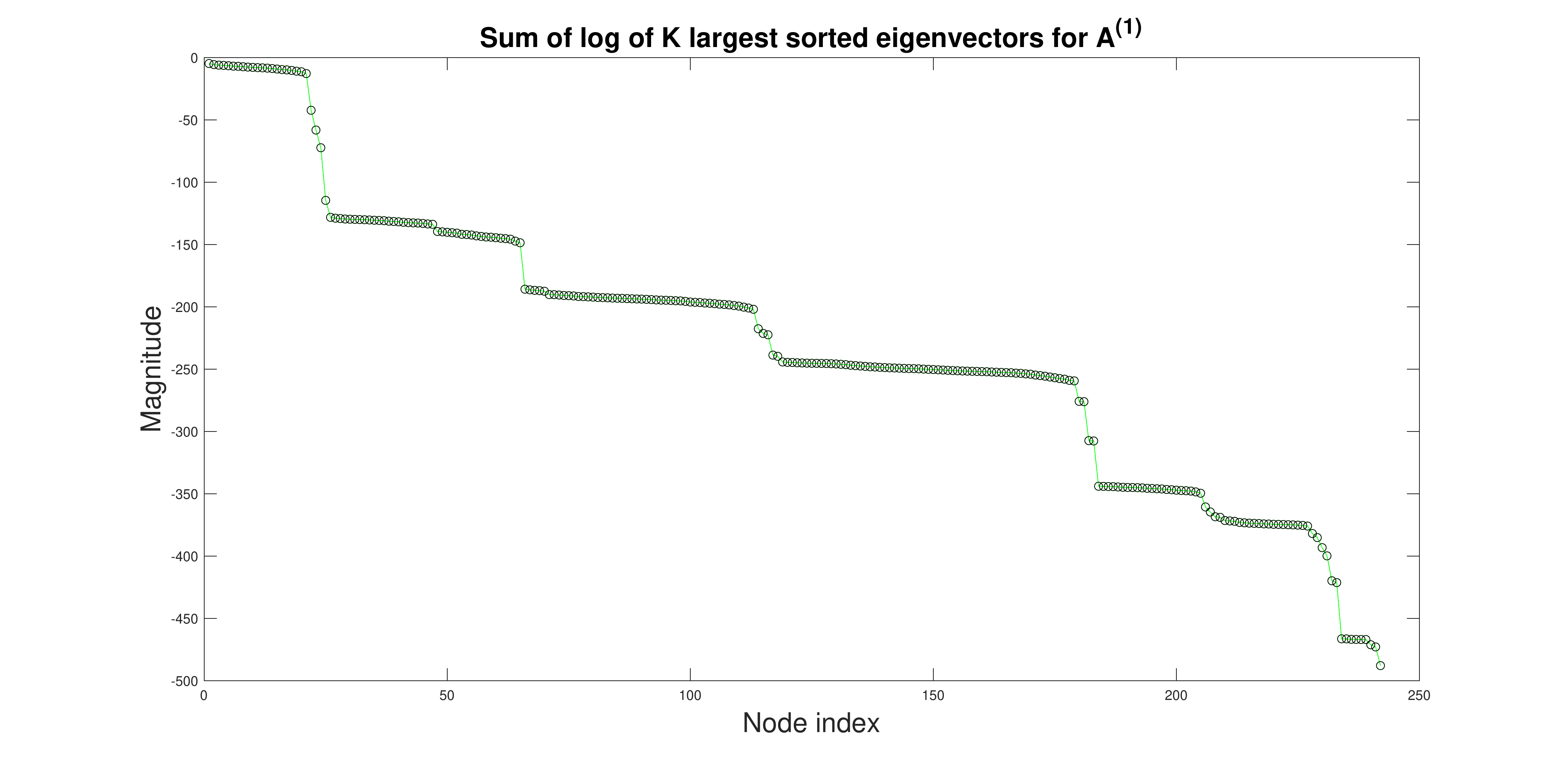}
	\end{minipage}%
	\begin{minipage}{.5\textwidth}
	\centerline{
	  \includegraphics[width = \textwidth]{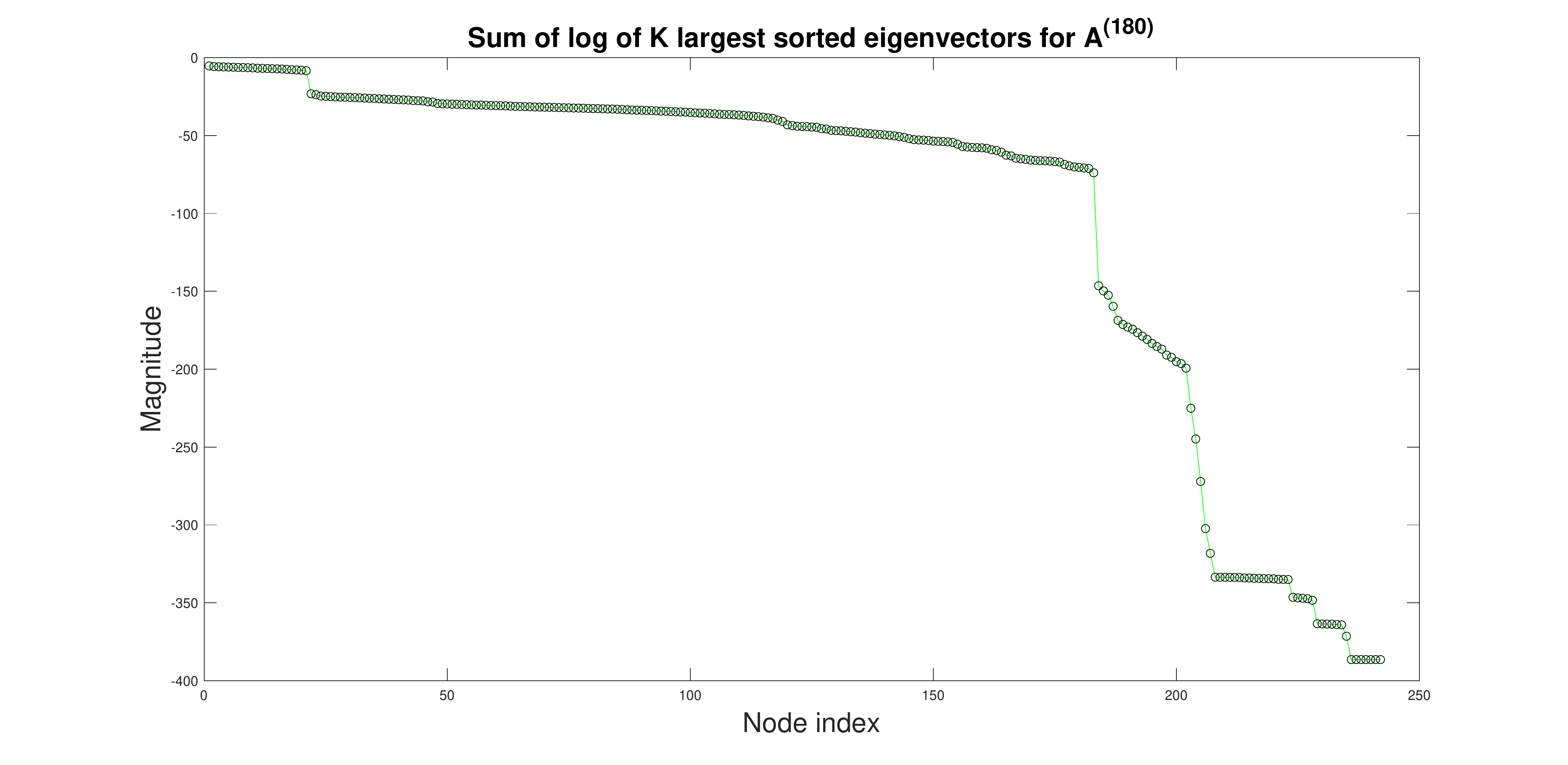}
	  }
	\end{minipage}
	\caption{(Left) A visualization of the sum of the logarithm of the absolute value of the $K$ largest sorted eigenvectors from $\bA^{(1)}$. (Right) A visualization of the sum of the logarithm of the absolute value of the $K$ largest sorted eigenvectors from $\bA^{(180)}$.}
	\label{fig:VecsDifTimes}
\end{figure}
\begin{figure}[H]
\begin{minipage}{.5\textwidth}
	  \centering
	  \includegraphics[width= \textwidth]{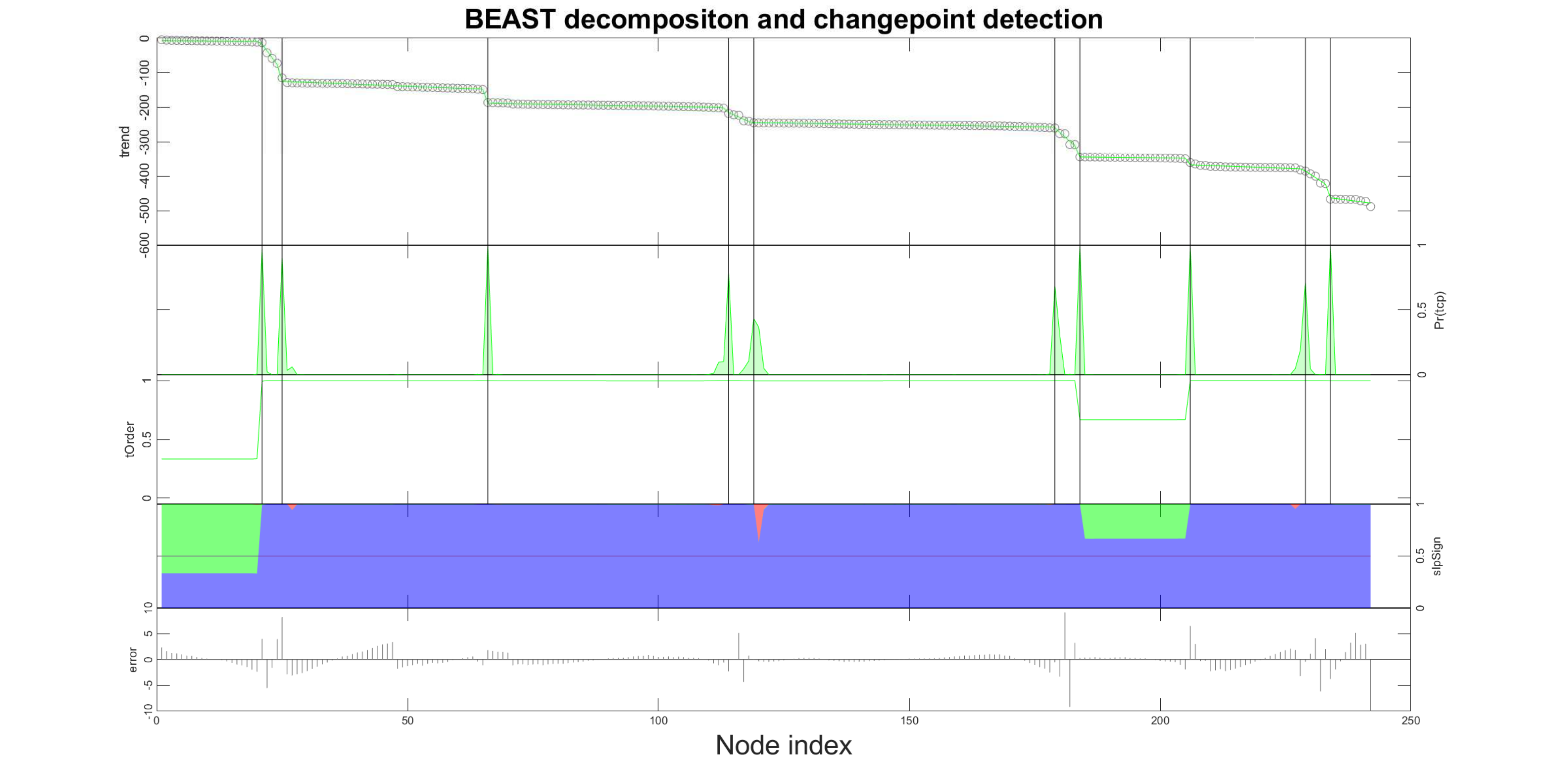}
	\end{minipage}%
	\begin{minipage}{.5\textwidth}
	\centerline{
	  \includegraphics[width = \textwidth]{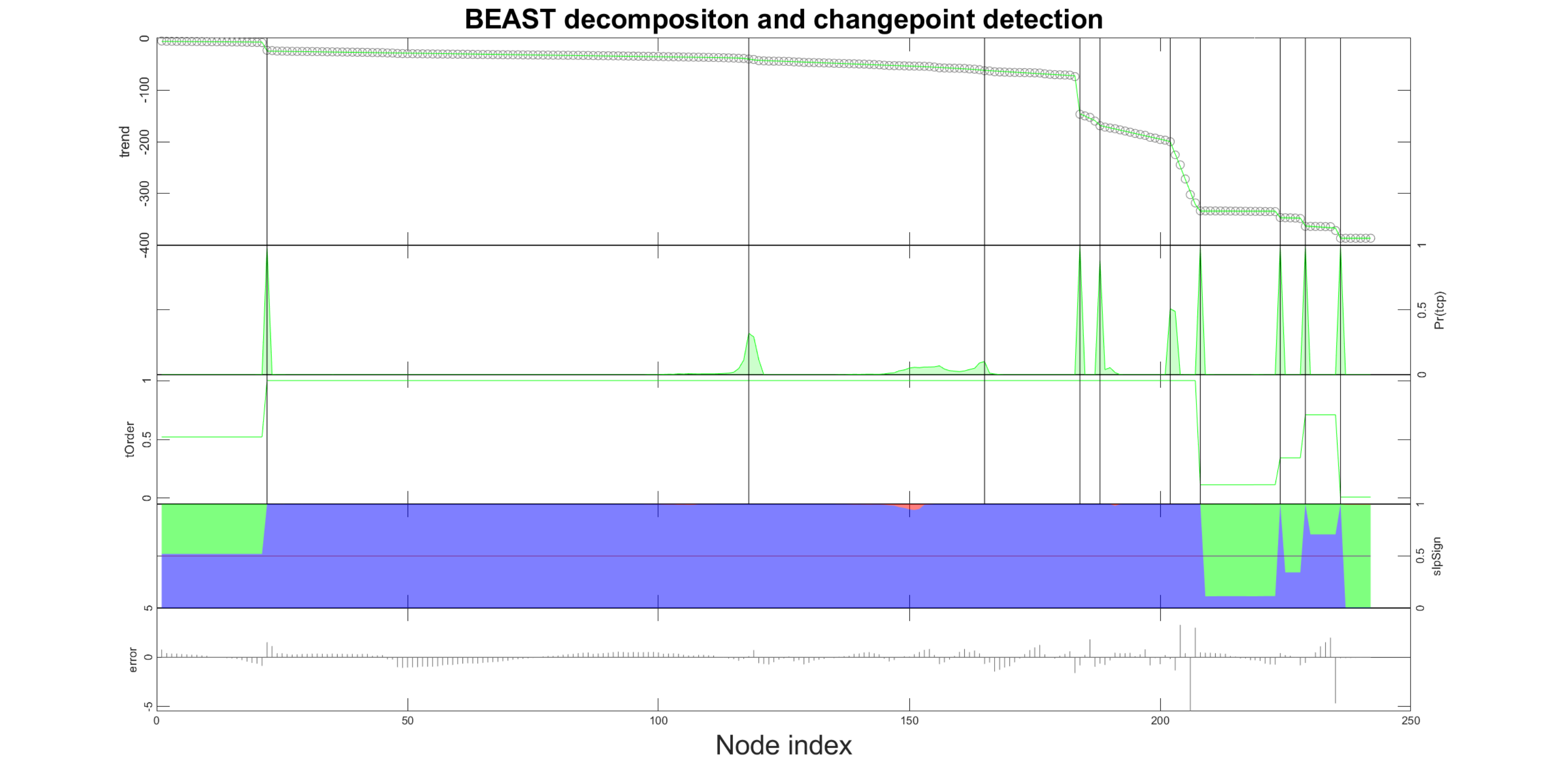}
	  }
	\end{minipage}
	\caption{An estimation of the change points in the plot in the sum of the largest $K$ eigenvectors for $\bA^{(1)}$ and $\bA^{(180)}$. The vertical lines indicate the locations of the changepoints, which are used to infer the communities}
	\label{fig:ChangePoints}
\end{figure}
	The geometry vector $\bs^{(k)}$ is determined by summing the absolute value of the largest $K$ eigenvectors and analyzing the change points.  Each change point is then interpreted as a new community within the graph $G^{(k)}$. A robust estimate for change point detection is depicted above which uses a Bayesian framework (see \cite{Beast}) though any state-of-the-art algorithm for change point detection should suffice. When two sequential change points are detected at $index_{i}$ and $index_{i+1}$ such that $\vert index_i - index_j \vert < \lceil \lambda_1^{(k)}\rceil$, we elect to take the average. This is due to the condition on the geometry, equation (\ref{eqn:GeomFeas}), which shows that the number of nodes in a community must be larger than the corresponding eigenvalue.
	
	Detecting the geometry for each graph $\Gk$ leads to a set of geometry vectors $\lbrace \bs^{(k)} \rbrace_{k=1}^N$. Before estimating a distribution of graphs, we now cluster the graphs $\Gk$ such that, within a cluster, the geometry vectors $\bs^{(k)}$ have the same number of non-zero entries. 
	
	For this particular data set, we find four distinct clusters, namely graphs with 4, 5, 6, and 7 communities respectively. We denote each subset of graphs by $M_c \subset M$, where $c$ denotes the number of communities. The size of each $M_c$ is given by the following.
\begin{align}
	&\vert M_4 \vert = 20\\
	&\vert M_5 \vert = 63\\
	&\vert M_6 \vert = 113\\
	&\vert M_7 \vert = 29\\
	&M = M_4 \cup M_5 \cup M_6 \cup M_7\\
	&\emptyset = M_i \cap M_j \quad \forall i \neq j
\end{align}

Here, we analyze only $M_5$ and $M_6$, the two clusters with a significant number of graphs, using Alg. \ref{alg:DetJKDE}, where we assume $J^{(k)}$ is a product of uniform probability measures, resulting in two probability measures, $\hat{\mu}_{M_5}$ and $\hat{\mu}_{M_6}$. We then sample $N = 500$ graphs according to $\hat{\mu}_{M_5}$ to determine a new set of graphs $\lbrace \Gk_{new} \rbrace_{k=1}^{57}.$ To compare the quality of the estimated probability measure $\hat{\mu}_{M_5}$, we compute the largest $5$ eigenvalues for each $\Gk \in M_5$ and each $\Gk_{new} \in \lbrace \Gk_{new} \rbrace_{k=1}^{500}$, perform kernel density estimation for the vectors of the eigenvalues in $\R^5$, and plot the results.
\begin{figure}[H]
	  \centering
	  \includegraphics[width= \textwidth]{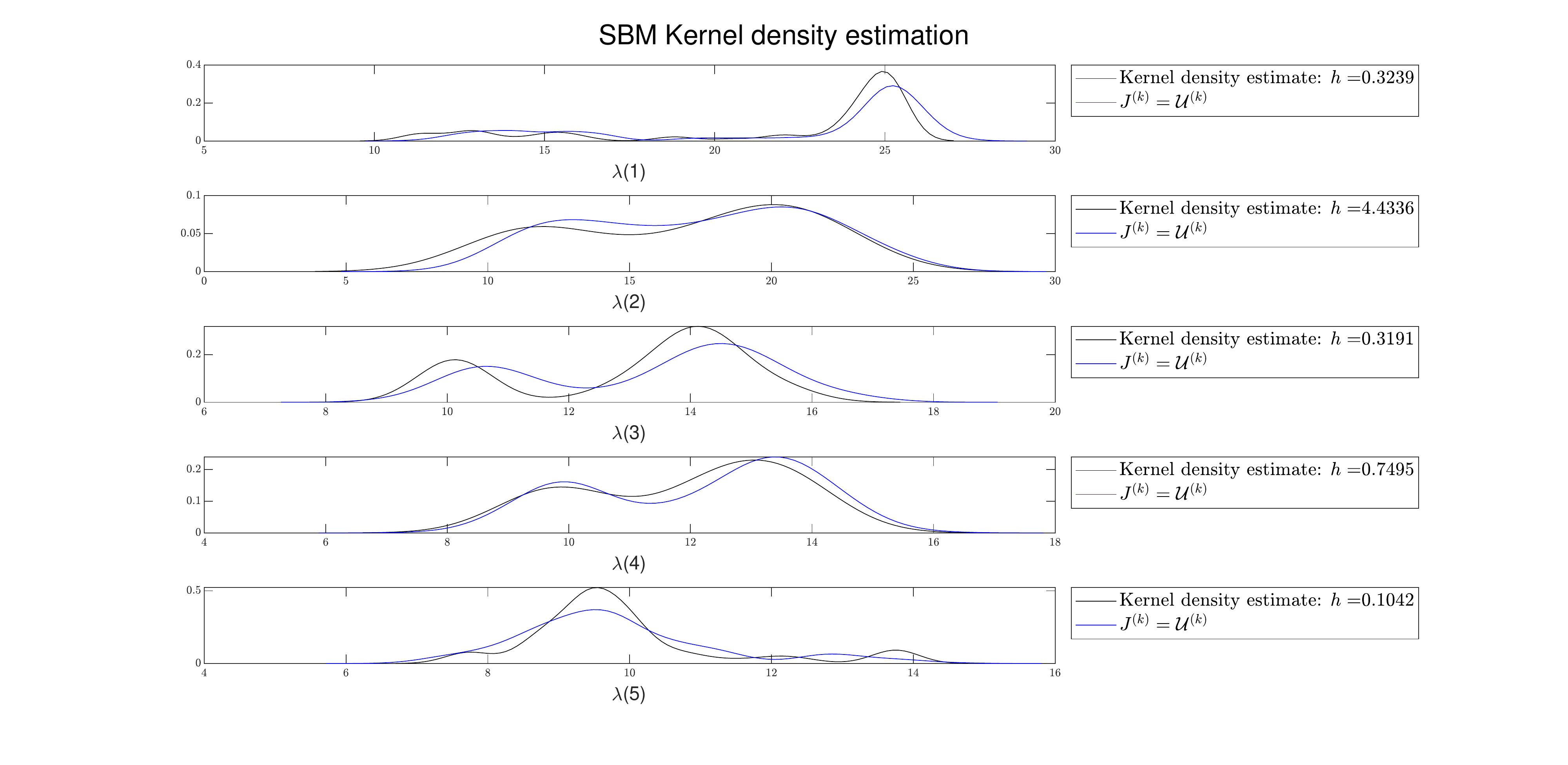}
	  \caption{(Black) A kernel density estimate of the sample eigenvalues of the adjacency matrices of graphs in $M_5$. (Blue) A kernel density estimate of the sample eigenvalues of the adjacency matrices of graphs in $\lbrace \Gk_{new} \rbrace_{k=1}^{|M_5|}$ where $\Gk_{new} \sim \hat{\mu}_5$.}
	   \label{fig:DensityEstimate5}
\end{figure}
For $M_6$, we perform the exact same procedure, except we consider the cluster of graphs with $6$ non-zero entries in the geometry vectors $\bs^{(k)}$.
\begin{figure}[H]
	  \centering
	  \includegraphics[width= \textwidth]{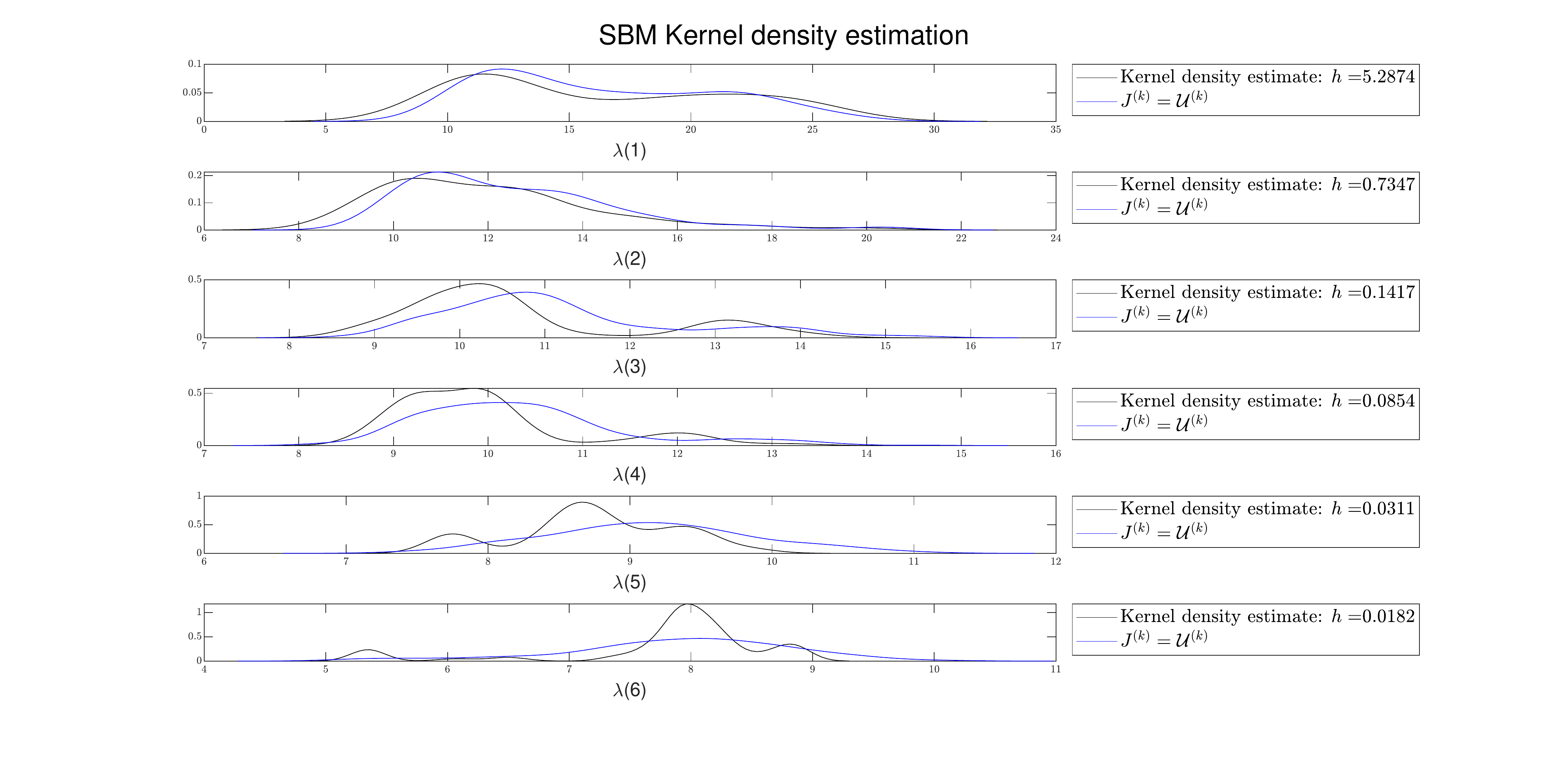}
	  \caption{(Black) A kernel density estimate of the sample eigenvalues of the adjacency matrices of graphs in $M_6$. (Blue) A kernel density estimate of the sample eigenvalues of the adjacency matrices of graphs in $\lbrace \Gk_{new} \rbrace_{k=1}^{|M_5|}$  $\Gk_{new} \sim \hat{\mu}_6$.}
	  \label{fig:DensityEstimate6}
\end{figure}

While the black curve need not resemble the true distribution of the eigenvalues well, this is the only baseline with which we may compare our results. Notably, the black curve defines a distribution in the space $\R^c$ but does not determine how to generate graphs with the corresponding distribution of eigenvalue. The primary advantage to our method is that the blue curve was generated by first sampling a graph and then computing the eigenvalues of the graph's adjacency matrix. The implication being that the probability measure 
\begin{align}
	\hat{\mu} = \frac 1 N \sum_{k=1}^N  \mu_{\omega_n^{(k)}, J^{(k)}, \epsilon^{(k)}, \bs^{(k)}}
\end{align}
found by Alg. \ref{alg:DetJKDE} distributes graphs such that the largest eigenvalues follow a distribution that is similar to the black curve.

The values of $h$ are presented to showcase that when $h$ is small, over-smoothing of the data is expected, which can be seen most clearly in Fig. \ref{fig:DensityEstimate6} for $\lambda(5)$ and $\lambda(6)$. In contrast, for larger values of $h$, such as in Fig. \ref{fig:DensityEstimate6} for $\lambda(1)$, oversmoothing is not expected because the condition given by equation (\ref{eqn:CondKDE2}) is met.
%______________________________________________________________________________________________
\section{Conclusions}
\label{sec:Conc}
%______________________________________________________________________________________________
	A preliminary step for the analysis of any graph-valued data set is the choice of metric, or measure of similarity, between graphs. When the distance is chosen with respect to spectral information, such as $d_{A_c}$, a wide class of generative models can be considered when fitting a data set, called inhomogeneous $\text{\ER}$ random graph models.  The spectral properties of graphs generated in this way have been well-studied, as evidenced by the work in \cite{BBK19, CCH20, FFHL19, T18}, among others. Two general results are observed for this class of models. First, there is low variability in the largest eigenvalues given a fixed function $f(x,y)$ that defines the inhomogeneous $\text{\ER}$ model. Second, the location and scale of the largest eigenvalues of the adjacency matrices of graphs generated in this way are dependent. It should be noted that here, the term inhomogeneous is used to characterize the edge probabilities of the graphs, and not the data sets of graphs generated in this manner. A sample set of graphs from such a model remains rather homogeneous.
	
	To mitigate the problem of low variance of the largest eigenvalues, when considering the class of stochastic block models, this manuscript introduced randomness in the parameter space via the distribution $J$ to better model heterogeneous data sets of graphs, $\sGk_{k=1}^N$. We have shown through Lemma \ref{lem:RelErrJMoments} that methods which estimate relative error in the parameters of stochastic block models can be translated to estimate relative error in the first and second moments of the distribution $J$. For parametric assumptions of $J$, the estimation of the first and second moments is typically sufficient to characterize $J$, but, as evidenced by the real-world data, a finite number of moments of $J$ may not characterize the distribution well. In these situations, a generalization of kernel density estimation for the space of graphs is introduced, and the limitations of such a perspective are explored as a function of sample size. It was shown experimentally that for finite $n$, we cannot expect to determine the ``bandiwdth'' for the estimator proposed in \ref{eqn:KernDensEst} for every sized sample $N$. 
	
	The limitations of the family of stochastic block models as kernels for kernel density estimation point to a fundamental limitation of modeling graphs when a Bernoulli process models the edge probabilities. To specify the limitations of this process we recall the results from \cite{FK81}, which were presented in equation (\ref{eqn:ExEqn}). The authors show that for an $\text{\ER}$ random graph with parameters $n$ and $p$, that
\begin{align}
	\lambda_1 \overset{d}{\to} N\left( (n-2)p + 1, 2p(1-p) + \cO(\frac{1}{\sqrt{n}})\right).
\end{align}
Recall that for an $\text{\ER}$, the entries of the adjacency matrix $\bm A$ is defined by the Bernoulli process
\begin{align}
	a_{ij} \sim \bern{p}.
\end{align}
The expected value and variance for each entry is given by
\begin{align}
	&\E{a_{ij}} = p\\
	&\Var(a_{ij}) = p(1-p).
\end{align}
Rewriting the results of \cite{FK81},
\begin{align}
	\lambda_1 \overset{d}{\to} N\left( (n-2)\E{a_{ij}} + 1, 2\Var(a_{ij}) + \cO(\frac{1}{\sqrt{n}})\right),
\end{align}
which shows clearly that while the edges of adjacency matrix, $\bm A$, are modeled by a one-parameter family such that the location and scale of each $a_{ij}$ is coupled, then coupling between the location and scale of the largest eigenvalues of a graph is also expected. These results are also seen in the inhomogeneous $\text{\ER}$ ensemble in equations (\ref{eqn:CovCCH}) and (\ref{eqn:EstLargeEigs-CCH}) of Theorem \ref{thm:CCH} where it can be inferred that any change in the eigenvalues or eigenfunctions of $L_f$, denoted $\theta_i$ and $r_i(x)$ respectively, change the values of equations (\ref{eqn:CovCCH}) and (\ref{eqn:EstLargeEigs-CCH}). It is for these reasons that a new model for graphs may need to be considered, one in which the location and scale of the entries of the adjacency matrix are perhaps independent.
%______________________________________________________________________________________________

\newpage
%_____________________________________________________________________________________________
\appendix
\begin{center}
\textbf{Appendix}\\
\end{center}
%_____________________________________________________________________________________________

We divide the appendix into five primary sections. First we introduce a classic theorem related to our work in \ref{app:Classic}. \ref{app:ApproxFV} is brief, in which Lemma \ref{lem:STFV_Cov} is proven. We next prove Corollary \ref{cor:FiniteMatrixRep} in \ref{app:FiniteRank}. \ref{app:FirstOrder} is devoted to the first-order computations of the eigenvalues and eigenvectors of stochastic block model graphs when $q = \epsilon p_{\min}$ where $p_{\min} = \min \bp$. \ref{app:RelErr} shows the computations for Lemma \ref{lem:RelErrJMoments}.
%_____________________________________________________________________________________________
\section{Classical results}
\label{app:Classic}
%_____________________________________________________________________________________________
\begin{theorem}[Weyl-Lidskii]\hfill \\
  \label{thm:SpecInc}
  Let $\bm H$ be a self-adjoint operator on a Hilbert space $\mathcal{H}$. Let $\bA$ be a bounded operator on $\mathcal{H}$ Let $\sigma(\bm H)$ and $\sigma(\bm H + \bA)$ denote the spectra of $\bm H$ and $(\bm H + \bA)$ respectively. Then
  \begin{equation}
\sigma(\bm H + \bA) \subset \left \lbrace \lambda : dist(\lambda,\sigma(\bm H)) \leq ||\bA|| \right \rbrace
\end{equation}
where $||\bA||$ denotes the operator norm of $\bA$.
\end{theorem}
\begin{proof}
These are standard bounds that can be found in many good books on matrix perturbation theory (e.g., \cite{SS90}).
\end{proof}

%_____________________________________________________________________________________________
\section{Approximately computing the sample Fr\'echet variance}
\label{app:ApproxFV}

%_____________________________________________________________________________________________
This appendix serves to prove Lemma \ref{lem:STFV_Cov}, which is a consequence of the following theorem. Let $\lbrace \Gk \rbrace_{k=1}^N$ be a sample of graphs with sample Fr\'echet mean $G_N^*$ with adjacency matrix $\bA_N^*$. Let $\bar{\blamb}$ denote the arithmetic mean of the $c$ largest eigenvalues.
%_____________________________________________________________________________________________
\begin{theorem}
  \label{thm:GeomEigsFM}
  $\forall \epsilon > 0$, $\exists n^* \in \N$ such that $\forall n > n^*$,
  \begin{equation}
    ||\sigma_c(\bA_N^*) - \bar{\blamb} ||_2 < \epsilon.
  \end{equation}
\end{theorem}
%_____________________________________________________________________________________________
%_____________________________________________________________________________________________
\begin{proof}
  The proof is in \cite{FM22}.
\end{proof}
%_____________________________________________________________________________________________
This shows that the $c$ largest eigenvalues of the sample Fr\'echet mean graph are approximated well by the sample mean eigenvalues when the graphs considered are sufficiently large. We utilize this fact to approximately compute the value of the sample total Fr\'echet variance.
%_____________________________________________________________________________________________
\begin{Lemma}[Lemma \ref{lem:STFV_Cov} from the main document]
\label{lem:STFV_Cov-app}
	Let $\lbrace \Gk \rbrace_{k=1}^N$ be a sample of graphs with sample Fr\'echet mean $G_N^*$ and total sample Fr\'echet variance $V_{N,tot}^*$. Let $\bar \blamb$ denote the arithmetic mean of the largest $c$ eigenvalues and $\hat{\Sigma}$ be the sample covariance matrix, then
	\begin{equation}
		\lim_{n \to \infty} \vert V_{N, tot}^* - \sum_{i=1}^c \hat{\Sigma}_{ii}\vert = 0
	\end{equation}
\end{Lemma}
%_____________________________________________________________________________________________
%_____________________________________________________________________________________________
\begin{lproof}
By Theorem \ref{thm:GeomEigsFM}, we have $\forall \epsilon > 0$, $\exists n^* \in \N$ such that for all $n > n^*$,
\begin{align}
	|| \sigma_c(\bA_N^*) - \bar{\blamb} ||_2 < \epsilon.
\end{align}
Observe the following,
\begin{align}
	\sum_{i=1}^c \hat{\Sigma}_{ii} &= \frac{1}{N-1} \sum_{k=1}^N (\blamb^{(k)} - \bar{\blamb})^T (\blamb^{(k)} - \bar{\blamb})\\
	&= \frac{1}{N-1} \sum_{k=1}^N ||\blamb^{(k)} - \bar{\blamb}||_2^2 \label{eqn:STFV-approx-sum}
\end{align}
Let $\blamb_N^* =  \sigma_c(\bA_N^*)$ and consider the sample total Fr\'echet variance.
\begin{align}
	V_{N, tot}^* &=  \frac{1}{N-1} \sum_{k=1}^N d_{A_c}^2(G_N^*,\Gk)\\
	&=\frac{1}{N-1} \sum_{k=1}^N || \sigma_c(\bA_N^*) - \sigma_c(\bA^{(k)}) ||_2^2\\
	&= \frac{1}{N-1} \sum_{k=1}^N || \blamb_N^* - \blamb^{(k)} ||_2^2. \label{eqn:STFV-approx}
\end{align}
Because the summation is finite, we only need to show that each term in the sums in equations (\ref{eqn:STFV-approx-sum}) and (\ref{eqn:STFV-approx}) are close.
\begin{align}
	||\blamb^{(k)} - \bar{\blamb}||_2  &= ||\blamb^{(k)} - \blamb_N^* + \blamb_N^*- \bar{\blamb}||_2\\
	&= ||\blamb^{(k)} - \blamb_N^*||_2 + ||\blamb_N^*- \bar{\blamb}||_2\\
	&\leq ||\blamb^{(k)} - \blamb_N^*||_2 + \epsilon.
\end{align}
Therefore,
\begin{align}
	\vert ||\blamb^{(k)} - \bar{\blamb}||_2 - ||\blamb^{(k)} - \blamb_N^*||_2 \vert &< \vert ||\blamb^{(k)} - \blamb_N^*||_2 + \epsilon - ||\blamb^{(k)} - \blamb_N^*||_2 \vert\\
	&= \epsilon.
\end{align}
As a consequence, the difference in the summations is
\begin{align}
	\left\vert V_{N,tot}^* - \sum_{i=1}^c \hat{\Sigma}_{ii} \right\vert &\leq \frac{1}{N-1} \sum_{k=1}^N \left\vert || \blamb_N^* - \blamb^{(k)} ||_2^2 - ||\blamb^{(k)} - \bar{\blamb}||_2^2 \right \vert\\
	&\leq \frac{1}{N-1} \sum_{k=1}^N \left\vert \epsilon \right \vert\\
	&= \frac{N}{N-1}\epsilon.
\end{align}
Therefore, the summations are arbitrarily close for sufficiently large graphs, $n > n^*$.
\end{lproof}

%_____________________________________________________________________________________________
\section{Proof of Corollary \ref{cor:FiniteMatrixRep}}
\label{app:FiniteRank}
%_____________________________________________________________________________________________
This appendix serves to prove Corollary \ref{cor:FiniteMatrixRep} which is a consequence of Theorem \ref{thm:CCH}. We restate both of these below for convenience. Let $f$ be a canonical stochastic block model kernel function and let $L_f$ be the associated linear integral operator with eigenfunctions $r_i(x)$ and eigenvalues denoted by $\theta_i = \lambda_i(L_f)$. Assume that $n^{-2/3} \ll \omega_n \ll 1$ and that $\lim_{n \to \infty} \omega_n = 0$. Because $\mu$ is always taken to be a stochastic block model kernel probability measure with parameters $\omega_n, \bm p,$ $q,$ and $\bs$, we denote the adjacency matrix of a random graph as $\bA_{\mu} = \bA_{\mu_{\omega_n, \bp, q, \bs}}$, where we have suppressed all the subscripts. 

\begin{theorem}[Chakrabarty, Chakraborty, Hazra 2020]
\label{thm:CCH-app}
\begin{align}
	\left(\omega_n^{-1/2}(\lambda_i(\bA_{\mu}) - \E{\lambda_i(\bA_{\mu})}[\mu]) \right) \overset{d}{\longrightarrow} (Z_i: 1 \leq i \leq c), \label{eqn:LimVar-CCH-app}
\end{align}
where the right hand side is a multivariate normal random vector in $\R^c$ with zero mean and
\begin{align}
	\Cov(Z_i, Z_j)  = 2 \int_0^1 \int_0^1 r_i(x) r_i(y) r_j(x) r_j(y) f(x,y) dx dy, \label{eqn:IntCov-app}
\end{align}
for all $1 \leq i,j \leq c$. The first order behavior of $\E{\lambda_i(\bA_{\mu})})$ is given by the following: For every $1 \leq i \leq c$,
 	\begin{equation}
		\E{\lambda_i(\bA_{\mu})} = \lambda_i(\bm B) + \mathcal{O}(\sqrt \omega_n + \frac{1}{n \omega_n}), \label{eqn:EstLargeEigs-CCH-app}
	\end{equation}
	where $\bm B$ is a $c \times c$ symmetric deterministic matrix defined by
	\begin{align}
	b_{j,l} = \sqrt{\theta_j \theta_l} n \omega_n \bm e^T_j \bm e_l + \theta_i^{-2} \sqrt{\theta_j \theta_l} (n \omega_n)^{-1} \bm e^T_j \E{(\bA_\mu - \E{\bA_\mu})^2} \bm e_l + \mathcal{O}(\frac{1}{n \omega_n}), \label{eqn:EstEigDisc-app}
	\end{align}
	and $\bm e_j$ is a vector with entries $\bm e_j(k) = \frac{1}{\sqrt n} r_j(\frac k n)$ for $1 \leq j \leq c$.
\end{theorem}
\begin{proof}
	This is a compilation of Theorems 2.3 and 2.4 from \cite{CCH20}.
\end{proof}
Define the matrices
\begin{align}
		\bm M = 
	\begin{bmatrix} 
		s_1 p_1 & \sqrt{s_1 s_2} q & \dots & \sqrt{s_1 s_c} q\\
		\sqrt{s_2 s_1} q & s_2 p_2 & \dots & \sqrt{s_2 s_c} q\\
		\vdots &  \vdots & \ddots & \vdots \\
		\sqrt{s_c s_1} q & \sqrt{s_c s_2} q & \dots & s_c p_c
	 \end{bmatrix} \quad 
	 	\bm M_{f} = 
	\begin{bmatrix} 
		 p_1 & q & \dots &  q\\
		 q & p_2 & \dots & q\\
		\vdots &  \vdots & \ddots & \vdots \\
		 q &  q & \dots &  p_c
	 \end{bmatrix} \label{eqn:FiniteMatrices-app}
\end{align}
where $\nu_k$ and $\bm v_k$ are the eigenvalues and eigenvectors of $\bm M$ respectively. 
\begin{corollary}
\label{cor:FiniteMatrixRep-app}
\begin{align}
	\left(\omega_n^{-1/2}(\lambda_i(\bA_{\mu}) - \E{\lambda_i(\bA_{\mu})}[\mu]) \right) \overset{d}{\longrightarrow} (Z_i: 1 \leq i \leq c), \label{eqn:LimVar-app}
\end{align}
where the right hand side is a multivariate normal random vector in $\R^c$ with zero mean and
\begin{align}
	\Cov(Z_i, Z_j)  = 2 \left( \bm v_k .* \bm v_j \right)^T \bm M_f \left( \bm v_k .* \bm v_j \right) \label{eqn:FirstOrderVar-app},
\end{align}
for all $1 \leq i,j \leq c$ and $.*$ denotes the component-wise product of the vectors.

The first order behavior of $\E{\lambda_i(\bA_{\mu})})$ is given by the following: For every $1 \leq i \leq c$,
 	\begin{equation}
		\E{\lambda_i(\bA_{\mu})} = \lambda_i(\bm B^*) + \mathcal{O}(\sqrt{\omega_n}) \label{eqn:FirstOrderMean-app}
	\end{equation}
	where $\bm B^* =  \bm B^{*,(1)} + \bm B^{*,(2)}$ whose components are given as
\begin{align}
	\left( \bm B^{*,(1)} \right)_{j,l} &= b_{j,l}^{*,(1)} = \begin{cases}  \nu_j n \omega_n \quad \text{if } j = l \\ 0 \quad \text{if } j \neq l \end{cases} \label{eqn:defB1-app}\\
	\left( \bm B^{*,(2)} \right)_{j,l} &= b_{j,l}^{*,(2)} = \nu_i^{-2} \sqrt{\nu_j \nu_l}  \sum_{k=1}^c \nu_k  \sum_{m=1}^c \frac{1}{\sqrt{s_m}} \bm v_j(m) \bm v_l(m) \bm v_k(m)\sum_{w=1}^c \sqrt{s_w} \bm v_k(w). \label{eqn:defB2-app}
\end{align}
\end{corollary}
%_____________________________________________________________________________________________
\begin{cproof}
	The proof is a consequence of Theorems \ref{thm:CCH-app} and \ref{thm:ApproxEigsSimp-app} along with Lemmas \ref{lem:eigFuncEigVec} and \ref{lem:ApproxEigsSimp-M}.
\end{cproof}
%_____________________________________________________________________________________________
To connect the corollary to Theorem \ref{thm:ApproxEigsSimp-app}, we introduce the following theorem from \cite{FM22}, which shows that the terms of the matrix $\bm B$ defined element-wise by equation (\ref{eqn:EstEigDisc-app}) can be estimated by a similar matrix $\bm B^*$.
%_____________________________________________________________________________________________
\begin{theorem}
 	\label{thm:ApproxEigsSimp-app}
 	For every $1 \leq i \leq c$,
 	\begin{equation}
		\E{\lambda_i(\bA_{\mu_{\omega_n f}})} = \lambda_i(\bm B^*) + \mathcal{O}(\sqrt \omega_n)
	\end{equation}
	where $\bm B^* =  \bm B^{*,(1)} + \bm B^{*,(2)}$ whose components are given as
\begin{align}
	\left( \bm B^{*,(1)} \right)_{j,l} &= b_{j,l}^{*,(1)} = \sqrt{\theta_j \theta_l} n \omega_n \int_0^1 r_j(x) r_l(x) dx = \begin{cases}  \theta_j n \omega_n \quad j = l \\ 0 \quad j \neq l \end{cases} \label{eqn:first}\\
	\left( \bm B^{*,(2)} \right)_{j,l} &= b_{j,l}^{*,(2)} = \theta_i^{-2} \sqrt{\theta_j \theta_l} \int_0^1 r_j(x) r_l(x) \int_0^1 f(x,y) dy dx. \label{eqn:second}
\end{align}
 \end{theorem}
 %_____________________________________________________________________________________________
 %_____________________________________________________________________________________________
\begin{proof}
	The proof is in \cite{FM22}. Notably, this is a minor modification to Theorem 2.4 in \cite{CCH20} when assuming that $\omega_n \to 0$. The proof relies on the fact that the terms in equation (\ref{eqn:EstEigDisc-app}) are a discretization of the quantities defined in equations (\ref{eqn:first}) and (\ref{eqn:second}), and that because the eigenfunctions are piecewise Lipschitz, the discretization converges at the rate $\cO(\frac 1 n)$.
\end{proof}
%_____________________________________________________________________________________________

We show the following three equalities within this appendix;
\begin{align}
	&\theta_j n \omega_n = \nu_j n \omega_n \label{eqn:equality1}\\
	&\theta_i^{-2} \sqrt{\theta_j \theta_l} \int_0^1 r_j(x) r_l(x) \int_0^1 f(x,y) dy dx = \nu_i^{-2} \sqrt{\nu_j \nu_l}  \sum_{k=1}^c \nu_k  \sum_{m=1}^c \frac{1}{\sqrt{s_m}} \bm v_j(m) \bm v_l(m) \bm v_k(m)\sum_{w=1}^c \sqrt{s_w} \bm v_k(w) \label{eqn:equality2}\\
	&2 \int_0^1 \int_0^1 r_i(x) r_i(y) r_j(x) r_j(y) f(x,y) dx dy = 2 \left( \bm v_k .* \bm v_j \right)^T \bm M_f \left( \bm v_k .* \bm v_j \right). \label{eqn:equality3}
\end{align}

The structure of the proof is as follows,
\begin{enumerate}
	\item In Lemma \ref{lem:eigFuncEigVec}, we show two quantities: (1) the eigenfunctions of $L_f$ are defined by the components of the vector $\bm v$ and (2) $\theta_k$ is given by $\nu_k$.\\
	\item By relating the eigenfunctions to the components of the vectors, we show that the integrals in equations (\ref{eqn:IntCov-app}) and (\ref{eqn:second}) can be written in terms of the vectors $\bm v_k$. Equation (\ref{eqn:first}) is a straightforward consequence of part 2 in step 1 because there are effectively no integrals.
\end{enumerate}
\noindent \textbf{Step 1}
%_____________________________________________________________________________________________
\begin{Lemma}
\label{lem:eigFuncEigVec}
Let $\nu_k$ and $\bm v_k$ be eigenvalues and eigenvectors of $\bm M$ respectively. Let $\theta_k$ and $r_k(x)$ be the eigenvalues and eigenfunctions of $L_f$ respectively. We then have the following,
\begin{align}
	&\theta_k = \nu_k \label{eqn:evals}\\
	&r_k(x_i^*) = \frac{\bm v_k(i)}{\sqrt{s_i}} \quad i = 1,...,c. \label{eqn:efuncs}
\end{align}
\end{Lemma}
As shown in Section 3 of \cite{CCH20}, $r_k(x)$ is piecewise constant on $c$ different blocks. The values for each block are given by (\ref{eqn:efuncs}). We characterize the intervals where each $r_k(x)$ is piecewise constant as
\begin{align}
	S_i = \left[\sum_{j=1}^{i-1} s_j, \sum_{j=1}^{i} s_j \right) \subset [0,1].
\end{align}
%_____________________________________________________________________________________________
%_____________________________________________________________________________________________
\begin{lproof}
	We first define a few notations. Recall that $\bs$ is the vector of relative community sizes.
Observe that $S_i \cap S_j = \emptyset$ for all $i \neq j$ and that $\int_{S_i} dx = s_i$, a fact we use throughout the proof. Additionally, each eigenfunction $r_k(x)$ is piecewise constant on each interval $S_i$. Define the points 
\begin{align}
	x_i^* \in S_i
\end{align}
as a fixed point in each interval $S_i$. We first observe that
\begin{align}
	\int_0^1 r_k(x) dx &= \sum_{i=1}^c \int_{S_i} r_k(x) dx \\
	&= \sum_{i=1}^c r_k(x_i^*) \int_{S_i} dx\\
	&=  \sum_{i=1}^c r_k(x_i^*) s_i.
\end{align}
This fact will aid our proof significantly. Define the vector
\begin{align}
	\bm r_k^n = \begin{bmatrix} \sqrt{s_1} r_k(x_1^*) \\ \vdots \\ \sqrt{s_c} r_k(x_c^*) \end{bmatrix}.
\end{align}
Consider that
\begin{align}
	\langle \bm r_k^n, \bm r_j^n \rangle &= \sum_{i=1}^c s_i r_k(x_i^*) r_j(x_i^*) \\
	&= \sum_{i=1}^c r_k(x_i^*) r_j(x_i^*) \int_{S_i} dx\\
	&= \sum_{i=1}^c \int_{S_i} r_k(x) r_j(x) dx\\
	&= \int_0^1 r_k(x) r_j(x) dx 
\end{align}
where we have used the fact that both $r_k(x)$ and $r_j(x)$ are constant on the intervals $S_i$. This implies that
\begin{align}
	&\bm r_k^n \perp \bm r_j^n \quad \forall k \neq j\\
	&1 = \langle \bm r_k^n , \bm r_k^n \rangle.
\end{align}
We now consider
\begin{align}
	\sum_{k=1}^c \theta_k \bm r_k^n(i) \bm r_k^n(j) &= \sum_{k=1}^c \sqrt{s_i} \theta_k r_k(x_i^*) \sqrt{s_j}r_k(x_j^*)\\
	&= \sqrt{s_i s_j} \sum_{k=1}^c \theta_k r_k(x_i^*) \sqrt{s_j}r_k(x_j^*)\\
	&= \sqrt{s_i s_j} f(x_i^*, x_j^*)\\
	&= (\bm M)_{ij}
\end{align}
where we have used the definition
\begin{align}
	f(x,y) = \sum_{k=1}^c \theta_k r_k(x) r_k(y).
\end{align}
The conclusion is then
\begin{align}
	\bm M = \sum_{k=1}^n \theta_k \bm r_k^n (\bm r_k^n)^T,
\end{align}
and $\bm r_k^n$ is an eigenvector. Because $\bm M$ is given in terms of the parameters $\bp,q,$ and $\bs$, we may compute the eigenvalues and eigenfunctions of $\bm M$ as $\nu_k$ and $\bm v_k$, where $\bm v_k$ satisfies
\begin{align}
	\bm v_k = \bm r_k^n.
\end{align}
We can therefore determine the value of each eigenfunction at the points $x_i^*$ as
\begin{align}
	r_k(x_i^*) = \frac{\bm v_k(i) }{\sqrt{s_i}} ,
\end{align}
and the eigenvalues,
\begin{align}
	\theta_k =\nu_k.
\end{align}
Because the eigenfunctions are piecewise constant and we know one value within each piece, this defines the entire eigenfunction which shows equation (\ref{eqn:evals}) and equation (\ref{eqn:efuncs}).
\end{lproof}
%_____________________________________________________________________________________________
\noindent \textbf{Step 2}

The result of the prior lemma can now be used to calculate the quantities in equations (\ref{eqn:equality1}), (\ref{eqn:equality2}), and (\ref{eqn:equality3}) which will accomplish step 2.

%_____________________________________________________________________________________________
Recall the matrices
 	\begin{align}
 		&\bm M = 
	\begin{bmatrix} 
		s_1 p_1 & \sqrt{s_1 s_2} q & \dots & \sqrt{s_1 s_c} q\\
		\sqrt{s_2 s_1} q & s_2 p_2 & \dots & \sqrt{s_2 s_c} q\\
		\vdots &  \vdots & \ddots & \vdots \\
		\sqrt{s_c s_1} q & \sqrt{s_c s_2} q & \dots & s_c p_c
	 \end{bmatrix} \quad \bm M_{f} = 
	\begin{bmatrix} 
		 p_1 & q & \dots &  q\\
		 q & p_2 & \dots & q\\
		\vdots &  \vdots & \ddots & \vdots \\
		 q &  q & \dots &  p_c
	 \end{bmatrix} 
	 \end{align}
\begin{Lemma}
 	\label{lem:ApproxEigsSimp-M}
Let $\nu_k$, $\bm v_k$ be eigenvalues and eigenvectors of $\bm M$. Then, for every $1 \leq i,j,l \leq c$,
\begin{align}
       &\theta_j n \omega_n = \nu_j n \omega_n \label{eqn:FirstResult}\\
	 &\theta_i^{-2} \sqrt{\theta_j \theta_l} \int_0^1 r_j(x) r_l(x) \int_0^1 f(x,y) dy dx = \nu_i^{-2} \sqrt{\nu_j \nu_l}  \sum_{k=1}^c \nu_k  \sum_{m=1}^c \frac{1}{\sqrt{s_m}} \bm v_j(m) \bm v_l(m) \bm v_k(m)\sum_{w=1}^c \sqrt{s_w} \bm v_k(w) \label{eqn:SecondResult}\\
	 &2 \int_0^1 \int_0^1 r_i(x) r_i(y) r_j(x) r_j(y) f(x,y) dx dy = 2\left( \bm v_k .* \bm v_j \right)^T \bm M_f \bm v_k .* \bm v_j. \label{eqn:ThirdResult}
\end{align}
\end{Lemma}
%_____________________________________________________________________________________________
%_____________________________________________________________________________________________
\begin{lproof}
	The lemma is shown by using the expressions derived in Lemma \ref{lem:eigFuncEigVec} along with the summation representation of $f(x,y)$. Equation (\ref{eqn:FirstResult}) is trivial to show because
	\begin{align}
		\theta_j = \nu_j,
	\end{align}
	as shown in Lemma \ref{lem:eigFuncEigVec}. We now consider equation (\ref{eqn:SecondResult}). To show this equality, we show how to compute the integral 
	\begin{align}
		\int_0^1 r_j(x) r_l(x) \int_0^1 f(x,y) dy dx
	\end{align}
	in terms of the vectors $\bm v_k$. First we use the summation representation of $f(x,y)$ to simplify as follows
\begin{align}
	\int_0^1 r_j(x) r_l(x) \int_0^1 f(x,y) dy dx &= \int_0^1 r_j(x) r_l(x) \int_0^1 \sum_{k=1}^c \theta_k r_k(x) r_k(y) dy dx\\
	&= \sum_{k=1}^c \theta_k  \int_0^1 r_j(x) r_l(x) r_k(x)dx\int_0^1 r_k(y) dy.
\end{align}
	Next, we observe that the eigenfunction is piecewise constant on the same intervals defined by $S_i$,
\begin{align}
	\sum_{k=1}^c \theta_k  \int_0^1 r_j(x) r_l(x) r_k(x)dx\int_0^1 r_k(y) dy = \sum_{k=1}^c \theta_k  \sum_{m=1}^c \int_{S_m} r_j(x) r_l(x) r_k(x)dx \sum_{w=1}^c \int_{S_w} r_k(y) dy.
	\end{align}
	Now, because each eigenfunction is constant on the intervals, we may pull out the constant which is simply the function values evaluated at a point in the interval.
\begin{align}
	\sum_{k=1}^c \theta_k  \sum_{m=1}^c \int_{S_m} r_j(x) r_l(x) r_k(x)dx \sum_{w=1}^c \int_{S_w} r_k(y) dy = \sum_{k=1}^c \theta_k  \sum_{m=1}^c r_j(x_m^*) r_l(x_m^*) r_k(x_m^*) \int_{S_m} dx \sum_{w=1}^c r_k(x_w^*) \int_{S_w} dy
	\end{align}
	Observe now that the integral $\int_{S_m} dx = s_m$.
\begin{align}
	\sum_{k=1}^c \theta_k  \sum_{m=1}^c r_j(x_m^*) r_l(x_m^*) r_k(x_m^*) \int_{S_m} dx \sum_{w=1}^c r_k(x_w^*) \int_{S_w} dy = \sum_{k=1}^c \theta_k  \sum_{m=1}^c s_m r_j(x_m^*) r_l(x_m^*) r_k(x_m^*) \sum_{w=1}^c s_w r_k(x_w^*) 
	\end{align}
	We next utilize equation (\ref{eqn:efuncs})
\begin{align}
	\sum_{k=1}^c \theta_k  \sum_{m=1}^c s_m r_j(x_m^*) r_l(x_m^*) r_k(x_m^*) \sum_{w=1}^c s_w r_k(x_w^*) = \sum_{k=1}^c \theta_k  \sum_{m=1}^c s_m \frac{\bm v_j(m) }{\sqrt{s_m}} \frac{\bm v_l(m) }{\sqrt{s_m}} \frac{\bm v_k(m) }{\sqrt{s_m}} \sum_{w=1}^c s_w \frac{\bm v_k(w) }{\sqrt{s_w}}
	\end{align}
	and we then simplify, which yields the following:
\begin{align}
	\sum_{k=1}^c \theta_k  \sum_{m=1}^c s_m \frac{\bm v_j(m) }{\sqrt{s_m}} \frac{\bm v_l(m) }{\sqrt{s_m}} \frac{\bm v_k(m) }{\sqrt{s_m}} \sum_{w=1}^c s_w \frac{\bm v_k(w) }{\sqrt{s_w}}= \sum_{k=1}^c \theta_k  \sum_{m=1}^c \frac{1}{\sqrt{s_m}} \bm v_j(m) \bm v_l(m) \bm v_k(m)\sum_{w=1}^c \sqrt{s_w} \bm v_k(w).
\end{align} 
Replacing each $\theta_k = \nu_k$ as shown by equation (\ref{eqn:evals}) gives the result.
 
We derive how to express equation (\ref{eqn:IntCov-app}) in terms of the eigenvalues and eigenvectors of $\bm M$ by way of a very similar computation,
 \begin{align}
\Cov(Z_i,Z_j)&= 2 \int_0^1 \int_0^1 r_i(x) r_i(y) r_j(x) r_j(y) f(x,y) dx dy\\
&= 2 \int_0^1 r_i(x) r_j (x) \int_0^1 \sum_{k=1}^c \theta_k r_k(x) r_i(y) r_j (y) r_k(y) dx dy\\
&= 2 \sum_{m=1}^c \int_{S_m} \sum_{w=1}^c \int_{S_w} r_i(x) r_j (x)\sum_{k=1}^c \theta_k r_k(x) r_i(y) r_j (y) r_k(y) dx dy.
\end{align}
We now use the piecewise constant behavior of the functions,
\begin{align}
\Cov(Z_i,Z_j) &= 2 \sum_{m=1}^c \int_{S_m} \sum_{w=1}^c \int_{S_w} r_i(x_m^*) r_j (x_m^*)\sum_{k=1}^c \theta_k r_k(x_m^*) r_i(y_w^*) r_j (y_w^*) r_k(y_w^*) dx dy\\
&= 2 \sum_{m=1}^c s_m \sum_{w=1}^c s_w r_i(x_m^*) r_j (x_m^*)\sum_{k=1}^c \theta_k r_k(x_m^*) r_i(y_w^*) r_j (y_w^*) r_k(y_w^*)
\end{align}
where we have used $\int_{S_m} dx = s_m$ and the fact that there is no longer any $x$ or $y$ dependence in the equations. Recall now that
\begin{align}
	 f(x_m^*, y_w^*) = \sum_{k=1}^c \theta_k r_k(x_m^*) r_i(y_w^*).
\end{align}
Using this and writing $s_w = \sqrt{s_w} \sqrt{s_w}$,
\begin{align}
\Cov(Z_i,Z_j) &= 2 \sum_{m=1}^c \sum_{w=1}^c  \sqrt{s_m} r_i(x_m^*) \sqrt{s_m} r_j (x_m^*) f(x_m^*, x_w^*) \sqrt{s_w} r_j (y_w^*) \sqrt{s_w} r_k(y_w^*).
\end{align}
The right hand side of the above equation can be written as $\bm h^T \bm M_f \bm h$ where
\begin{align}
	\bm h(m) = \sqrt{s_m} r_i(x_m^*) \sqrt{s_m}r_j (x_m^*).
\end{align}
Writing in terms of the vector $\bm v_k$ we see that
\begin{align}
	\bm h = \bm v_i .* \bm v_j,
\end{align}
where $.*$ denotes the component wise product. Therefore,
\begin{align}
	\Cov(Z_i,Z_j) = 2\left( \bm v_i .* \bm v_j \right)^T \bm M_f \bm v_i .* \bm v_j 
\end{align}
\end{lproof}
%_____________________________________________________________________________________________
\section{First-Order Computations}
\label{app:FirstOrder}
%_____________________________________________________________________________________________
Throughout this appendix we prove Lemma \ref{lem:FirstOrder}, which is restated below for convenience.
\begin{Lemma}[Lemma \ref{lem:FirstOrder} from the main document]
\label{lem:FirstOrder-app}
	Let $\bm M$ and $\bm M_f$ be as defined in equation (\ref{eqn:FiniteMatrices}). Assume $q = \epsilon p_{\min}$ where $p_{\min} = \underset{i = 1,...,c}{\min}\bp$ and $\epsilon \ll 1$.
	\begin{align}
	&\left \vert \frac{\E{\lambda_i(\bA)}}{ n \omega_n s_i} - p_i \right \vert = \cO\left(\epsilon^2 p_{\min}\right) + \cO \left(\frac{t_i(\bp)}{n \omega_n s_i}\right), \label{eqn:FirstOrderErrMean-app}\\
	&\left \vert \Cov(Z_i, Z_j) - \begin{cases}2 p_i \quad \text{if } i =j\\ 0 \quad \text{if } i \neq j\end{cases}\right \vert = \cO(\epsilon^2 p_{\min}^2), \label{eqn:FirstOrderErrVar-app}
\end{align}
where $t_i(\bp)$ is a bounded function of the parameters $\bp$ and $Z$ is defined in Theorem \ref{thm:CCH}.
\end{Lemma}
\begin{lproof}
	The proof is a consequence of Lemmas \ref{lem:RelErrEigs} and \ref{lem:RelErrCov}
\end{lproof}

The ideas for the proof are related to the quantities in Corollary \ref{cor:FiniteMatrixRep}. Recall that
\begin{align}
	\E{\lambda_i(\bA)} = \lambda_i(\bm B^*) + \cO(\sqrt{\omega_n}),\\
	\Cov(Z_i,Z_j) = 2\left( \bm v_i .* \bm v_j \right)^T \bm M_f \bm v_i .* \bm v_j.
\end{align}
Within this appendix, we always assume that $q = \epsilon p_{\min}$ where $\epsilon \ll 1$ and $p_{\min} = \min \bp$. Under these assumptions we calculate a first-order expansion for $\lambda_i(\bm B^*) $ and $2\left( \bm v_i .* \bm v_j \right)^T \bm M_f \bm v_i .* \bm v_j.$ to prove Lemma \ref{lem:FirstOrder-app}. To accomplish these tasks, we split this appendix into three subsections
\begin{enumerate}
	\item In \ref{subsec:prelim}, we show the first order behavior of the eigenvectors and eigenvalues of $\bm M$ when $q = \epsilon p_{\min}$.
	\item In \ref{subsec:ExpEigFirst}, we show equation (\ref{eqn:FirstOrderErrMean-app}).
	\item In \ref{subsec:VarFirst}, we show equation (\ref{eqn:FirstOrderErrVar-app}).
\end{enumerate}
%________________________________________________________________________________________________
\subsection{Step 1: Preliminaries}
\label{subsec:prelim}
%________________________________________________________________________________________________
Assume that $q = \epsilon p_{\min}$, where $\epsilon \ll 1$ and $p_{\min} = \min \bp$. Let $\bm M$ be defined as in equation (\ref{eqn:FiniteMatrices}) with eigenvalues and eigenvectors given by $\nu_k$ and $\bm v_k$, respectively. The matrix $\bm M$ may be split into two components,
\begin{align}
	&\bm M_0 = 
	\begin{bmatrix} 
		s_1 p_1 & 0 & \dots & 0\\
		0 & s_2 p_2 & \dots & 0\\
		\vdots &  \vdots & \ddots & \vdots \\
		0 & 0 & \dots & s_c p_c
	 \end{bmatrix}\\
	 \epsilon p_{\min} &\bm M_1 = 
	\begin{bmatrix} 
		0 & \sqrt{s_1 s_2}  & \dots & \sqrt{s_1 s_c} \\
		\sqrt{s_2 s_1} q & 0 & \dots & \sqrt{s_2 s_c} \\
		\vdots &  \vdots & \ddots & \vdots \\
		\sqrt{s_c s_1}  & \sqrt{s_c s_2}  & \dots & 0
	 \end{bmatrix}\\
	&\bm M = \bm M_0 + \epsilon p_{\min} \bm M_1
\end{align}
Let $\gamma_i$ and $\bm g_i$ be the eigenvalues and eigenvectors of $\bm M_0$. Clearly,
\begin{align}
	&\gamma_i = s_i p_i,\\
	&\bm g_i = \bm e_i,
\end{align}
where $\bm e_i$ denotes the canonical basis vector in $\R^c$. We assume that the eigenvalues of $\bm M_0$, denoted by $\gamma_i$, are well separated so that the first order expansion of the eigenvectors is well-behaved. This is an assumption of the model and is the same assumption that the eigenvalues of $L_f$ are well-separated. An expansion for the eigenvectors of $\bm M$, denoted by $\bm v_i$, in terms of the eigenvectors of $\bm M_0$ is then
\begin{align}
	\bm v_i &= \bm e_i + \epsilon p_{\min} \sum_{k \neq i}^c \frac{\bm e_i \bm M_1 \bm e_k}{\gamma_i - \gamma_k} \bm e_k + \cO((\epsilon p_{\min})^2)\\
	&= \bm e_i + \epsilon p_{\min}  \sum_{k \neq i}^c \frac{\sqrt{s_i s_k}}{s_i p_i - s_k p_k} \bm e_k + \cO((\epsilon p_{\min})^2).
\end{align}

Within this section, we also determine a first order expansion for the eigenvalues of $\bm M$, denoted by $\nu_i$. 
\begin{align}
	\nu_i &= \bm v_i^T \bm M \bm v_i\\
	&= \left(\bm e_i + \epsilon p_{\min} \sum_{k \neq i}^c \frac{\sqrt{s_i s_k}}{s_i p_i - s_k p_k} \bm e_k\right)^T \bm M \left(\bm e_i + \epsilon p_{\min} \sum_{k \neq i}^c \frac{\sqrt{s_i s_k}}{s_i p_i - s_k p_k} \bm e_k\right)+ \cO((\epsilon p_{\min})^2)\\
	&= \bm e_i^T \bm M \bm e_i + \epsilon p_{\min} \sum_{k \neq i}^c \frac{\sqrt{s_i s_k}}{s_i p_i - s_k p_k} \bm e_k^T \bm M \bm e_i + \epsilon p_{\min} \bm e_i^T \bm M \sum_{k \neq i}^c \frac{\sqrt{s_i s_k}}{s_i p_i - s_k p_k} \bm e_k + \cO((\epsilon p_{\min})^2).
\end{align}
We have the following identity to help simplify:
\begin{align}
	\bm e_i^T \bm M \bm e_k = (\bm M)_{ik} = m_{ik} = m_{ki}
\end{align}
since $\bm M$ is symmetric. Therefore,
\begin{align}
	\nu_i &= \bm e_i^T \bm M \bm e_i + \epsilon p_{\min} \sum_{k \neq i}^c \frac{\sqrt{s_i s_k}}{s_i p_i - s_k p_k} \bm e_k^T \bm M \bm e_i + \epsilon p_{\min} \bm e_i^T \bm M \sum_{k \neq i}^c \frac{\sqrt{s_i s_k}}{s_i p_i - s_k p_k} \bm e_k + \cO((\epsilon p_{\min})^2)\\
	&= \bm e_i^T \bm M \bm e_i + 2\epsilon p_{\min}\sum_{k \neq i}^c \frac{\sqrt{s_i s_k}}{s_i p_i - s_k p_k} \bm e_k^T \bm M \bm e_i + \cO((\epsilon p_{\min})^2)\\
	&= m_{ii} + 2\sum_{k \neq i}^c \epsilon p_{\min} \frac{\sqrt{s_i s_k}}{s_i p_i - s_k p_k} m_{ki} + \cO((\epsilon p_{\min})^2)\\
	&= s_i p_i+ 2\epsilon p_{\min} \sum_{k \neq i}^c \frac{\sqrt{s_i s_k}}{s_i p_i - s_k p_k} \epsilon p_{\min} \sqrt{s_i s_k} + \cO((\epsilon p_{\min})^2)\\
	&= s_i p_i + 2 (\epsilon p_{\min})^2 \sum_{k \neq i}^c \frac{s_i s_k}{s_i p_i - s_k p_k}  + \cO((\epsilon p_{\min})^2)\\
	&= s_i p_i + \cO((\epsilon p_{\min})^2)\\
	&= \gamma_i + \cO((\epsilon p_{\min})^2). \label{eqn:CompEigsExp}
\end{align}

%__________________________________________________________________________________________
\subsection{Step 2: Expected eigenvalues}
\label{subsec:ExpEigFirst}
%__________________________________________________________________________________________
This section of the appendix computes equation (\ref{eqn:FirstOrderErrMean-app}). We begin this section by recalling the first order estimate of the expected eigenvalues given by Corollary \ref{cor:FiniteMatrixRep} in terms of the matrix $\bm B^*$.

\begin{equation}
		\E{\lambda_i(\bA_{\mu})} = \lambda_i(\bm B^*) + \mathcal{O}(\sqrt{\omega_n}),
\end{equation}
	where $\bm B^* =  \bm B^{*,(1)} + \bm B^{*,(2)}$ whose components are given as
\begin{align}
	\left( \bm B^{*,(1)} \right)_{j,l} &= b_{j,l}^{*,(1)} = \begin{cases}  \nu_j n \omega_n \quad j = l \\ 0 \quad j \neq l \end{cases} \label{eqn:defB1}\\
	\left( \bm B^{*,(2)} \right)_{j,l} &= b_{j,l}^{*,(2)} = \nu_i^{-2} \sqrt{\nu_j \nu_l}  \sum_{k=1}^c \nu_k  \sum_{m=1}^c \frac{1}{\sqrt{s_m}} \bm v_j(m) \bm v_l(m) \bm v_k(m)\sum_{w=1}^c \sqrt{s_w} \bm v_k(w). \label{eqn:defB2}
\end{align}

Observe that the eigenvalues of $\bm B^*$ are determined by the eigenvalues and eigenvectors of the matrix $\bm M$.  The results of this subsection are summarized by the following lemma, which shows equation (\ref{eqn:FirstOrderErrMean-app}) and its proof.
\begin{Lemma}
\label{lem:RelErrEigs}
Assume $q = \epsilon p_{\min}$, where $p_{\min} = \underset{i = 1,...,c}{\min}\bp$ and $\epsilon \ll 1$ then
\begin{align}
	\left \vert \frac{\E{\lambda_i(\bA)}}{ n \omega_n s_i} - p_i \right \vert = \cO(\epsilon^2 p_{\min}) + \cO\left(\frac{t_i(\bp)}{n \omega_n}\right).
\end{align}
\end{Lemma}
The proof of the above lemma takes the following intermediate steps. First we show that
\begin{align}
	\lambda_i(\bm B^*) = \lambda_i(\bm B^{*,(1)}) + \cO(t_i(\bp))
\end{align}
using Weyl-Lidskii's theorem. Next we show that 
\begin{align}
	\lambda_i(\bm B^{*,(1)}) = n \omega_n \left(s_i p_i + \cO(\epsilon^2 p_{\min}^2) \right).
\end{align}
We conclude the proof using the result that
\begin{align}
	\E{\lambda_i(\bA_{\mu})} = \lambda_i(\bm B^*) + \mathcal{O}(\sqrt{\omega_n}).
\end{align}
We now begin with the proof.
\begin{lproof}
	As stated, we first show that by Weyl-Lidskii's theorem, 
	\begin{align}
		\lambda_i(\bm B^*) = \lambda_i(\bm B^{*,(1)}) + \cO(t_i(\bp)).
	\end{align}
	For the theorem we take the following quantities.
	\begin{align}
		\bm A = \bm B^{*,(2)}\\
		\bm H = \bm B^{*,(1)}.
	\end{align}
	Then, $\bm B^* = \bm A + \bm H$. Furthermore, $\bm H$ is self-adjoint because it is diagonal and $\bm A$ is bounded. Since all the eigenvalues are real, we have the following result:
	\begin{align}
		\vert \lambda_i(\bm B^*) - \lambda_i( \bm B^{*,(1)}) \vert \leq ||\bm B^{*,(2)}||.
	\end{align}
	Observe that because $\bm B^{*,(2)}$ is independent of $n$, we may conclude that
	\begin{align}
		\vert \lambda_i(\bm B^*) - \lambda_i( \bm B^{*,(1)}) \vert = \cO(t_i(\bp)), \label{eqn:step1lem}
	\end{align}
	where $t_i(\bp)$ is a bounded function of the parameters independent of $n$. This accomplishes the first step of the proof. Next we show that
	\begin{align}
		\lambda_i( \bm B^{*,(1)}) = n \omega_n \left(s_i p_i + \cO(\epsilon^2 p_{\min}^2) \right).
	\end{align}
	This is nearly trivially true. Observe that because $\bm B^{*,(1)}$ is diagonal, the $i$-th eigenvalue is given as follows,
	\begin{align}
		\lambda_i( \bm B^{*,(1)}) = b^{*,(1)}_{i,i} = \nu_i n \omega_n.
	\end{align}
	Because of equation (\ref{eqn:CompEigsExp}), we have
	\begin{align}
		\nu_i n \omega_n = n \omega_n \left( s_i p_i + \cO(\epsilon^2 p_{\min}^2)\right).
	\end{align}
	To conclude, we have the following set of equalities:
	\begin{align}
		\E{\lambda_i(\bA)} = \lambda_i(\bm B^*) + \cO(\sqrt{\omega_n}).
	\end{align}
	Equation (\ref{eqn:step1lem}) shows that
	\begin{align}
		\E{\lambda_i(\bA)} = \lambda_i( \bm B^{*,(1)}) + \cO(t_i(\bp))
	\end{align}
	Replacing $ \lambda_i( \bm B^{*,(1)})$ with $n \omega_n \left( s_i p_i + \cO(\epsilon^2 p_{\min}^2)\right)$ yields
	\begin{align}
		\E{\lambda_i(\bA)} =n \omega_n \left( s_i p_i + \cO(\epsilon^2 p_{\min}^2)\right) + \cO(t_i(\bp))
	\end{align}
	and finally dividing by $n \omega_n s_i$ shows that
	\begin{align}
		\frac{\E{\lambda_i(\bA)}}{n \omega_n s_i} = p_i + \cO(\frac{\epsilon^2 p_{\min}^2}{s_i}) + \cO(\frac{t_i(\bp)}{n \omega_n})
	\end{align}
	Because $s_i$ is a constant with respect to $n$ and $\epsilon$, we simplify this expression as
	\begin{align}
		\frac{\E{\lambda_i(\bA)}}{n \omega_n s_i} = p_i + \cO(\epsilon^2 p_{\min}^2) + \cO(\frac{t_i(\bp)}{n \omega_n }).
	\end{align}
	We conclude the proof with the following
	\begin{align}
		\left \vert \frac{\E{\lambda_i(\bA)}}{n \omega_n s_i} - p_i \right \vert = \cO(\epsilon^2 p_{\min}^2) + \cO(\frac{t_i(\bp)}{n \omega_n}).
	\end{align}
\end{lproof}
It is worth noting that typically, the function $t_i(\bp)$ is suppressed when reporting the error. In this instance, we make this term explicit because the eventual distribution on the parameters, $J$, for the random parameter stochastic block model, impacts this error term.

The next subsection shows similar calculations except for the term $\Cov(Z_i, Z_j)$ as defined in Corollary \ref{cor:FiniteMatrixRep}.
%________________________________________________________________________________________________
\subsection{Step 3: Covariance}
\label{subsec:VarFirst}
%________________________________________________________________________________________________
This section of the appendix is summarized by the following lemma which shows equation (\ref{eqn:FirstOrderErrVar-app}),
\begin{Lemma}
\label{lem:RelErrCov}
Assume $q = \epsilon p_{\min}$, where $p_{\min} = \underset{i = 1,...,c}{\min}\bp$ and $\epsilon \ll 1$ then
	\begin{align}
		\left \vert \Cov(Z_k, Z_l) - \begin{cases}2 p_k \quad \text{if } k = l\\ 0 \quad \text{if } k \neq l\end{cases}\right \vert = \cO(\epsilon^2 p_{\min}^2),
	\end{align}
	where $Z$ is defined in Theorem \ref{thm:CCH}.
\end{Lemma}
\begin{lproof}
	The proof is split into two subsections. First, we analyze the case of $k \neq l$ and show that all contributions to covariance occur at the second order. Then, we consider the behavior when $k = l$. To begin we recall a few quantities,
	\begin{align}
	\Cov(Z_k, Z_l)  = 2 \left( \bm v_k .* \bm v_l \right)^T \bm M_f \left( \bm v_k .* \bm v_l \right).
\end{align}
Recall matrix $\bm M_f$,
\begin{align}
\bm M_{f} = 
	\begin{bmatrix} 
		 p_1 & \epsilon p_{\min} & \dots &  \epsilon p_{\min}\\
		 \epsilon p_{\min} &p_2 & \dots & \epsilon p_{\min}\\
		\vdots &  \vdots & \ddots & \vdots \\
		 \epsilon p_{\min} &  \epsilon p_{\min} & \dots &  p_c
	 \end{bmatrix} 
\end{align}
though it is important to remember that $\bm v_k$ is an eigenvector of $\bm M$ and not $\bm M_f$.

We first compute $\bm v_k .* \bm v_l$ to first order for any choice of $k$ and $l$.
\begin{align}
	\bm v_k .* \bm v_l &= \left( \bm e_k + \sum_{j \neq k}^c \frac{\epsilon p_{\min} \sqrt{s_k s_j}}{s_k p_k - s_j p_j} \bm e_j \right).*\left( \bm e_l + \sum_{r \neq l}^c \frac{\epsilon p_{\min} \sqrt{s_l s_r}}{s_l p_l - s_r p_r} \bm e_r \right) + \cO((\epsilon p_{\min})^2)\\
	&= \bm e_k .* \bm e_l + \sum_{j \neq k}^c \frac{\epsilon p_{\min} \sqrt{s_k s_j}}{s_k p_k - s_j p_j} \bm e_j.* \bm e_l + \sum_{r \neq l}^c \frac{\epsilon p_{\min} \sqrt{s_l s_r}}{s_l p_l - s_r p_r} \bm e_k .*\bm e_r + \cO((\epsilon p_{\min})^2).
\end{align}

We now show that if $k \neq l$ then the contribution to covariance is only on the order of $\cO((\epsilon p_{\min})^2)$. To see this, observe that if $k \neq l$,
\begin{align}
	\bm v_k .* \bm v_l = \sum_{j \neq k}^c \frac{\epsilon p_{\min} \sqrt{s_k s_j}}{s_k p_k - s_j p_j} \bm e_j.* \bm e_l + \sum_{r \neq l}^c \frac{\epsilon p_{\min} \sqrt{s_l s_r}}{s_l p_l - s_r p_r} \bm e_k .*\bm e_r + \cO((\epsilon p_{\min})^2).
\end{align}
Therefore, the vector
\begin{align}
	\bm v_k .* \bm v_l = \epsilon p_{\min} \bm d^{k,l} + \cO((\epsilon p_{\min})^2).
\end{align}
where $\bm d^{k,l}$ is some vector that depends on $k,l$. Computing $\Cov(Z_k, Z_l)$ we see 
\begin{align}
	\Cov(Z_k, Z_l)  &= 2 \left( \bm v_k .* \bm v_l \right)^T \bm M_f \left( \bm v_k .* \bm v_l \right)\\
	&= 2 \epsilon p_{\min} (\bm d^{k,l})^T \bm M_f \left(\epsilon p_{\min} \bm d^{k,l} \right) + \cO((\epsilon p_{\min})^2)\\
	&= 2 (\epsilon p_{\min})^2 (\bm d^{k,l})^T \bm M_f  \bm d^{k,l} + \cO( (\epsilon p_{\min})^2)\\
	&= \cO( (\epsilon p_{\min})^2). \label{eqn:kneql}
\end{align}
Next, we show the behavior of $\Cov(Z_k,Z_l)$ for $k = l$. We have
\begin{align}
	\bm v_k .* \bm v_k &= \bm e_k .* \bm e_k + \sum_{j \neq k}^c \frac{\epsilon p_{\min} \sqrt{s_k s_j}}{s_k p_k - s_j p_j} \bm e_j.* \bm e_k + \sum_{j \neq k}^c \frac{\epsilon p_{\min} \sqrt{s_k s_j}}{s_k p_k - s_j p_j} \bm e_k .*\bm e_j + \cO((\epsilon p_{\min})^2).
\end{align}
However, if $j \neq k$, then $\bm e_j.* \bm e_k = \bm 0$. Therefore,
\begin{align}
	\bm v_k .* \bm v_k &= \bm e_k .* \bm e_k + \cO((\epsilon p_{\min})^2)\\
	&=\bm e_k + \cO((\epsilon p_{\min})^2)
\end{align}
because $\bm e_k$ is the canonical basis vector. To compute the variance, we need only to compute
\begin{align}
	\Cov(Z_k, Z_k)  &= 2 \left( \bm v_k .* \bm v_k \right)^T \bm M_f \left( \bm v_k .* \bm v_k \right)\\
	&= 2 \bm e_k^T \bm M_f \bm e_k + \cO((\epsilon p_{\min})^2)\\
	&= 2 p_k + \cO((\epsilon p_{\min})^2). \label{eqn:k=l}
\end{align}
The combination of equations (\ref{eqn:kneql}) and (\ref{eqn:k=l}) yields
\begin{align}
	\left \vert \Cov(Z_k, Z_l) - \begin{cases}2 p_k \quad \text{if } k = l\\ 0 \quad \text{if } k \neq l\end{cases}\right \vert = \cO(\epsilon^2 p_{\min}^2).
\end{align}
\end{lproof}
%____________________________________________________________________________________________
\section{Proof of Lemma \ref{lem:RelErrJMoments}}
\label{app:RelErr}
%____________________________________________________________________________________________
This appendix serves to prove Lemma \ref{lem:RelErrJMoments}, which is restated below for convenience.
\begin{Lemma}[Lemma \ref{lem:RelErrJMoments} from the main document]
\label{lem:RelErrJMoments-app}
Let $\bm P_j$ be an observation from $J$ with components $P_i$. Let $P_{\min} = \min_{i=1,...,c} P_i$ and define $q = \epsilon P_{\min}$ where $\epsilon \ll 1$.
\begin{align}
	\frac{\E{\lambda_i}[H_n]}{n \omega_n s_i} &= \E{P_i}[J] + \cO(\epsilon^2 ) + \cO\left(\frac{1}{n \omega_n}\right) \label{eqn:First-E}\\
 	\frac{\Cov_{H_{n}}(\lambda_i, \lambda_j)}{n^2 \omega_n^2 s_i s_j}&+ \begin{cases} \frac{2\E{\lambda_i}[H_n]}{n^3 \omega_n^2 s_i^3} \quad \text{if } i = j \\ 0 \quad \text{if } i \neq j \end{cases}\\
 	 &= \Cov_{J}\left(P_i + \cO(\epsilon^2 P_{\min}^2) + \cO\left(\frac{t_i(\bp)}{n \omega_n }\right),  P_j + \cO(\epsilon^2 P_{\min}^2) + \cO\left(\frac{t_j(\bp)}{n \omega_n}\right)\right) +  \cO\left(\frac{\epsilon^2}{n^2 \omega_n }\right)
\end{align}
where $t_i(\bp)$ is a bounded function of the parameters for each $i$.
\end{Lemma}
\begin{lproof}
	The proof is a consequence of Propositions \ref{prop:FirstPart}  and \ref{prop:FourthPart} which are direct consequences of Lemma \ref{lem:FirstOrder}.
\end{lproof}
The proof for the above lemma is given in four parts. 
\begin{enumerate}
	\item First we show in Proposition \ref{prop:FirstPart} that
\begin{align}
	\frac{\E{\lambda_i}[H_n]}{n \omega_n s_i} &= \E{P_i}[J] + \cO(\epsilon^2 ) + \cO(\frac{1}{n \omega_n}).
\end{align}
	\item We then turn our attention to the covariance terms. In Proposition \ref{prop:SecondPart}, we show that
\begin{align}
	\E{\Cov(Z_i, Z_j)}[J] = \begin{cases}\E{2 P_i}[J] + \cO(\epsilon^2) \quad \text{if } i = j \\ \cO(\epsilon^2) \quad \text{if } i \neq j \end{cases}.
\end{align}
	\item Next, we show in Proposition \ref{prop:ThirdPart} that
 \begin{align}
 	\frac{\Cov_{J}\left(\E{\lambda_i(\bA)}[\mu],\E{\lambda_j(\bA)}[\mu] \vert \bm P_{J} = \bp\right)}{n^2 \omega_n^2 s_i s_j} = \Cov_{J}\left(P_i + \cO(\epsilon^2 P_{\min}) + \cO(\frac{t_i(\bp)}{n \omega_n}) ,P_j + \cO(\epsilon^2 P_{\min}) + \cO(\frac{t_j(\bp)}{n \omega_n})\right).
 \end{align}
	\item After recalling the definition of the first moment of $H_n$, we have in Proposition \ref{prop:FourthPart} that
 \begin{align}
 	\frac{\Cov_{H_{n}}(\lambda_i, \lambda_j)}{n^2 \omega_n^2 s_i s_j}&+ \begin{cases} \frac{2\E{\lambda_i}[H_n]}{n^3 \omega_n^2 s_i^3} \quad \text{if } i = j \\ 0 \quad \text{if } i \neq j \end{cases}\\
 	 &= \Cov_{J}\left(P_i + \cO(\epsilon^2 P_{\min}^2) + \cO\left(\frac{t_i(\bp)}{n \omega_n }\right),  P_j + \cO(\epsilon^2 P_{\min}^2) + \cO\left(\frac{t_j(\bp)}{n \omega_n}\right)\right) +  \cO\left(\frac{\epsilon^2}{n^2 \omega_n }\right)
\end{align}
where $t_i(\bp)$ is a bounded function of the parameters for each $i$.
 \end{enumerate}
 We conclude the proof by combining the results of Step 1 and Step 4.

\noindent \textbf{Step 1:}\\
This step is characterized by the following proposition.
\begin{Proposition}
 \label{prop:FirstPart}
 \begin{align}
 	\frac{\E{\lambda_i}[H_n]}{n \omega_n s_i} &= \E{P_i}[J] + \cO(\epsilon^2 ) + \cO(\frac{1}{n \omega_n})
\end{align}
 \end{Proposition}
 \begin{propproof}
 	The proof is a direct consequence of Lemma \ref{lem:FirstOrder}. First we recall the definition of $\E{\lambda_i}[H_n]$,
 	\begin{align}
 		\E{\lambda_i}[H_n] = \E{\E{\lambda_i(\bA)\vert \bm P_J = \bp}[\mu]}[J].
 	\end{align}
 	Lemma \ref{lem:FirstOrder} shows that
 	\begin{align}
 		\E{\lambda_i(\bA)\vert \bm P_J = \bp}[\mu] = p_i + \cO(\epsilon^2 p_{\min}^2) + \cO(\frac{t_i(\bp)}{n \omega_n s_i}).
 	\end{align}
 	Substitution yields
 	\begin{align}
 		\E{\lambda_i}[H_n] = \E{P_i + \cO(\epsilon^2 P_{\min}^2) + \cO(\frac{t_i(\bp)}{n \omega_n s_i})}[J],
 	\end{align}
 	where $P_i$ and $P_{\min}$ denote that these are now random variables. Now, because the support of $J$ is a subset of $[0,1]^c$ and $t_i(\bp)$ is a bounded function of the parameters, we conclude
 	\begin{align}
 		\E{\lambda_i}[H_n] = \E{P_i}[J] + \cO(\epsilon^2) +  \cO(\frac{1}{n \omega_n s_i}).
 	\end{align}
 \end{propproof}

\noindent \textbf{Step 2:} This step considers $\E{\Cov(Z_i, Z_j)}[J]$. It also follows as a result of Lemma \ref{lem:FirstOrder}
\begin{Proposition}
\label{prop:SecondPart}
	Let $Z$ be as defined in Corollary \ref{cor:FiniteMatrixRep}; then,
	\begin{align}
	\E{\Cov(Z_i, Z_j)}[J] = \begin{cases}\E{2 P_i}[J] + \cO(\epsilon^2) \quad \text{if } i = j \\ \cO(\epsilon^2) \quad \text{if } i \neq j \end{cases}.
	\end{align}
\end{Proposition}
\begin{propproof}
	 Lemma \ref{lem:FirstOrder} shows that
	 \begin{align}
	 \Cov(Z_i, Z_j)  = \begin{cases}2 p_i + \cO(\epsilon^2 p_{\min}^2) \quad \text{if } i = j\\ \cO(\epsilon^2 p_{\min}^2) \quad \text{if } i \neq j\end{cases}. 
	 \end{align}
	 Taking an expectation over $J$,
	 \begin{align}
	  \E{\Cov(Z_i, Z_j)}[J]  = \begin{cases}2 \E{P_i + \cO(\epsilon^2 P_{\min}^2)}[J] \quad \text{if } i = j\\ \E{\cO(\epsilon^2 P_{\min}^2)}[J] \quad \text{if } i \neq j\end{cases}
	 \end{align}
	 where $P_i$ and $P_{\min}$ denote random variables. As before, $\E{\cO(\epsilon^2 P_{\min}^2)}[J] = \cO(\epsilon^2)$ which gives the result
	 \begin{align}
	  \E{\Cov(Z_i, Z_j)}[J]  = \begin{cases}2 \E{P_i}[J]+ \cO(\epsilon^2 ) \quad \text{if } i = j\\ \cO(\epsilon^2) \quad \text{if } i \neq j\end{cases}
	 \end{align}
\end{propproof}
\noindent \textbf{Step 3:} This step shows the proper first order scaling of the term $\Cov_{J}\left(\E{\lambda_i(\bA)}[\mu],\E{\lambda_j(\bA)}[\mu] \vert \bm P_J = \bp \right)$. It is another result of Lemma \ref{lem:FirstOrder}.
\begin{Proposition}
\label{prop:ThirdPart}
 \begin{align}
 	\frac{\Cov_{J}\left(\E{\lambda_i(\bA)}[\mu],\E{\lambda_j(\bA)}[\mu] \vert \bm P_{J} = \bp\right)}{n^2 \omega_n^2 s_i s_j} = \Cov_{J}\left(P_i + \cO(\epsilon^2 P_{\min}) + \cO(\frac{1}{n \omega_n}) ,P_j + \cO(\epsilon^2 P_{\min}) + \cO(\frac{1}{n \omega_n})\right).
 \end{align}
\end{Proposition}
 \begin{propproof}
 	We have seen that
 	\begin{align}
 	\frac{\E{\lambda_i(\bA)}[\mu]}{ n \omega_n s_i} = p_i + \cO(\epsilon^2 p_{\min}) + \cO(\frac{t_i(\bp)}{n \omega_n s_i}). \label{eqn:ToSub}
 	\end{align}
 	We use this first-order estimate of the expected eigenvalues to show the result. We begin with 
 	\begin{align}
 		\frac{\Cov_{J}\left(\E{\lambda_i(\bA)}[\mu],\E{\lambda_j(\bA)}[\mu] \vert \bm P_{J} = \bp\right)}{n^2 \omega_n^2 s_i s_j} = \Cov_{J}\left(\frac{\E{\lambda_i(\bA)}[\mu]}{n \omega_n s_i},\frac{\E{\lambda_j(\bA)}[\mu]}{n \omega_n s_j} \vert \bm P_{J} = \bp\right).
 	\end{align}
 	We then substitute in equation (\ref{eqn:ToSub}) and find
 	\begin{align}
 	\Cov_{J}\left(\frac{\E{\lambda_i(\bA)}[\mu]}{n \omega_n s_i},\frac{\E{\lambda_j(\bA)}[\mu]}{n \omega_n s_j} \vert \bm P_{J} = \bp\right) = \Cov_{J}\left(P_i + \cO(\epsilon^2 P_{\min}) + \cO(\frac{t_i(\bp)}{n \omega_n s_i}),P_j + \cO(\epsilon^2 P_{\min}) + \cO(\frac{t_j(\bp)}{n \omega_n s_j})\right)
 	\end{align}
 	which concludes the calculations and the proof.
\end{propproof}

\noindent \textbf{Step 4:} This step simply puts together the prior three propositions. 
\begin{Proposition}
\label{prop:FourthPart}
Let $\bm P_J$ be an observation from $J$ with components $P_i$, let $P_{\min} = \underset{i=1,...,c}{\min} P_i$ and define $q = \epsilon P_{\min}$ where $\epsilon \ll 1$. Let $Z$ be as defined in Corollary \ref{cor:FiniteMatrixRep}.
	 \begin{align}
 	\frac{\Cov_{H_{n}}(\lambda_i, \lambda_j)}{n^2 \omega_n^2 s_i s_j}&+ \begin{cases} \frac{2\E{\lambda_i}[H_n]}{n^3 \omega_n^2 s_i^3} \quad \text{if } i = j \\ 0 \quad \text{if } i \neq j \end{cases}\\
 	 &= \Cov_{J}\left(P_i + \cO(\epsilon^2 P_{\min}^2) + \cO\left(\frac{t_i(\bp)}{n \omega_n }\right),  P_j + \cO(\epsilon^2 P_{\min}^2) + \cO\left(\frac{t_j(\bp)}{n \omega_n}\right)\right) +  \cO\left(\frac{\epsilon^2}{n^2 \omega_n }\right)
 \end{align}
 where $t_i(\bp)$ is a bounded function of the parameters for each $i$.
\end{Proposition}
\begin{propproof}
	We begin with the definition of the scaled covariance term, $\Cov_{H_{n}}(\frac{\lambda_i}{\sqrt{\omega_n}}, \frac{\lambda_j}{\sqrt{\omega_n}}).$
	\begin{align}
		\Cov_{H_{n}}(\frac{\lambda_i}{\sqrt{\omega_n}}, \frac{\lambda_j}{\sqrt{\omega_n}}) &= \Cov_{J}(\E{\frac{1}{\sqrt{\omega_n}}\lambda_i(\bA)  \vert \bm P_J = \bp}[\mu],\E{\frac{1}{\sqrt{\omega_n}} \lambda_j(\bA) \vert \bm P_J = \bp}[\mu])\\
	 & + \E{\Cov_{\mu}\left(\frac{1}{\sqrt{\omega_n}}\lambda_i(\bA) ,\frac{1}{\sqrt{\omega_n}}\lambda_j(\bA)  \vert \bm P_J = \bp \right)}[J]
	\end{align}
	When $n$ is large we replace $\Cov_{\mu}\left(\frac{1}{\sqrt{\omega_n}}\lambda_i(\bA) ,\frac{1}{\sqrt{\omega_n}}\lambda_j(\bA)  \vert \bm P_J = \bp \right)$ with $\Cov(Z_i,Z_j)$. By Proposition \ref{prop:SecondPart}, we have
	\begin{align}
		\Cov_{H_{n}}(\frac{\lambda_i}{\sqrt{\omega_n}}, \frac{\lambda_j}{\sqrt{\omega_n}}) &= \Cov_{J}\left(\E{\frac{1}{\sqrt{\omega_n}}\lambda_i(\bA)  \vert \bm P_J = \bp}[\mu],\E{\frac{1}{\sqrt{\omega_n}} \lambda_j(\bA) \vert \bm P_J = \bp}[\mu]\right)\\
	 & + \begin{cases}\E{2 P_i}[J] + \cO(\epsilon^2) \quad \text{if } i = j \\ \cO(\epsilon^2) \quad \text{if } i \neq j \end{cases}.
	\end{align}
	To continue, we divide both sides by $n^2 \omega_n s_i s_j$,
	 \begin{align}
		\frac{\Cov_{H_{n}}(\frac{\lambda_i}{\sqrt{\omega_n}}, \frac{\lambda_j}{\sqrt{\omega_n}})}{n^2 \omega_n s_i s_j} &= \frac{\Cov_{J}(\E{\frac{1}{\sqrt{\omega_n}}\lambda_i(\bA)  \vert \bm P_J = \bp}[\mu],\E{\frac{1}{\sqrt{\omega_n}} \lambda_j(\bA) \vert \bm P_J = \bp}[\mu])}{n^2 \omega_n s_i s_j}\\
	 & + \begin{cases}\frac{\E{2 P_i}[J] + \cO(\epsilon^2)}{n^2 \omega_n s_i^2} \quad \text{if } i = j \\ \frac{\cO(\epsilon^2)}{n^2 \omega_n s_i s_j} \quad \text{if } i \neq j \end{cases}. \label{eqn:FinalToSub}
	\end{align}
	We rewrite 
	\begin{align}
		 \frac{\Cov_{J}\left(\E{\frac{\lambda_i(\bA)}{\sqrt{\omega_n}}  \vert \bm P_J = \bp}[\mu],\E{\frac{ \lambda_j(\bA)}{\sqrt{\omega_n}} \vert \bm P_J = \bp}[\mu]\right)}{n^2 \omega_n s_i s_j} =  \Cov_{J}(\E{\frac{1}{n \omega_n s_i}\lambda_i(\bA)  \vert \bm P_J = \bp}[\mu],\E{\frac{1}{n \omega_n s_i} \lambda_j(\bA) \vert \bm P_J = \bp}[\mu]).
	\end{align}
	We now use Proposition \ref{prop:ThirdPart}
	\begin{align}
		 \Cov_{J}(\E{\frac{\lambda_i(\bA)}{n \omega_n s_i}  \vert \bm P_J = \bp}[\mu],\E{\frac{ \lambda_j(\bA)}{n \omega_n s_j} \vert \bm P_J = \bp}[\mu]) = \Cov_{J}\left(P_i + \cO(\epsilon^2 P_{\min}) + \cO(\frac{t_i(\bp)}{n \omega_n}) ,P_j + \cO(\epsilon^2 P_{\min}) + \cO(\frac{t_j(\bp)}{n \omega_n})\right).
	\end{align}
	Substituting this in to equation (\ref{eqn:FinalToSub}) and factoring out $\frac{1}{\sqrt{\omega_n}}$ on the left hand side,
	\begin{align}
 	\frac{\Cov_{H_{n}}(\lambda_i, \lambda_j)}{n^2 \omega_n^2 s_i s_j} &= \Cov_{J}\left(P_i + \cO(\epsilon^2 P_{\min}^2) + \cO(\frac{1}{n \omega_n }),  P_j + \cO(\epsilon^2 P_{\min}^2) + \cO(\frac{1}{n \omega_n})\right)\\
	&+ \begin{cases} \frac{1}{n^2 \omega_n s_i^2} \E{2 P_i}[J] + \cO(\frac{\epsilon^2}{n^2 \omega_n}) \quad \text{if } i = j \\ \cO(\frac{\epsilon^2}{n^2 \omega_n }) \quad \text{if } i \neq j \end{cases}. \label{eqn:finsub}
 \end{align}
 Recall from Step 1 that
 \begin{align}
	\frac{\E{\lambda_i}[H_n]}{n \omega_n s_i} &= \E{P_i}[J] + \cO(\epsilon^2 ) + \cO(\frac{1}{n \omega_n}).
 \end{align}
 Solving this for $\E{P_i}[J] $ we have
  \begin{align}
	\E{P_i}[J] &= \frac{\E{\lambda_i}[H_n]}{n \omega_n s_i} + \cO(\epsilon^2 ) + \cO(\frac{1}{n \omega_n}).
 \end{align}
 Substituting this expression in to (\ref{eqn:finsub}),
 \begin{align}
 	\frac{\Cov_{H_{n}}(\lambda_i, \lambda_j)}{n^2 \omega_n^2 s_i s_j} &= \Cov_{J}\left(P_i + \cO(\epsilon^2 P_{\min}^2) + \cO(\frac{t_i(\bp)}{n \omega_n }),  P_j + \cO(\epsilon^2 P_{\min}^2) + \cO(\frac{t_j(\bp)}{n \omega_n})\right)\\
	&+ \begin{cases} \frac{2}{n^2 \omega_n s_i^2} \left(\frac{\E{\lambda_i}[H_n]}{n \omega_n s_i} + \cO(\epsilon^2 ) + \cO(\frac{t_i(\bp)}{n \omega_n})\right) + \cO(\frac{\epsilon^2}{n^2 \omega_n}) \quad \text{if } i = j \\ \cO(\frac{\epsilon^2}{n^2 \omega_n }) \quad \text{if } i \neq j \end{cases}.
 \end{align}
 Simplifying and collecting all error terms outside the covariance, we have
  \begin{align}
 	\frac{\Cov_{H_{n}}(\lambda_i, \lambda_j)}{n^2 \omega_n^2 s_i s_j} &= \Cov_{J}\left(P_i + \cO(\epsilon^2 P_{\min}^2) + \cO(\frac{t_i(\bp)}{n \omega_n }),  P_j + \cO(\epsilon^2 P_{\min}^2) + \cO(\frac{t_j(\bp)}{n \omega_n})\right)\\
	&+ \begin{cases} \frac{2\E{\lambda_i}[H_n]}{n^3 \omega_n^2 s_i^3} + \cO(\frac{\epsilon^2}{n^2 \omega_n}) \quad \text{if } i = j \\ \cO(\frac{\epsilon^2}{n^2 \omega_n }) \quad \text{if } i \neq j \end{cases}.
 \end{align}
 And finally solving for the second moment of $J$,
\begin{align}
 	\frac{\Cov_{H_{n}}(\lambda_i, \lambda_j)}{n^2 \omega_n^2 s_i s_j}&+ \begin{cases} \frac{2\E{\lambda_i}[H_n]}{n^3 \omega_n^2 s_i^3} \quad \text{if } i = j \\ 0 \quad \text{if } i \neq j \end{cases}\\
 	 &= \Cov_{J}\left(P_i + \cO(\epsilon^2 P_{\min}^2) + \cO(\frac{t_i(\bp)}{n \omega_n }),  P_j + \cO(\epsilon^2 P_{\min}^2) + \cO(\frac{t_j(\bp)}{n \omega_n})\right) +  \cO(\frac{\epsilon^2}{n^2 \omega_n }),
\end{align}
which concludes the proof.
\end{propproof}
\end{document}